\documentclass[a4paper,fleqn]{cas-sc}

\usepackage[authoryear]{natbib}

\usepackage{threeparttable}
\usepackage{siunitx}
\usepackage{xurl}
\usepackage{longtable}
\usepackage{rotating}
\usepackage{placeins}

\definecolor{jevink}{HTML}{16324F}
\definecolor{jevfill}{HTML}{4C7FB8}
\definecolor{jevtrack}{HTML}{E4EAF1}
\definecolor{jevwarn}{HTML}{A8322A}
\newlength{\jevbarw}
\newcommand{\jevbar}[2]{%
  \leavevmode
  {\color{jevfill}\rule[-0.3pt]{#1\jevbarw}{4.2pt}}%
  {\color{jevtrack}\rule[-0.3pt]{#2\jevbarw}{4.2pt}}}
\newcommand{\jevbr}{\tabularnewline}
\newcommand{\jevhead}[1]{\textcolor{jevink}{\bfseries
  \begin{tabular}[b]{@{}c@{}}#1\end{tabular}}}
\newcommand{\jevheadp}[1]{\textcolor{jevink}{\bfseries #1}}
\newcommand{\jevflag}[1]{\textcolor{jevwarn}{#1}}
\newcommand{\jevbest}[1]{\textbf{#1}}

\ExplSyntaxOn
\cs_set:Npn \__make_fig_caption:nn #1#2
{
  \l_fig_align_tl
  \skip_vertical:N \l_fig_abovecap_skip
  \setbox\cascaptionbox=\hbox{%
     \rmfamily\small\textbf{\color{scolor}#1:}~#2}
  \ifdim\the\wd\cascaptionbox<\dim_use:N \l_fig_width_dim\relax
    \parbox{ \l_fig_width_dim }
      {\unskip\ignorespaces\hfil\rmfamily\small
       \textbf{\color{scolor}#1:}~#2\hfil\par }
  \else
    \parbox{ \l_fig_width_dim }
      {\rightskip=0pt\unskip\ignorespaces\rmfamily
       \small\textbf{\color{scolor}#1:}~#2\par }
  \fi
  \skip_vertical:N \l_fig_belowcap_skip
}
\ExplSyntaxOff
\ExplSyntaxOn
\cs_set:Npn \__make_tbl_caption:nn #1#2
{
  \l_tbl_align_tl
  \skip_vertical:N \l_tbl_abovecap_skip
  {\parbox{ \dimexpr(\l_tbl_width_dim)}
    {\rightskip=0pt\rmfamily\small\textbf{\color{scolor}#1}\par#2\par\vskip4pt }}
  \skip_vertical:N \l_tbl_belowcap_skip
}
\ExplSyntaxOff
\makeatletter
\renewcommand{\LT@makecaption}[3]{%
  \LT@mcol\LT@cols l{\parbox[t]{\LTcapwidth}{%
    \rightskip=0pt\rmfamily\small\textbf{\color{scolor}#2}\par #3\par\vskip2pt}}}
\makeatother

\makeatletter
\long\def\jev@rotmakecaption#1#2{%
  \vskip\abovecaptionskip
  \parbox[t]{\linewidth}{\rightskip=0pt\rmfamily\small
    \textbf{\color{scolor}#1}\par #2\par}%
  \vskip4pt}
\newcommand{\jevrotcaption}[2]{%
  \let\@makecaption\jev@rotmakecaption
  \caption{#2}\label{#1}}
\makeatother

\newtheorem{proposition}{Proposition}
\newtheorem{lemma}[proposition]{Lemma}
\newproof{proofprop}{Proof of Proposition~\ref{prop:floor}}
\newproof{prooflem}{Proof of Lemma~\ref{lem:noise}}
\newproof{pf}{Proof}
\newcommand{\appendixtheorems}{\renewcommand{\theproposition}{\thesection.\arabic{proposition}}\setcounter{proposition}{0}}

\newcommand{\limitationsparagraph}{}   % no-op marker read by src/style_lint.py

\newcommand{\Agreemax}{99.7}
\newcommand{\Agreemin}{97.9}
\newcommand{\BaseRateMax}{3.63}
\newcommand{\BaseRateMin}{0.43}
\newcommand{\BelowFloorVars}{drug involved, phone use, wrong way, medical episode}
\newcommand{\BestIntentionalFone}{0.91}
\newcommand{\BestUnbeltedFone}{0.88}
\newcommand{\BestVehicleDefectFone}{0.99}
\newcommand{\BestWrongWayFone}{0.77}
\newcommand{\BootstrapB}{1{,}000}
\newcommand{\BreakEvenFraction}{0.67}

\newcommand{\BreakEvenModel}{0.58}
\newcommand{\BudgetCodedSatN}{8}
\newcommand{\BudgetCompN}{9}
\newcommand{\BudgetLowerN}{9}

\newcommand{\CodedConf}{99}
\newcommand{\CoderAgreement}{99.0}
\newcommand{\CoderHours}{3.4}
\newcommand{\CoderKappaMax}{1.00}
\newcommand{\CoderKappaMedian}{0.95}
\newcommand{\CoderKappaMin}{0.95}
\newcommand{\CoderMedianSec}{7.4}
\newcommand{\CoderNAnswers}{5{,}024}
\newcommand{\CoderNCalib}{94}
\newcommand{\CoderNDouble}{2{,}371}
\newcommand{\CoderNPairs}{2{,}465}
\newcommand{\CoderNUnclear}{18}
\newcommand{\CoderN}{3}

\newcommand{\ConcurrencyN}{10}
\newcommand{\CostModelRsq}{0.987}

\newcommand{\CostPerThousand}{0.154}
\newcommand{\CostSweepCalls}{420}
\newcommand{\CostSweepCost}{0.03}
\newcommand{\CostSweepQuestionSets}{7}
\newcommand{\CostWholeCorpusSonnet}{39{,}875}
\newcommand{\CostWholeCorpus}{774}

\newcommand{\DesignDeff}{1.95}
\newcommand{\DesignMinCell}{5}
\newcommand{\DesignNActivePairs}{1{,}265}
\newcommand{\DesignNCandidates}{640}
\newcommand{\DesignNCollapsed}{6}
\newcommand{\DesignNControlPairs}{1{,}200}

\newcommand{\DesignNEffVarHi}{93}
\newcommand{\DesignNEffVarLo}{64}
\newcommand{\DesignNEff}{1{,}265}
\newcommand{\DesignNStrata}{5}
\newcommand{\DesignNWithheld}{137}
\newcommand{\DesignWeightRatio}{189}

\newcommand{\DiffFableGPT}{0.082}
\newcommand{\DiffJevFableHi}{0.098}
\newcommand{\DiffJevFableLo}{0.030}
\newcommand{\DiffJevFable}{0.059}
\newcommand{\DiffJevGPTHi}{0.048}
\newcommand{\DiffJevGPTLo}{-0.002}
\newcommand{\DiffJevGPT}{0.023}
\newcommand{\DownAddedMaxVar}{alcohol involved}
\newcommand{\DownAddedMax}{1{,}784}
\newcommand{\DownAddedTotal}{10{,}747}
\newcommand{\DownCodedTotal}{28{,}057}
\newcommand{\DownConfirmHiVar}{phone use}
\newcommand{\DownConfirmHighThreshold}{0.70}
\newcommand{\DownConfirmHi}{0.90}
\newcommand{\DownConfirmLoVar}{unbelted}
\newcommand{\DownConfirmLo}{0.15}
\newcommand{\DownNConfirmHigh}{4}
\newcommand{\DownNConfirmLow}{2}
\newcommand{\DownNFatal}{964}
\newcommand{\DownNRankMove}{4}
\newcommand{\DownNSeverity}{44{,}403}
\newcommand{\DownNVars}{9}
\newcommand{\DownRelMaxVar}{phone use}
\newcommand{\DownRelMax}{182}
\newcommand{\DownRelMedian}{39}
\newcommand{\DownRelMinVar}{unbelted}
\newcommand{\DownRelMin}{16}
\newcommand{\DownTopAddedHi}{1{,}522}
\newcommand{\DownTopAddedLo}{1{,}091}
\newcommand{\DownTopAdded}{1{,}334}
\newcommand{\DownTopCoded}{732}

\newcommand{\ECEGapMax}{2.2}
\newcommand{\ECEGapMedian}{1.0}
\newcommand{\ECEGapMin}{0.8}
\newcommand{\ECEGapNAbove}{7}
\newcommand{\ECEGapNTied}{2}
\newcommand{\ECEGapN}{9}
\newcommand{\ECEmaxVar}{wrong way}
\newcommand{\ECEmax}{0.045}
\newcommand{\ECEminVar}{unbelted}
\newcommand{\ECEmin}{0.012}
\newcommand{\EmptyShareFirst}{8.24}
\newcommand{\EmptyShareLast}{0.00}
\newcommand{\EndpointClip}{10^{-6}}
\newcommand{\ExclBudgetError}{6.9}
\newcommand{\ExclBudgetReview}{3.7}
\newcommand{\ExclFoneCorrect}{0.912}
\newcommand{\ExclFoneError}{0.875}

\newcommand{\ExclN}{49}

\newcommand{\FableDistinctP}{87}
\newcommand{\FableECE}{0.0282}

\newcommand{\FableFone}{0.967}

\newcommand{\FableSlope}{2.55}

\newcommand{\FableWeakVars}{wrong way}

\newcommand{\FableZ}{-6.0}
\newcommand{\FirstYear}{2017}
\newcommand{\FloorBoundMaxVar}{phone use}
\newcommand{\FloorBoundMax}{0.0057}
\newcommand{\FloorBoundMin}{0.0046}
\newcommand{\FloorFlagShare}{20}
\newcommand{\FloorInflMaxVar}{preg mentioned}
\newcommand{\FloorInflMax}{80}
\newcommand{\FrontierBatchSize}{20}
\newcommand{\FrontierBatches}{20}

\newcommand{\FrontierCostRatioHaiku}{26}
\newcommand{\FrontierCostRatioOpus}{129}
\newcommand{\FrontierCostRatio}{52}
\newcommand{\FrontierCostSonnet}{7.95}

\newcommand{\FrontierNRecovered}{3}
\newcommand{\FrontierNVars}{16}
\newcommand{\FrontierNarratives}{400}
\newcommand{\FrontierNotRecoveredVars}{wrong way}
\newcommand{\FrontierRecoveredVars}{unbelted, intentional, vehicle defect}
\newcommand{\FrontierWeakThreshold}{0.80}
\newcommand{\FullCovRiskMax}{2.08}
\newcommand{\FullCovRiskMin}{0.33}

\newcommand{\GPTDistinctP}{79}
\newcommand{\GPTECE}{0.0027}

\newcommand{\GPTFone}{0.885}

\newcommand{\GPTNWeakJevSound}{1}

\newcommand{\GPTPrecision}{0.970}
\newcommand{\GPTRecall}{0.814}
\newcommand{\GPTSlope}{0.98}

\newcommand{\GPTWeakVars}{vehicle defect, aggression, wrong way}

\newcommand{\GPTZ}{-0.3}

\newcommand{\GoldActiveThreshold}{0.05}
\newcommand{\GoldBSS}{0.828}

\newcommand{\GoldBaseRate}{2.7}

\newcommand{\GoldBinTargets}{30, 40, 40, 40, 30}
\newcommand{\GoldBinsText}{[0,0.05); [0.05,0.3); [0.3,0.7); [0.7,0.95); [0.95,1]}
\newcommand{\GoldBrier}{0.0045}
\newcommand{\GoldBudgetMaxVar}{wrong way}

\newcommand{\GoldBudgetNSplits}{100}
\newcommand{\GoldBudgetNinetyFiveInSample}{10.1}
\newcommand{\GoldBudgetNinetyFive}{10.0}
\newcommand{\GoldBudgetNinetyInSample}{0.0}
\newcommand{\GoldBudgetNinety}{0.8}
\newcommand{\GoldBudgetOptimism}{0.8}
\newcommand{\GoldBudgetPerVarHi}{72.2}
\newcommand{\GoldBudgetPerVarLo}{0.0}
\newcommand{\GoldBudgetPrecHeldOut}{0.906}
\newcommand{\GoldBudgetPrecLCB}{0.844}

\newcommand{\GoldECEhi}{0.0245}
\newcommand{\GoldECElo}{0.0212}

\newcommand{\GoldECE}{0.0231}
\newcommand{\GoldExpectedP}{4.8}

\newcommand{\GoldFoneHi}{0.936}
\newcommand{\GoldFoneLo}{0.871}
\newcommand{\GoldFone}{0.908}

\newcommand{\GoldIntercept}{-0.82}
\newcommand{\GoldKappa}{0.906}

\newcommand{\GoldNDisputed}{48}
\newcommand{\GoldNNarratives}{400}

\newcommand{\GoldNStrong}{11}
\newcommand{\GoldNUsable}{2{,}416}
\newcommand{\GoldNVars}{16}
\newcommand{\GoldNWeakPrecision}{3}
\newcommand{\GoldNWeakRecall}{1}
\newcommand{\GoldNWeak}{4}

\newcommand{\GoldNZeroBudget}{6}
\newcommand{\GoldNarratives}{400}
\newcommand{\GoldNegControls}{3}
\newcommand{\GoldPIIWithheld}{137}
\newcommand{\GoldPrecisionHi}{0.926}
\newcommand{\GoldPrecisionLo}{0.877}
\newcommand{\GoldPrecision}{0.902}
\newcommand{\GoldRecall}{0.915}
\newcommand{\GoldRelExcess}{80}
\newcommand{\GoldSlope}{1.63}
\newcommand{\GoldWeakDisputes}{5}
\newcommand{\GoldWeakPrecisionVars}{unbelted, intentional, wrong way}
\newcommand{\GoldWeakRecallPrec}{0.96}
\newcommand{\GoldWeakRecallValue}{0.54}

\newcommand{\GoldWeakRecallVar}{vehicle defect}
\newcommand{\GoldWeakVars}{unbelted, intentional, wrong way, vehicle defect}

\newcommand{\GoldWorstPrecVar}{unbelted}
\newcommand{\GoldWorstPrec}{0.29}

\newcommand{\GoldZSE}{0.23}
\newcommand{\GoldZ}{-5.3}
\newcommand{\GridDistinct}{100}
\newcommand{\GridMax}{1.00}
\newcommand{\GridMin}{0.01}
\newcommand{\GridStep}{0.01}
\newcommand{\JevPricePerM}{0.042}
\newcommand{\KappaCodedMinHuman}{0.45}
\newcommand{\KappaCodedMinVar}{unbelted}
\newcommand{\KappaCodedMin}{0.10}
\newcommand{\KappaGapMaxCoded}{0.48}
\newcommand{\KappaGapMaxHuman}{0.96}
\newcommand{\KappaGapMaxVar}{drug involved}

\newcommand{\KappaGapMedian}{0.26}
\newcommand{\KappaGapNPositive}{9}
\newcommand{\KappaGapN}{9}
\newcommand{\KappamaxVar}{animal involved}
\newcommand{\Kappamax}{0.91}
\newcommand{\KappaminVar}{unbelted}
\newcommand{\Kappamin}{0.10}
\newcommand{\LastYear}{2025}
\newcommand{\LeakMaxNLeaked}{7{,}577}
\newcommand{\LeakMaxNOff}{189{,}733}
\newcommand{\LeakMaxVar}{medical type}
\newcommand{\LeakMax}{3.99}
\newcommand{\LeakMin}{0.03}

\newcommand{\LengthPNinetyFive}{935}
\newcommand{\LengthPNinetyNine}{1427}
\newcommand{\MaxTries}{5}
\newcommand{\MeanTokens}{3{,}673}
\newcommand{\MedianChars}{378}
\newcommand{\MedianLatency}{0.20}
\newcommand{\MedianNarrativeChars}{378}
\newcommand{\MidHeavyVars}{medical episode, aggression, intentional, wrong way, unbelted, witness cited}
\newcommand{\MinChars}{40}

\newcommand{\NBelowOnePct}{2}
\newcommand{\NCeilingUpturn}{16}
\newcommand{\NChoice}{8}
\newcommand{\NCrashesJoined}{5{,}601{,}890}
\newcommand{\NGatedMeasured}{7}
\newcommand{\NGated}{7}
\newcommand{\NGridAnswersMillions}{2.4}
\newcommand{\NGridOnes}{2}
\newcommand{\NGridShapeVars}{16}
\newcommand{\NGridZeros}{0}

\newcommand{\NLeanQuestions}{10}
\newcommand{\NMidHeavy}{6}
\newcommand{\NPresence}{16}
\newcommand{\NPrevalenceVars}{9}
\newcommand{\NQuestions}{27}
\newcommand{\NScore}{3}

\newcommand{\NarrOnlyHi}{70}
\newcommand{\NarrOnlyLo}{47}
\newcommand{\NarrOnlyNVars}{9}
\newcommand{\NarrOnlyN}{102}
\newcommand{\NarrOnlyShare}{41}
\newcommand{\NarrOnlyStrongHi}{87}
\newcommand{\NarrOnlyStrongLo}{61}
\newcommand{\NarrOnlyStrongN}{64}
\newcommand{\NarrOnlyStrong}{77}
\newcommand{\NarrOnlyWeakHi}{43}
\newcommand{\NarrOnlyWeakLo}{12}
\newcommand{\NarrOnlyWeakN}{38}
\newcommand{\NarrOnlyWeak}{24}
\newcommand{\NarrOnly}{59}
\newcommand{\NarrativeTokens}{89}
\newcommand{\NbelowFloor}{4}
\newcommand{\Ncorpus}{5{,}018{,}080}
\newcommand{\NfloorFlagged}{5}
\newcommand{\NgridVars}{16}

\newcommand{\Nrecords}{5{,}109{,}746}
\newcommand{\NstageTwoRandom}{150{,}000}
\newcommand{\NstageTwo}{195{,}857}
\newcommand{\NvarsCoded}{9}
\newcommand{\NzeroBudgetOne}{5}
\newcommand{\NzeroBudget}{9}
\newcommand{\OtherAgreeMin}{0.9975}
\newcommand{\OtherNoulsN}{14}
\newcommand{\OtherShiftMax}{0.0039}
\newcommand{\PIICase}{0.77}
\newcommand{\PIIDate}{3.13}
\newcommand{\PIIDigitRun}{1.93}
\newcommand{\PIIDob}{0.29}
\newcommand{\PIIIdNum}{0.20}
\newcommand{\PIITitledName}{4.65}

\newcommand{\PairedN}{2{,}416}
\newcommand{\PrecAlcoholTopBlock}{84}
\newcommand{\PrecAlcoholTopCoverage}{63}
\newcommand{\PrecMaxVar}{animal involved}
\newcommand{\PrecMax}{0.85}
\newcommand{\PrecMinVar}{unbelted}
\newcommand{\PrecMin}{0.21}
\newcommand{\PrecTargetNinetyFive}{95}
\newcommand{\PrecTargetNinety}{90}
\newcommand{\PrevalenceMaxPct}{3.64}
\newcommand{\ProjectAfterCalls}{2{,}000}
\newcommand{\PursuitMeanAfter}{0.160}
\newcommand{\PursuitMeanBefore}{0.395}
\newcommand{\PursuitN}{13}
\newcommand{\PursuitPosAfter}{1}
\newcommand{\PursuitPosBefore}{4}

\newcommand{\RecalIsoECE}{0.0069}
\newcommand{\RecalIsoGain}{3.3}
\newcommand{\RecalMinPositives}{20}
\newcommand{\RecalPerVarECEHi}{0.0339}
\newcommand{\RecalPerVarECELo}{0.0007}
\newcommand{\RecalPerVarECEMedian}{0.0034}
\newcommand{\RecalPerVarNImproved}{10}
\newcommand{\RecalPerVarN}{11}
\newcommand{\RecalPerVarRawMedian}{0.0261}

\newcommand{\RecalPlattECEhi}{0.0112}
\newcommand{\RecalPlattECElo}{0.0045}
\newcommand{\RecalPlattECE}{0.0069}

\newcommand{\RecalPlattSlope}{0.97}
\newcommand{\RecalPooledOnlyN}{5}

\newcommand{\RecalRawECE}{0.0231}
\newcommand{\RecalRepeats}{50}
\newcommand{\ReviewDecisionsPerYear}{68{,}049}
\newcommand{\ReviewNarrativeShare}{11.1}
\newcommand{\ReviewNarrativesPerYear}{61{,}994}
\newcommand{\ReviewPerYearMaxVar}{alcohol involved}
\newcommand{\ReviewPerYearMax}{20{,}262}

\newcommand{\RiskTargetFive}{5}
\newcommand{\RiskTargetOne}{1}
\newcommand{\SchemaQFull}{27}
\newcommand{\SchemaQScreen}{10}
\newcommand{\SchemaTokens}{3{,}451}
\newcommand{\SchemaValidateCompared}{2{,}000}

\newcommand{\ScreenCostShare}{0.33}
\newcommand{\Slopemax}{1.24}
\newcommand{\Slopemin}{0.51}

\newcommand{\StageOneCoded}{499{,}500}
\newcommand{\StageOneCost}{25.23}
\newcommand{\StageOneFrame}{500{,}000}

\newcommand{\StageTwoCost}{30.22}
\newcommand{\StageTwoErrors}{0}
\newcommand{\StageTwoPNinetyLatency}{0.33}
\newcommand{\StageTwoReqPerS}{44}

\newcommand{\StrongFoneThreshold}{0.85}
\newcommand{\TauGate}{0.5}
\newcommand{\TexasPerYear}{557{,}564}
\newcommand{\ThreshBestGainVar}{medical episode}
\newcommand{\ThreshBestGain}{0.03}
\newcommand{\ThreshGainThreshold}{0.05}
\newcommand{\ThreshMedianGain}{-0.01}

\newcommand{\ThreshNStable}{10}

\newcommand{\ThreshRepeats}{200}
\newcommand{\ThreshWeakNPosMax}{32}
\newcommand{\ThreshWeakNPosMin}{5}
\newcommand{\ThreshWeakNThick}{2}
\newcommand{\ThreshWeakNThin}{2}
\newcommand{\ThreshWeakThickBestGain}{-0.01}
\newcommand{\ThreshWeakThickVars}{vehicle defect, wrong way}
\newcommand{\ThreshWeakThinBestGain}{0.42}
\newcommand{\ThreshWeakThinVars}{intentional, unbelted}
\newcommand{\ThreshWrongWayNPos}{21}
\newcommand{\ThreshWrongWayOof}{0.48}
\newcommand{\ThreshWrongWayTau}{0.60}
\newcommand{\TokenAlpha}{171}
\newcommand{\TokenBchar}{0.236}
\newcommand{\TokensPerQuestion}{122}
\newcommand{\TotalSpend}{55.79}
\newcommand{\UncertainHi}{0.7}
\newcommand{\UncertainLo}{0.3}
\newcommand{\WeakFoneThreshold}{0.70}

\newcommand{\WrongWayProxyRatio}{17}

\begin{document}
\let\WriteBookmarks\relax

\shorttitle{Calibrated decisions at scale}
\shortauthors{A. Rafe and S. Das}

\title[mode=title]{Calibrated Decisions at Scale: Converting Police Crash Narratives into
Probabilistic Crash Variables with a System One Model (Jev)}

\author[1]{Amir Rafe}
\cormark[1]
\ead{amir.rafe@txstate.edu}
\credit{Conceptualization, Methodology, Software, Formal analysis, Data curation, Validation,
Visualization, Writing - original draft, Writing - review \& editing}

\author[1]{Subasish Das}
\credit{Conceptualization, Methodology, Supervision, Writing - review \& editing}

\affiliation[1]{organization={Ingram School of Engineering, Texas State University},
  city={San Marcos},
  state={TX},
  country={USA}}

\cortext[1]{Corresponding author.}

\begin{abstract}
Crash datasets that carry an investigator narrative hold information the coded fields omit.
Coding those narratives at scale has been blocked by three obstacles. Frontier large language
models are costly at that scale, their generated text cannot be verified, and no rule says how
much output a human must check. This paper formulates narrative coding as gated, typed
decisions answered by Jev, a System One model that returns probabilities over analyst-defined
options and generates no text. A screen covered \StageOneCoded{} Texas narratives and
\NstageTwo{} were coded with a \NQuestions{}-question schema. Cost is governed by schema size
rather than narrative length. The probabilities are audited against
coded fields and against \GoldNUsable{} blinded human judgments drawn under a stated sampling
design. Two frontier large language models are benchmarked on the same records. Against human
labels the typed model attains an $F_1$ of \GoldFone{}. One frontier model gains
\DiffJevFable{} and the other is indistinguishable from it. Calibration varies by model rather
than by paradigm, so each model must be audited. Recalibration on the same labels reduces
calibration error by a factor of \RecalIsoGain{}. Agreement with coded fields understates
fidelity to the narrative by a median of \KappaGapMedian{} in kappa. A resolution-floor bound
covers any model that reports probabilities on a discrete grid. A review budget over flagged
records gives the records a human must read per variable and per year. Adding the calibrated
variables to the coded fields raises the injury and fatal crashes attributed to nine factors by
\DownAddedTotal{} per year.
\end{abstract}

\begin{highlights}
\item A typed decision model codes crash narratives at scale with a measured cost model
\item Its probabilities rank well but overstate prevalence until recalibrated on labels
\item A two-decimal output grid puts a floor on calibration error for rare variables
\item The human-review budget must be defined over flagged records rather than all records
\item Adding calibrated narrative variables changes which factors a safety diagnosis ranks
\end{highlights}


\maketitle

%% =====================================================================================
\section{Introduction}
\label{sec:intro}

The investigator narrative is the richest field in a crash database and the one that safety
analysis uses least. The Texas Crash Records Information System (CRIS) holds \Nrecords{}
crash records for \FirstYear{} to \LastYear{}, of which \Ncorpus{} carry a narrative longer
than \MinChars{} characters, which is about \TexasPerYear{} new narratives every year. The
coded fields of a crash report are fixed by the design of the form, so they record what the
form asks and nothing else. The narrative is where the officer writes that the driver had a
seizure, fell asleep, hydroplaned, was pregnant, was chased by another motorist, or lost a
wheel. Several of these events have no coded field at all, and others have a field that is
filled inconsistently. Text mining of narratives has already recovered hydroplaning crashes
that no field records \citep{das2020hydroplaning}, work-zone crashes that the coded fields
misclassified \citep{sayed2021identifica}, agricultural crashes hidden in general categories
\citep{kim2021crash}, and speeding designations that the coded flag missed
\citep{fitzpatrick2017an}. Each of those studies treated the narrative as a record of
circumstances that the coded fields do not carry, which is the premise of the present work.
What a narrative supplies is the officer's written account, so the quantity this paper measures
is what the narrative states rather than what occurred, and every claim below is framed that
way.

Three obstacles have kept narrative coding out of routine agency practice despite this
evidence, and the first of them is cost. Running a frontier language model over a state's
full narrative corpus is expensive enough that it is not done, so studies code a few thousand
records and extrapolate. The largest recent benchmark of frontier models against official crash coding
used some four thousand matched Arkansas fatal crashes \citep{bharati2026benchmarking}, and the Kentucky secondary-crash audit worked on a single variable in some sixteen thousand manually reviewed narratives \citep{zhang2025secondary}. A method that cannot be afforded at population scale cannot
replace or supplement coded fields at population scale, whatever its accuracy on a sample.

The second obstacle is the verification of generated output. A generative model asked to
return labels or a JSON object can invent option values, omit fields, and answer a question that was not asked, so
every output needs a parser and a validity check. When such a model is asked how confident it is, the stated number is a piece of generated text rather than a measured probability. Whether that number is calibrated depends on how it is elicited \citep{kadavath2022know,tian2023just}, and verbalized confidence tends toward over-confidence \citep{xiong2024express}. A
coding pipeline whose confidence values cannot be trusted forces an agency either to check
everything or to check nothing.

The third obstacle, and the one this paper is built around, is that no published study tells
a safety office how much human checking a narrative-derived variable needs. Agreement with
coded fields is the usual evaluation, and it answers a different question. A classifier that
is right most of the time does not tell an analyst which records to open. A probability that
is right as often as it claims to be does, because a target precision can then be converted
into a share of flagged records that a person must read. That conversion is only
valid if the probabilities are calibrated, and calibration is a property that has to be
measured against an external reference rather than assumed from a vendor's description.

The situation changed with the arrival of typed decision models. In September 2026 TypeSafe
released Jev, which its developer describes as a System One model
\citep{typesafe2026launch,typesafe2026systemone}. Such a model takes a state object and a set
of typed questions and returns, for each question, an answer from a fixed option set together
with a probability, in one pass and with no generated text. The developer states that the training
objective optimizes probabilities against outcomes and prices the model at
\$\JevPricePerM{} per million input tokens with no charge for output
\citep{typesafe2026models}. If those probabilities are calibrated, the review budget of a
narrative-derived variable becomes a computable quantity and the cost of computing it over a whole state corpus is measured in hundreds of dollars. Whether they are calibrated is an
empirical question on which the developer has published no metrics, and it is the question
this paper answers independently.

No published work combines the four elements that such an answer requires. None uses a
non-generative typed-decision model for crash-narrative coding. None audits the calibration
of narrative-derived probabilities variable by variable against both a large administrative
reference and a human reference with known sampling weights. None converts calibration into
a human-review budget defined over the records an agency would actually examine. And none
does these things at the scale of a state corpus with a cost model measured rather than
projected. The nearest studies evaluate generative models on attributes the database already
codes and report agreement \citep{bharati2026benchmarking,zhang2025secondary}, which is a
sound design for what it measures but leaves the calibration and the review budget
untouched.

Those gaps define the four research questions that this paper answers. The first asks whether
a typed decision model can code narratives at the scale of a state corpus with a schema written
by a safety analyst, and at what cost, so that large-scale coding can replace designs that code
a sample and extrapolate. The second asks whether the probabilities such a model returns are
calibrated, measured first against the coded fields and then against human judgment of the same
narratives, and whether the answer holds for frontier generative models run on the same
records. The third asks how much human review a narrative-derived variable needs once its
calibration is known, expressed per variable in records an agency must read each year. The
fourth asks what changes in a safety diagnosis when calibrated narrative variables are added to
the coded contributing-factor fields, which is the question that decides whether any of the
rest is worth an agency's attention. The paper answers the first affirmatively and reports the
executed sizes rather than a projection. For the second it finds that the typed model and one
frontier arm are not calibrated as delivered, while the other frontier arm is. It answers the
third with a budget defined over flagged records and certified out of fold, and the fourth by
measuring the injury and fatal crashes that the coded fields miss and the narratives supply.

This paper makes five connected contributions to the coding of crash narratives. It first
formulates narrative coding as a set of gated, typed decisions, in which presence questions
open detail questions and named entities are selected from a closed option set rather than
generated. The effect of each design rule is measured instead of asserted, including a paired
before-and-after test of an exclusion rule on a fixed sample. The paper then runs that
decomposition at scale, screening \StageOneCoded{} Texas narratives with a
\NLeanQuestions{}-question schema and coding \NstageTwo{} of them with the full
\NQuestions{}-question schema, at a measured cost of \$\CostPerThousand{} per thousand
narratives. Applied to the whole corpus of \Ncorpus{} narratives at that rate the projected
cost is \$\CostWholeCorpus{}, and the corpus itself was not coded. The fitted cost model shows
that spend is governed by the size of the question schema rather than by the length of the
text, which reverses the premise behind two-stage screening designs.

The third contribution is an independent calibration and selective-prediction audit of every
variable, against coded CRIS fields as a large agreement reference and against a
cell-stratified human reference set with calibration weights. That second reference measures
fidelity to the narrative rather than the facts of the crash, because a coder is asked what the
narrative says and not what probably happened. The audit finds that the model returns
probabilities on a discrete two-decimal grid whose floor exceeds the base rate of several
target variables, which places a floor on the attainable calibration error that no improvement
in ranking can remove. The fourth contribution is a reformulation of the review budget.
Accuracy-based selective risk carries little information when the positive class is rare,
because abstaining on nothing already satisfies conventional risk targets. The budget is
therefore defined over flagged records at a target precision, certified on narratives that did
not select the threshold, and reported per variable per year. The fifth contribution joins the
audit to a transportation outcome. It measures how the injury and fatal crashes attributed to
each contributing factor change when the calibrated narrative variables are added to the coded
fields, and the rate at which coders confirm narrative-only flags decides which of those
changes can be believed. Two frontier generative models are run on the same human-labeled
records as an accuracy comparison, and the schema, the runner, the metric code and every
aggregated output behind a figure or table are released.

The rest of the paper proceeds as follows. Section~\ref{sec:background} places the work in
the literatures on narrative coding, calibration and selective prediction, and typed decision
models. Section~\ref{sec:methods} describes the corpus, the schema, the reference sets and
the evaluation, and states the quantities the analysis computes as numbered equations.
Section~\ref{sec:results} reports the population run, the calibration audit, the review
budgets and the comparison with keyword rules and frontier models. Section~\ref{sec:discussion}
compares the findings with the closest published studies, and Section~\ref{sec:conclusion}
summarizes the contributions, states the boundaries of the study and sets out the work that
follows from them.

%% =====================================================================================
\section{Background}
\label{sec:background}

\subsection{Crash narratives and narrative coding in safety research}
\label{sec:bg-narratives}

Work on crash narratives has moved from keyword rules, through topic models and fine-tuned
encoders, to zero-shot prompting of large language models, and the coded field has served as
the reference at every stage. Rule-based and feature-based methods identified work-zone
crashes that the coded fields misclassified \citep{sayed2021identifica}, agricultural crashes
that a curated keyword list could recover from general categories \citep{kim2021crash}, and
contributory factors beyond the single code a form permits \citep{zou2024crashcausi}. They
also identified hydroplaning involvement, which no CRIS field records and which interpretable
models could associate with vehicle characteristics only after the narrative supplied the
label \citep{das2020hydroplaning}. Topic modeling exposed latent themes across a decade of fatal
crash narratives \citep{kwayu2021}, and fine-tuned encoders extended the approach to pedal
misapplication \citep{bareiss2021finding} and to severity classification on Texas narratives
\citep{oliaee2023using}. The same progression occurred in adjacent domains, with construction-accident narratives classified by support vector machines and other classical learners \citep{goh2017construction} and railway-accident narratives by deep networks
\citep{heidarysafa2018railway}. Two properties recur across this literature and carry into the present design. The narrative
is reliable enough to serve as the standard against which a coded field is judged, as in the
speeding-designation review of \citet{fitzpatrick2017an}, and it exposes under-reporting in
the coded fields themselves, which \citet{arteaga2025a} used a language-model framework to
quantify.

Two recent studies bracket the current state of practice with language models and together
motivate this paper. \citet{bharati2026benchmarking} evaluated six frontier language models
against official Arkansas fatal-crash coding on some four thousand matched crashes over six attributes
the database already codes. They found that a keyword-rule baseline performed comparably to
the best model and that differences across attributes exceeded differences across models,
from which they recommended attribute-specific evaluation and human review. This paper takes
that recommendation as its starting point, evaluates every variable separately, and
quantifies the human review the recommendation calls for. \citet{zhang2025secondary} audited
a single variable, secondary-crash involvement, in Kentucky narratives and showed that a
fine-tuned transformer beat prompted language models at a fraction of their cost, which sets
the cost bar a new method has to clear. Earlier comparative work on ChatGPT, BARD and GPT-4
was qualitative and small \citep{mumtarin2023llm}, and more recent frameworks generate
explanations and causal chains rather than audited probabilities
\citep{crashsage2025,su2026llmdriven}. Fine-tuning a language model to infer hazardous driver actions from narratives improved on inconsistent manual coding and produced probabilistic outputs, but it required labeled training data and reported no calibration audit \citep{hazardous2025}. Instruction-tuned models have been prompted for vulnerability indicators in police text without a confidence audit \citep{vulnerability2024}. A recent systematic review of natural language processing applied to crash reports covers severity, duration and causation and reports no calibrated, selective coding \citep{nlpsafety2026review}. Outside transportation, language models have been evaluated as text annotators \citep{gilardi2023chatgpt,ding2023annotator,ziems2024css}
and for structured extraction from scientific prose \citep{dagdelen2024structured}, and
those studies report agreement with human labels rather than calibration. The
clinical-text literature, which is several years ahead on validating model-extracted
variables, reaches a consistent verdict that carries over here. Agreement with expert adjudication is high on a binary progression question assessed from radiology reports \citep{kristjnsson2026prompting}. It is also high on the presence and grade of adverse events extracted from oncology notes but degrades on their finer attributes \citep{guillot2026extracting}, and it degrades on the finer grades of postoperative complications classified from discharge letters \citep{warmer2026zeroshot}.
That literature also supplies three cautions that shaped the schema design. Extraction performance depends heavily on how
the prompt is constructed \citep{wu2026automated}, prompt architecture induces order and
label biases that survive an instruction that the options were shuffled
\citep{brucks2025prompt}, and the largest model is not always the right one, since
fine-tuned encoders beat decoder-only models on overdose-death classification
\citep{funnell2026improving}. A typed interface with a fixed option set narrows the surface those two cautions apply to,
because the model selects from an option set the analyst fixes and returns no free text, so no
answer can be invalid and none has to be parsed. Whether the option order or the criteria
wording still moves the returned probabilities is a separate question, and this study does not
test it, so the claim made here is about output validity rather than about the absence of
order or wording effects.

The injury-surveillance literature on auto-coding narratives is the same problem in a
neighboring field and is more mature on two points that this paper needs. It confronted the division of labor between machine and human early, with Bayesian classifiers coding narratives into cause groups \citep{lehto2009bayesian} and human-machine ensembles routing the narratives coded with low confidence to manual review \citep{marucciwellm2017classifyin}. The same group developed
semi-automated coding of short injury narratives from large administrative databases
\citep{marucciwellman2015practical}, and Naive Bayes coding was applied to workers'
compensation claims and to near-miss narratives from the fire service
\citep{bertke2012naive,taylor2014nearmiss}. It also confronted rare categories
explicitly, asking whether more training data or filtering rescues them
\citep{nanda2018improving} and comparing auto-coding methods on the same causation task
\citep{bertke2016comparison}, with factorization models as one answer
\citep{chen2015factorization}. Early work showed that computerized coding of narrative text improves surveillance when a
human remains in the loop \citep{wellman2004computerized}, and a systematic review surveys the machine-learning approaches to that surveillance text \citep{vallmuur2015review}. What that line did not do is measure
whether the confidence that governs the routing is calibrated, which is what distinguishes
a review threshold from a review budget.

The coded fields that narrative studies use as their reference carry documented error, and
the direction of that error matters for the audit reported here. The validity of police-reported crash data has been examined since \citet{shinar1983validity}, and \citet{austin1995mistakes} showed that mistakes in the locational variables of accident records can be identified by comparing the report with accurate highway data from other sources. Under-reporting of crashes and injuries to police is large and uneven across countries, severities and road users \citep{elvik1999incomplete,alsop2001underreporting,amoros2006underreporting}. Linked police and hospital records show that the police file misses injured people it should contain \citep{watson2015underreporting,janstrup2016underreporting}, and police-reported injury codes are an inadequate proxy for the injury severity that in-depth investigation records \citep{farmer2003reliability}. Under-reporting also biases the parameters of severity
models \citep{yamamoto2008underreporting,ye2011underreporting}, and a review of crash data
quality for road-safety research lists incompleteness of contributing-factor fields among
its open problems \citep{imprialou2019quality}. The common pattern across these studies is omission, because a coded field fails to record a
factor far more often than it records one that did not occur, and that is the one-sided error
structure that Section~\ref{sec:evaluation} formalizes. De-identification is the last concern
the narrative literatures share. Automated removal of identifiers from free text has a long record in
clinical narratives \citep{uzuner2007deid,neamatullah2008deid,meystre2010deid,stubbs2015i2b2},
and \citet{ma2026pii} address it for crash narratives with a local agentic workflow that
reaches high recall. This paper applies a simpler regular-expression pass and a model-based
screen, and reports residual-identifier rates rather than claiming removal.

\subsection{Calibration, selective prediction, and label-efficient evaluation}
\label{sec:bg-calibration}

A probabilistic classifier is calibrated when, among the cases it assigns probability $p$, a
fraction $p$ are positive, and the tools for assessing that property come from forecast
verification. The Brier score \citep{brier1950verification} is a strictly proper scoring rule
\citep{gneiting2007strictly}, and its decomposition into reliability, resolution and
uncertainty \citep{murphy1973new} separates calibration from discrimination. Subjective
probability forecasts of weather were the first to be assessed this way
\citep{murphy1977reliability}, and \citet{degroot1983forecasters} formalized the comparison of forecasters by calibration and refinement, which is why the forecasting literature frames the target as maximizing sharpness subject to calibration \citep{gneiting2007probabilis}.
Reliability diagrams make calibration visible, and their binning has been put on a firmer
footing by consistency bands \citep{brocker2007diagrams}, by the decomposition of proper
scores \citep{brocker2009reliability}, and by binning-free estimators that make the diagram
reproducible \citep{dimitriadis2021stable,gneiting2023regression,dimitriadis2024evaluating}.
In machine learning the same quantity is summarized by expected calibration error
\citep{naeini2015obtaining}, which is sensitive to its binning \citep{nixon2019measuring},
whose estimation properties were examined by \citet{vaicenavicius2019evaluating}, and whose
verified estimation requires discretized outputs \citep{kumar2019verified}. Modern neural classifiers are usually poorly calibrated \citep{guo2017calibration}, which is why calibration has to be measured rather than assumed. Post-hoc recalibration by logistic scaling
\citep{platt1999probabilistic}, isotonic regression \citep{zadrozny2002transformi}, and
their comparison across learners \citep{niculescumiz2005predicting} remain the standard
remedies, extended by beta calibration \citep{kull2017beta} and Dirichlet calibration
\citep{kull2019dirichlet}.

The clinical prediction-model literature adds the operational reading that this paper adopts.
Calibration is assessed by a logistic regression of the outcome on the logit of the prediction,
whose slope and intercept were introduced by \citet{cox1958two}, and by the test statistic of
\citet{spiegelhalter1986probabilistic}. \citet{van2016a} defined a hierarchy of calibration
from the mean level through the slope to moderate and strong calibration, and
\citet{van2019calibratio} argue that poorly calibrated algorithms are misleading and potentially harmful for decision-making. The quantity that matters for a decision is the net benefit
at the threshold actually used \citep{vickers2006decision,steyerberg2010assessing}, and the
textbook of \citet{steyerberg2009book} treats recalibration as a routine step in model
updating. This paper carries that view to a review budget, where the decision threshold is a
target precision and the benefit is the share of flagged records that need not be read.

Selective classification supplies the decision-theoretic frame for abstention. The optimum
error-reject trade-off was characterized by \citet{chow1970on}, and the modern formulation
traces selective risk against coverage as the abstention threshold moves
\citep{elyaniv2010foundations,geifman2017selective}. Theoretical guarantees exist for
pointwise-competitive selection \citep{wiener2015agnostic}, hinge-loss and support-vector
formulations of the reject option \citep{bartlett2008reject,wegkamp2011support}, and a
general learning theory of rejection \citep{cortes2016rejection}. Deep networks have been
given an integrated reject option \citep{geifman2019selectivenet}, with the maximum softmax
probability as the baseline confidence score \citep{hendrycks2017baseline}. That literature
was developed for balanced problems, and its risk targets assume them. Narrative-derived safety variables are rare, and the precision-recall literature
has long observed that error rates dominated by the negative class conceal what happens on
the positive class \citep{davis2006prroc,saito2015prplot}. Section~\ref{sec:evaluation}
shows that the standard formulation becomes vacuous in that regime and defines the budget
over flagged records instead.

Label-efficient evaluation is the last ingredient of the design. The Horvitz--Thompson
estimator
\citep{horvitz1952generalization} makes population estimates from a sample with unequal
inclusion probabilities, with the design-based variance theory of survey sampling behind it
\citep{kish1965survey,sarndal1992model} and the bootstrap as a general interval method
\citep{efron1993bootstrap}. Active evaluation applies the same idea to classifier assessment,
choosing which items to label so that a performance estimate is precise at a fixed labeling
cost. It has been developed for F-measures \citep{sawade2010fmeasure},
for classifiers at web scale \citep{bennett2010stratified,katariya2012active}, for entity
resolution \citep{marchant2017oasis}, for rare categories \citep{poms2021lowshot}, and for
model evaluation in general \citep{kossen2021activetesting}. The reference set of this paper
follows that design, stratifying on the model's own probability so that the middle of the
range, where calibration is informative and cases are scarce, is estimable at all. Agreement
between coders is measured with Cohen's kappa \citep{cohen1960coefficient} on the scale of
\citet{landis1977measurement}, with the cautions of \citet{mchugh2012kappa}, and exact
binomial intervals \citep{clopper1934use} and the paired test of \citet{mcnemar1947note}
complete the toolkit.

\subsection{Decision models without text generation and the calibration of language-model confidence}
\label{sec:bg-decision}

The model audited here belongs to a class its developer calls System One models
\citep{typesafe2026systemone}. Such a model takes an unstructured state and a set of typed
questions and returns, for each question, an answer from the option set the analyst defined
together with a probability, in a single parallel pass and with no generated text. The developer describes the
training objective as reinforcement learning for calibrated decisions and states that the
probabilities are optimized against outcomes, that the model never makes a type error, and
that higher confidence means higher accuracy \citep{typesafe2026launch}. Three question types
exist, a binary presence question, a choice among named options, and an ordered score, and
the model is priced per input token with output free \citep{typesafe2026models}. These
claims are quoted here rather than adopted. The documentation states no numeric precision,
granularity or bounds for the returned probabilities, and the developer has published no
calibration metrics, so the output grid reported in Section~\ref{sec:res-population} and the
calibration reported in Section~\ref{sec:res-calibration} are measurements of this paper
rather than documented properties.

The contrast with generative models is in how a probability is obtained, and the distinction
has three levels that are easily conflated. A generative model does produce token likelihoods,
which are well defined and describe the words it selected, but a label's probability is not
one of them and the likelihood of the token ``yes'' is not the model's belief that the answer
is yes. A second level is self-evaluation, in which the model is asked whether its own answer
is correct. A third is verbalized confidence, in which the model is asked to state a number.
The second and third have been studied intensively, and the third is what the frontier arms of
this paper supply. Pre-trained transformers were found to be reasonably calibrated in domain out of the box \citep{desai2020calibration}, and generative language models answering factual questions were found to be poorly calibrated \citep{jiang2021calibration}. Large
models can predict whether they know an answer when asked directly, but the property
degrades with fine-tuning \citep{kadavath2022know}, and models can be trained to state
uncertainty in words \citep{lin2022words}. Conversational agents were over-confident until
linguistically calibrated \citep{mielke2022reducing}, and eliciting usable confidence from a
model tuned with human feedback requires deliberate strategies \citep{tian2023just}. An empirical comparison of elicitation methods found that verbalized confidence tends to be over-confident \citep{xiong2024express}. Semantic entropy over sampled answers has been proposed as an alternative \citep{kuhn2023semantic}, and a survey of the field catalogs the confidence-estimation and calibration methods proposed to date together with their open challenges \citep{geng2024survey}. \citet{zhou2024reliable}
report that larger and more instructable models become less reliable in the sense that
they answer more and are wrong more often on hard items. A typed interface that returns a
probability directly sidesteps elicitation, which is exactly why its calibration has to be
audited rather than inferred.

Structured-output and cascade designs are the engineering context. Constrained decoding
guarantees that a generated string matches a grammar \citep{willard2023guided}, which
removes type errors but says nothing about confidence, and cascades route inputs to a
cheaper model first and escalate on low confidence \citep{chen2023frugalgpt}, which
presupposes that the confidence is meaningful. Selective prediction for language models
follows the same logic, abstaining on domain shift \citep{kamath2020selective}, across
settings \citep{varshney2022selective}, and after self-evaluation
\citep{chen2023selfeval}. The vendor's own examples of confidence-gated fallback use
small demonstrations without an external reference \citep{typesafe2026systemone}.

Taken together, the three literatures leave a specific gap. Narrative coding has reached
population scale with rules and sample scale with generative models, and it validates
against coded fields whose error is one-sided. Calibration and selective prediction supply
the estimators and the decision frame, and both were developed for continuous outputs and for
classes that are not extremely rare. Neither assumption is a flaw in that literature, and
neither holds here, so each has to be replaced with something specific to this setting. A
discrete output grid admits an exact calibration error and imposes a bound on how small that
error can be, which is a property of the model's interface rather than of its skill. A rare
positive class makes an error-rate target uninformative and calls instead for a measure of the
workload a precision requirement implies. Typed decision models return probabilities cheaply
and without elicitation, but no independent measurement of those probabilities exists on any
safety corpus. This paper fills that gap on four fronts. It audits a typed decision model against both an
administrative and a human reference, and it derives the grid bound and applies the exact
estimator that bound permits. It defines the review budget over the records an agency would
read, and it measures what the resulting variables change in a safety diagnosis.

%% =====================================================================================
\section{Data and methods}
\label{sec:methods}

Figure~\ref{fig:framework} gives the shape of the pipeline before the parts are defined. A
crash narrative enters as program state, a typed schema of \NPresence{} presence, \NChoice{}
categorical and \NScore{} ordinal questions is put to the model in one call, and the model,
which is fixed and never trained here, returns a probability per question. Those
probabilities form a table in which every cell carries its own uncertainty, and the only
fitted object in the whole pipeline is the one-dimensional recalibration map that turns the
returned probabilities into calibrated ones. Two references, drawn below the pipeline, decide
what a probability means. The coded CRIS fields are compared against, at the scale of the random stratum, but never
fitted on, and the human reference set supplies the labels on which the recalibration map is
fitted and on which accuracy is measured. The right-hand block of the figure is the
deliverable of the audit, the share of flagged records a person must open to reach a target
precision, which the results report as \GoldBudgetNinety\% at the \PrecTargetNinety\% target.
Table~\ref{tab:notation} collects the notation used in this section, and Table~\ref{tab:eqmap}
in Appendix~\ref{app:proofs} maps every numbered equation to the script and result file that
computes it. The section first describes the corpus and its de-identification, then the
schema and the two-stage coding design with its cost model, then the three reference sets
against which the model is scored, and finally the evaluation quantities.

\begin{figure}[pos=t]\centering
\includegraphics[width=\textwidth]{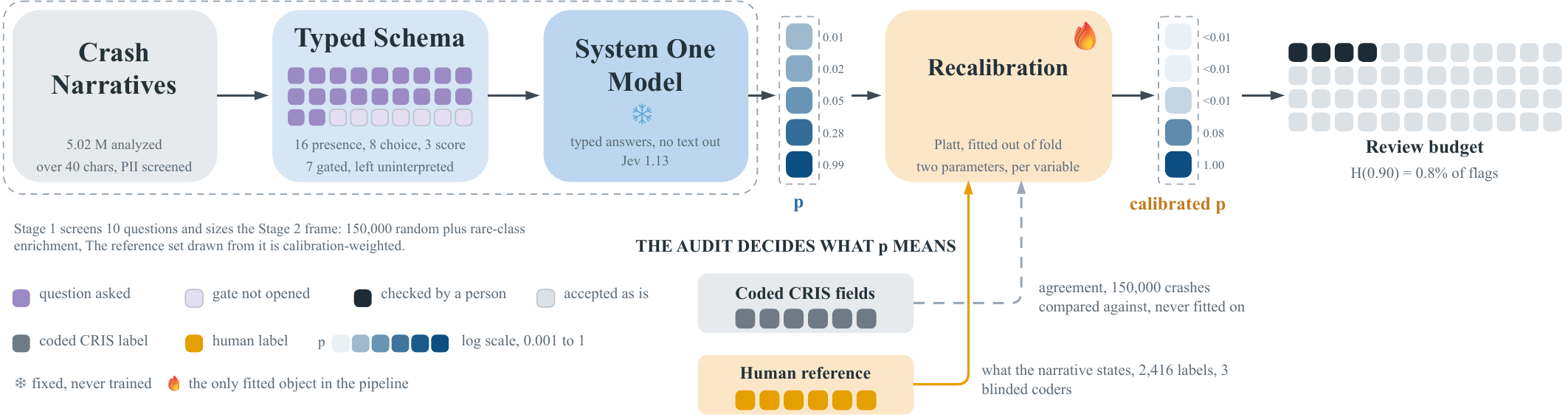}
\caption{The framework. The pipeline runs from left to right, starting with the corpus of
\Ncorpus{} screened narratives and the typed schema, a grid of questions in which filled
squares are asked of every narrative and pale squares wait on a presence question. The
snowflake marks the model as fixed and never trained, the flame marks the recalibration map
as the only fitted object, and the grid at the right is the review budget, its dark squares
the records a person must check. The coded CRIS fields (gray) join by a dashed line because
they are only compared against; the human reference set (orange) joins by a solid line
because it supplies the labels the map is fitted on.}
\label{fig:framework}
\end{figure}

\begin{table}[pos=!htbp]
\centering
\small
\caption{Notation. Inclusion probabilities carry one index and answer distributions carry two,
and the target precision carries a star, so the three uses of $\pi$ do not collide.}
\label{tab:notation}
\begin{tabular}{ll}
\toprule
\jevhead{Symbol} & \jevhead{Meaning} \\
\midrule
$n_i$, $i = 1, \dots, N$ & narrative $i$ of the analyzed corpus of size $N$ \\
$\mathcal{V} = \mathcal{B} \cup \mathcal{C} \cup \mathcal{S}$ & variables, partitioned into presence, categorical and ordinal sets \\
$O_v$, $o_v^{\varnothing}$ & option set of a categorical variable and its no-match option \\
$p_{iv} \in [0,1]$ & returned probability that presence variable $v$ holds in narrative $i$ \\
$\pi_{iv} \in \Delta(O_v)$, $\hat{o}_{iv}$ & returned distribution over options and its mode \\
$c_{iv} = \phi(\pi_{iv})$ & returned confidence attached to a categorical or ordinal answer \\
$g(v)$, $\tau_g$ & gate of a dependent variable and the gate threshold \\
$\lambda_v$ & leakage rate of a gated variable \\
$y_{iv} \in \{0,1\}$, $y'_{iv}$ & true label and coded-field label \\
$\hat{y}_{iv} = \mathbb{1}[p_{iv} > \tau]$ & decision at threshold $\tau$ \\
$\theta$ & abstention threshold on the confidence score, distinct from $\tau$ \\
$w_i$ & calibration weight of a pair in the human reference sample \\
$a_{iv}$ & probability that variable $v$ was assigned to narrative $i$ for labeling \\
$\rho$ & base rate (prevalence) of a variable \\
$G$, $\delta$, $w_0$ & attainable output grid, its smallest positive value, and the mass on zero \\
$\kappa_{iv}$ & confidence score used for abstention \\
$C(\cdot)$, $R(\cdot)$ & coverage and selective risk \\
$\pi^{*}$, $H(\pi^{*})$ & target precision and the review budget at that precision \\
$k$, $\ell_i$ & questions per call and narrative length in characters \\
\bottomrule
\end{tabular}
\end{table}

\subsection{Corpus and de-identification}
\label{sec:corpus}

The corpus is the crash-level narrative file of the Texas Crash Records Information System
for the years \FirstYear{} to \LastYear{}, held under the authors' data agreement with the
Texas Department of Transportation. It is joined to the crash-unit-person extract of the same
system by crash identifier. The narrative file holds \Nrecords{} crash records, and after
removing records whose narrative is empty or has \MinChars{} characters or fewer, \Ncorpus{}
narratives remain and form the analyzed corpus. The median narrative has \MedianChars{}
characters, the ninety-fifth percentile has \LengthPNinetyFive{}, and the ninety-ninth has
\LengthPNinetyNine{}, so a typical narrative contributes about \NarrativeTokens{} input tokens
to a model call. Narratives are written either in mixed case or in capitals, in a ratio that
is stable across years and serves as a proxy for reporting agency (Figure~\ref{fig:corpus}c).
Narrative availability changed sharply over the period, because the share of crash records
with an empty narrative falls from \EmptyShareFirst\% in \FirstYear{} to \EmptyShareLast\% in
\LastYear{} (Figure~\ref{fig:corpus}a). Any analysis that pools years without conditioning on
narrative availability under-represents the early period, and prevalence trends estimated
from narratives alone are confounded with this reporting change. Every prevalence reported
here is therefore stated on the analyzed corpus with its denominator explicit.

Names and telephone numbers were replaced with placeholders by the data owner before the
file was received, and the identifiers that survive that pass were measured over all
\Ncorpus{} narratives with the regular expressions released with the code. Date strings occur
in \PIIDate\% of narratives, long digit runs in \PIIDigitRun\%, case or report numbers in
\PIICase\%, date-of-birth tokens in \PIIDob\%, and driver-license or identification numbers
in \PIIIdNum\%. Titled personal names, such as an officer's rank followed by a surname, match
in \PIITitledName\% of narratives, though most of these name the reporting officer rather than
an involved party. A second-pass regular-expression redaction replaces each of these patterns
with a placeholder in any copy of a narrative used for display. A model-based
residual-identifier screen, which asks a single presence question about any unreplaced name,
date of birth, license number, address or telephone number, is applied before any narrative
is shown to a person outside the research team.

The sequence in which those three protections were applied differs by recipient, and it is
stated here exactly rather than summarized. The data owner's redaction of names and telephone
numbers applies to every narrative and preceded everything else. The typed-model runs, both the
first-stage screen of \StageOneCoded{} narratives and the full coding of \NstageTwo{}, sent
that owner-redacted text and nothing further, so the second-pass regular-expression redaction
and the model-based screen were not applied before them. Those two further protections were
applied to the \DesignNCandidates{} candidates of the reference frame, and only to them, before
any narrative reached a human coder or either frontier vendor. \GoldPIIWithheld{} candidates
were withheld on the screen's account and replaced. The consequence is that the human coders
and the two frontier vendors received twice-redacted and screened text, while the typed-model
service received the owner-redacted text at the rates of residual identifiers reported above.
That asymmetry follows from the order in which the work was done rather than from a judgment
that one recipient needs less protection, and an agency repeating this work should apply the
screen before the first model call. No narrative is reproduced in this paper, and every example
of narrative text is synthetic.

\subsection{Schema design and two-stage coding}
\label{sec:schema}

Narrative coding is formulated as a set of typed decisions over the narratives $n_i$ and the
variable set $\mathcal{V}$, which is partitioned into presence variables $\mathcal{B}$,
categorical variables $\mathcal{C}$ with an explicit no-match option
$o_v^{\varnothing} \in O_v$, and ordinal variables $\mathcal{S}$ with levels $0, \dots, L_v$.
For a presence variable the model returns a probability $p_{iv} \in [0,1]$ that the variable
holds. For a categorical variable it returns a distribution $\pi_{iv} \in \Delta(O_v)$ with
mode $\hat{o}_{iv}$ and a confidence $c_{iv} = \phi(\pi_{iv})$, and for an ordinal variable a
distribution over levels with an expected level and a confidence. A dependent variable $v$
has a gate $g(v) \in \mathcal{B}$, and its answer is interpreted through the gating operator
\begin{equation}
\tilde{o}_{iv} \;=\;
\begin{cases}
\hat{o}_{iv} & \text{if } p_{i,g(v)} > \tau_g,\\
o_v^{\varnothing} & \text{otherwise},
\end{cases}
\qquad \tau_g = \TauGate,
\label{eq:gate}
\end{equation}
so that a detail question is read only when its presence question has fired. The empirical
justification for the operator is the leakage rate
\begin{equation}
\lambda_v \;=\; \Pr\!\left(\hat{o}_{iv} \neq o_v^{\varnothing} \;\middle|\; p_{i,g(v)} \le \tau_g\right),
\label{eq:leak}
\end{equation}
the share of narratives in which the model would have asserted a detail whose presence
question it had itself answered in the negative. Both quantities are computed when the raw
answers are flattened, and the raw answer is kept beside the gated one so that
Eq.~(\ref{eq:leak}) can be reported for every gated variable.

The schema holds \NQuestions{} questions, of which \NPresence{} are presence questions,
\NChoice{} are categorical, \NScore{} are ordinal and \NGated{} are gated, and
Table~\ref{tab:schema-full} in Appendix~\ref{app:schema} reproduces every instruction and
criterion as sent to the model. Seven design rules govern the schema, and they were derived
from the documented failure modes of the model class. Each question makes one narrow judgment,
presence comes before detail with the detail gated by Eq.~(\ref{eq:gate}), and every
categorical question carries an explicit no-match option. Named entities such as drugs are
selected from a closed option set rather than generated, and numbers are pre-bucketed in the
criteria text, so that a stage of pregnancy is described in weeks and months rather than left
for the model to compute. No question uses a double negative or refers to another question,
because questions cannot see each other, and exclusions are written into the criteria after
error analysis. The effect of the last rule was measured rather than asserted. The aggression
question had fired on police pursuits, which are enforcement actions rather than aggression
directed at another road user, so the exclusion was added and the full schema was re-run on a
fixed sample of \SchemaValidateCompared{} narratives before and after the change. Among the
\PursuitN{} narratives containing pursuit language, positives fell from \PursuitPosBefore{}
to \PursuitPosAfter{} and the mean probability from \PursuitMeanBefore{} to
\PursuitMeanAfter{}. The \OtherNoulsN{} other presence questions were unchanged, with label
agreement of at least \OtherAgreeMin{} and a mean absolute probability shift of at most
\OtherShiftMax{}, so the change was surgical rather than a general suppression. The criteria
text is the entire specification of what each variable means, and the human coders and the
frontier models of Section~\ref{sec:reference} received that same text verbatim.

Coding proceeds in two stages over a simple random sample drawn from the analyzed corpus with
a fixed seed. In the first stage a lean schema of \NLeanQuestions{} presence questions is run
over a random frame of \StageOneFrame{} narratives, of which \StageOneCoded{} returned an
answer, at a cost of \$\StageOneCost{}. The stage estimates prevalence at large sample size
with exact binomial intervals and supplies a pool of likely positives for the rare classes
that a purely random second-stage sample would barely contain. In the second stage the full
schema is run over a frame with two parts. The first part is a random subset of
\NstageTwoRandom{} narratives from the first-stage sample, and it is the inferential
backbone, because every population estimate in this paper is computed on it alone. The
second part enriches each rare class from keyword hits and from first-stage positives up to
a per-class ceiling, so that per-variable calibration and precision are estimable for classes
that are otherwise too rare to measure. It is never pooled with the random part for a
prevalence estimate. Every narrative carries its stratum and its inclusion probability,
recorded at coding time because they cannot be recovered afterwards, and estimates that use
the enriched part are inverse-probability weighted. The enriched allocation is trimmed
proportionally within class so that the projected spend stays below its budget line, and the
trim is folded into the recorded inclusion probabilities. The frame is shuffled before
running, so that an interrupted run would leave a representative prefix rather than a
truncation that drops the classes ordered last. The second stage coded \NstageTwo{}
narratives with no errors at a cost of \$\StageTwoCost{}.

Input tokens, and therefore cost, are well described by a linear model in the number of
questions $k$ in a call and the narrative length $\ell_i$ in characters,
\begin{equation}
T(k, \ell_i) \;=\; \alpha + \beta k + b\,\ell_i,
\label{eq:cost}
\end{equation}
fitted on a sweep of \CostSweepQuestionSets{} question counts over \CostSweepCalls{} calls at a
cost of \$\CostSweepCost{}. Let $\varpi$ be the price per token and $\bar{\ell}$ the mean
narrative length. A design that screens every narrative with a lean schema of $k_s$ questions
and then codes a fraction $q$ of them with the full schema of $k_f$ questions costs
$N\varpi[T(k_s, \bar{\ell}) + q\,T(k_f, \bar{\ell})]$, against $N\varpi\,T(k_f, \bar{\ell})$
for coding everything once. Screening therefore pays only when
\begin{equation}
q \;<\; q^{*} \;=\; 1 - \frac{T(k_s, \bar{\ell})}{T(k_f, \bar{\ell})}
\;=\; 1 - \frac{\alpha + \beta k_s + b\bar{\ell}}{\alpha + \beta k_f + b\bar{\ell}}.
\label{eq:twostage}
\end{equation}
Because $\beta k$ dominates $b\bar{\ell}$ for schemas of the size used here, the break-even
fraction depends on how many questions the screen drops and not on how much text is read.
All runs use the pinned model identifier, and the identifier returned by the service is
recorded on every call so that a silent version change is detectable after the fact.
Requests run at a concurrency of \ConcurrencyN{} with exponential backoff over at most
\MaxTries{} attempts. The runner accumulates realized input tokens, projects the full-run cost
after the first \ProjectAfterCalls{} calls, and aborts if that projection exceeds the ceiling
it was given, so a mispriced schema cannot run away unattended. The total spend across
every model call reported here was \$\TotalSpend{}. The schema is
released verbatim, the runner and the metric code are released with their unit tests, and
every aggregated output behind a figure or table is released.

\subsection{Reference sets}
\label{sec:reference}

The first reference is the coded CRIS record. Narratives are joined to the person- and
unit-level extract by crash identifier, and the coded fields are aggregated to crash level for
all \NCrashesJoined{} crashes with the mapping of Table~\ref{tab:coded-map}. That mapping gives
nine presence variables a coded counterpart and records two deviations from the original plan,
the contributing-factor codes for wrong-way travel in place of the broader opposite-direction
collision type, and the ill-driver factor as the reference for medical episodes. Coded fields
are themselves incomplete, so every statistic computed against them is an agreement measure
bounded by the reference's own error, and it is labeled as such throughout. The keyword baseline
scored against that reference is a family of regular expressions per variable, listed in
Table~\ref{tab:keywords} and tuned only on a held-out development sample.

The second reference is a human reference set drawn under a cell-filling design, described
here as it was executed rather than as it was specified. Each variable's returned probability is
cut into the cells \GoldBinsText{}, and narratives were drawn sequentially from the second-stage
random stratum by a greedy rule that repeatedly took a narrative filling the largest number of
still-short cells, against a target of \GoldBinTargets{} per cell. A narrative drawn for one
variable is labeled for all of its active variables, those with
$p_{iv} \ge \GoldActiveThreshold$, plus \GoldNegControls{} randomly chosen low-probability
variables as negative controls. \DesignNCandidates{} candidates were drawn, \DesignNWithheld{}
were withheld by the identifier screen of Section~\ref{sec:corpus}, and the first
\GoldNarratives{} that cleared it were labeled, with the estimand fixed before any label
existed.

A greedy draw followed by a screen has no inclusion probabilities that can be written down, so
the weights are calibration weights in the sense of \citet{deville1992calibration} rather than
inverse inclusion probabilities. The analysis population is the second-stage random stratum, and
the cell separates the two routes into the sample exactly. A variable at or above
\GoldActiveThreshold{} is always labeled, so its assignment probability is one, and a variable
below it is labeled only as one of the \GoldNegControls{} controls, so its assignment
probability is \GoldNegControls{} divided by that narrative's count of low-probability
variables. A pair's design weight is its narrative's base weight divided by that assignment
probability, scaled within its cell so that the weighted sample reproduces the population count
of the cell. Appendix~\ref{app:coders} gives the construction in full, including the rule that
merges cells holding fewer than \DesignMinCell{} labeled pairs, and both factors are recorded
for every pair in the released audit table.

The base rate, the precision and the observed positive rate at a grid value are all population
means of some quantity $z_i$, and each is estimated by the H\'ajek ratio
\begin{equation}
\hat{\theta} \;=\; \frac{\sum_{i \in s} w_i\, z_i}{\sum_{i \in s} w_i},
\label{eq:ht}
\end{equation}
where $s$ is the labeled sample and $w_i$ the calibration weight of pair $i$. The ratio of two
weighted totals is consistent rather than unbiased, at the order Appendix~\ref{app:proofs}
gives. The weights span a factor of \DesignWeightRatio{} and give a design effect of
\DesignDeff{}, so the \GoldNUsable{} usable judgments carry the information of about
\DesignNEff{}, which is what every interval below is built on.

Intervals resample narratives rather than judgments, because one narrative supplies up to
\FrontierNVars{} judgments and they are not independent. The cluster bootstrap draws narratives
with replacement within the \DesignNStrata{} draw routes the frame records, carrying all of a
narrative's pairs and holding their weights fixed. With $B = \BootstrapB{}$ replicates the
interval is
\begin{equation}
\left[\hat{\theta}^{(\alpha/2)} - \bar{b},\; \hat{\theta}^{(1-\alpha/2)} - \bar{b}\right],
\qquad \bar{b} = \frac{1}{B}\sum_{r=1}^{B} \hat{\theta}^{(r)} - \hat{\theta},
\label{eq:htvar}
\end{equation}
where the percentile interval is re-centered by the bootstrap bias, because an absolute
deviation such as calibration error can only be pushed upward by resampling noise, and the raw
bounds are kept in the released result files. Where a binomial interval is wanted for a weighted
count, the Kish effective sample size $n_{\mathrm{eff}} = (\sum w_i)^2 / \sum w_i^2$ replaces
the raw count, and every interval obtained that way is approximate.

\CoderN{} coders with graduate training in transportation safety labeled \CoderNPairs{}
narrative-variable pairs, \CoderNAnswers{} judgments in total, under the protocol and the fixed
roster that Appendix~\ref{app:coders} sets out. They were recruited to three different profiles
so that agreement between them would not measure a shared domain inference. Each saw the
redacted narrative and the same criteria text the model was given, under the rule to label what
the narrative states rather than what probably happened, and saw nothing else, so neither the
model's probability, nor the coded field, nor another coder's answer. \CoderNDouble{} pairs
were double-coded and
\CoderNCalib{} triple-coded as a shared calibration block, on which pairwise Cohen's kappa,
\begin{equation}
\kappa \;=\; \frac{p_o - p_e}{1 - p_e},
\label{eq:kappa}
\end{equation}
with $p_o$ the observed and $p_e$ the chance agreement, runs from \CoderKappaMin{} to
\CoderKappaMax{} with a median of \CoderKappaMedian{}. That coder-to-coder statistic is
unweighted, because it describes the panel rather than the population, whereas every kappa
comparing the model with the human reference carries the calibration weights and is labeled
weighted. Raw agreement across all pairs is \CoderAgreement\%, high partly because the frame is
dominated by clear negatives, and the median judgment took \CoderMedianSec{} seconds. A pair is
resolved by majority, and ties and pairs any coder marked unclear are dropped from every rate,
which leaves \GoldNUsable{} usable judgments with \GoldNDisputed{} disputed pairs and
\CoderNUnclear{} unclear answers. Scoring the model against a label the coders could not agree
on would measure nothing, so the reference is defined on the cases people find clear.

The third reference is a pair of frontier generative models from different vendors, run on the
same \FrontierNarratives{} narratives and the same \CoderNPairs{} pairs, and scored against the
same human labels with the same weights and metric code. They received the criteria text
verbatim and the same per-narrative variable assignment, and nothing else, so neither the coded
fields, nor the base rates, nor the typed model's probabilities. A generative model does carry
token likelihoods, but those describe the words it chose rather than its belief about the event,
and neither vendor's interface exposed them here. Each arm was therefore asked to state a
probability and to use the full range, which makes the elicited number a verbalized confidence
rather than a read-out of the model's internals. The arms ran through an agent harness in
\FrontierBatches{} batches of \FrontierBatchSize{} narratives, with all criteria in context and
no constraint on intermediate reasoning, so they had more context and more inference-time
computation per judgment than the typed model. Their cost was not metered by that harness, which
is why the price comparison in the results prices a comparable tier rather than these arms. Two
exact statistics complete the toolkit, of which the first is the Clopper--Pearson interval
that prevalences and confirmation rates carry for $x$ successes in $n$ trials,
\begin{equation}
\left[\,L,\; U\right], \qquad
L = \begin{cases} 0 & x = 0\\ B^{-1}\!\left(\tfrac{\alpha}{2};\, x,\, n - x + 1\right) & x > 0\end{cases},
\qquad
U = \begin{cases} 1 & x = n\\ B^{-1}\!\left(1 - \tfrac{\alpha}{2};\, x + 1,\, n - x\right) & x < n\end{cases},
\label{eq:cp}
\end{equation}
where $B^{-1}$ is the beta quantile function. The two boundary cases are stated because the
beta quantile is undefined at them. Paired comparisons of two methods on the
same narratives use the exact McNemar test on the discordant counts $n_{10}$ and $n_{01}$,
\begin{equation}
P \;=\; \min\!\left\{1,\; 2\sum_{j \le \min(n_{10}, n_{01})} \binom{n_{10}+n_{01}}{j} 2^{-(n_{10}+n_{01})}\right\}.
\label{eq:mcnemar}
\end{equation}

\subsection{Evaluation}
\label{sec:evaluation}

Calibration is assessed with proper scores, binned and exact calibration errors, and the
logistic recalibration statistics. With weights $w_i$ normalized to sum to the sample size,
the Brier score and its Murphy decomposition over bins $k$ with weight $W_k$, mean prediction
$\bar{p}_k$ and observed rate $\bar{y}_k$ are
\begin{equation}
\mathrm{BS} = \frac{1}{n}\sum_i w_i (p_i - y_i)^2,\qquad
\mathrm{BS} \approx \underbrace{\sum_k \tfrac{W_k}{n}(\bar{p}_k - \bar{y}_k)^2}_{\text{reliability}}
\;-\; \underbrace{\sum_k \tfrac{W_k}{n}(\bar{y}_k - \bar{y})^2}_{\text{resolution}}
\;+\; \underbrace{\bar{y}(1-\bar{y})}_{\text{uncertainty}},
\label{eq:brier}
\end{equation}
where the approximation is exact on the binned representation. The expected and maximum
calibration errors over equal-mass bins are
\begin{equation}
\mathrm{ECE} = \sum_k \frac{W_k}{n}\,\bigl|\bar{p}_k - \bar{y}_k\bigr|,\qquad
\mathrm{MCE} = \max_k \bigl|\bar{p}_k - \bar{y}_k\bigr|.
\label{eq:ece}
\end{equation}
The model returns probabilities on a discrete grid $G$, so a binning-free estimator exists.
Grouping by attainable value $v \in G$ with weight $W_v$ and observed rate $r_v$ gives
\begin{equation}
\mathrm{ECE}_G = \sum_{v \in G} \frac{W_v}{n}\,\bigl|v - r_v\bigr|,
\label{eq:ecegrid}
\end{equation}
which needs no binning parameter. It is the estimator used for the human reference set, for every bootstrap interval and for the reliability curves. The coded-field errors tabulated in the results use the equal-mass estimator of Eq.~(\ref{eq:ece}) computed upstream, and the two agree at the printed precision except for one variable noted there. The calibration slope and intercept are the maximum-likelihood coefficients of
\begin{equation}
\Pr(y_i = 1 \mid p_i) = \sigma\!\left(a + b\,\mathrm{logit}(p_i)\right),
\label{eq:cox}
\end{equation}
with $a = 0$ and $b = 1$ under calibration. A slope above one means probabilities that are
not extreme enough and a slope below one means probabilities that are too extreme, and
quasi-complete separation is flagged rather than reported as a slope. The Spiegelhalter
statistic
\begin{equation}
z = \frac{\sum_i w_i (y_i - p_i)(1 - 2p_i)}{\sqrt{\sum_i w_i^2 (1 - 2p_i)^2\, p_i (1 - p_i)}}
\label{eq:spieg}
\end{equation}
is standard normal for independent observations with known probabilities. The reference set
is neither weighted nor independent in that sense, so the statistic is reported here as a
descriptive index and its reference distribution is obtained by the cluster bootstrap of
Eq.~(\ref{eq:htvar}) instead. Its sign is not interpreted, because the factor $1 - 2p$ in the
summand reverses above one half, and the direction of the miscalibration is read from the
weighted mean probability minus the prevalence, which Appendix~\ref{app:proofs} justifies.

The discrete grid has a consequence that the continuous calibration literature does not
treat, and it is stated here as a proposition. Let $\delta$ be the smallest positive value in
$G$, $w_0$ the probability mass the model places on exactly zero, and $\rho$ the base rate of
the variable.

\begin{proposition}[Resolution floor]
\label{prop:floor}
If the model is calibrated then $\rho = \mathbb{E}[p] \ge \delta(1 - w_0)$, so calibration
requires $w_0 \ge 1 - \rho/\delta$. If $\rho < \delta(1 - w_0)$, no allocation of grid values
is calibrated, the mean probability exceeds the prevalence, and
\begin{equation}
\mathrm{ECE}_G \;\ge\; \mathbb{E}[p] - \rho \;\ge\; \delta(1 - w_0) - \rho \;>\; 0.
\label{eq:floor}
\end{equation}
\end{proposition}

The proof is given in Appendix~\ref{app:proofs} and takes five lines. The model audited here
never returned zero in \NGridAnswersMillions{} million answers, so $w_0 = 0$ and the bound is
$\delta - \rho$ with $\delta = \GridStep$, and the proposition applies to every variable whose
base rate lies below \GridStep. The base rates that decide this are the coded-field
prevalences of Table~\ref{tab:calibration}, so every application of the proposition in this
paper is an application against the coded reference and is named as such. For such a variable
the mean probability is not a prevalence estimate, whatever the model's ranking quality.

Selective prediction is evaluated in two forms, the conventional one and the one defined
over flagged records. Let $\kappa_{iv} = \max(p_{iv}, 1 - p_{iv})$ be the confidence of a
presence decision and $\theta$ an abstention threshold, which is a different quantity from the
decision threshold $\tau$ of Table~\ref{tab:notation} and is given its own symbol for that
reason. The risk-coverage curve and its area are
\begin{equation}
C(\theta) = \Pr(\kappa \ge \theta),\qquad
R(\theta) = \Pr(\hat{y} \ne y \mid \kappa \ge \theta),\qquad
\mathrm{AURC} = \int_0^1 R\,\mathrm{d}C,
\label{eq:riskcov}
\end{equation}
integrated as a step function over tied confidence values, with the excess area measured
against an oracle that ranks every correct decision above every incorrect one. The coverage
attainable at a target risk $r$ is $C^{*}(r) = \max\{C(\theta) \mid R(\theta) \le r\}$. This
formulation carries little information when the positive class is rare. A rule that answers in
the negative everywhere has full-coverage risk $R(0) = \rho$, so any target $r \ge \rho$ is
already met by a rule that flags nothing and reads nothing. An error target at or above the
base rate therefore certifies a screener that does no work, which is a statement about the
target rather than about the model, and it does not follow from this that the model's own
review budget is zero. All \NPrevalenceVars{} variables with a coded reference have
thresholded rates below five percent, running up to \PrevalenceMaxPct\%, and \NBelowOnePct{}
of them fall below the \RiskTargetOne\% target as well. The quantity a safety office faces is instead defined over the records the
model flags. Let $F = \{i \mid p_{iv} > \tau\}$ be the flagged set, ranked by probability, and
for an acceptance threshold $t \ge \tau$ define the coverage and precision
\begin{equation}
c(t) = \frac{\sum_{i \in F} w_i\,\mathbb{1}[p_{iv} \ge t]}{\sum_{i \in F} w_i},\qquad
\pi(t) = \frac{\sum_{i \in F} w_i\,\mathbb{1}[p_{iv} \ge t]\,y_{iv}}{\sum_{i \in F} w_i\,\mathbb{1}[p_{iv} \ge t]},
\label{eq:prec}
\end{equation}
so that $1 - \pi(t)$ is the risk of accepting the top share $c(t)$ without reading it. The
review budget at target precision $\pi^{*}$ is the share of flagged records that cannot be
auto-accepted,
\begin{equation}
H(\pi^{*}) \;=\; 1 - \max\bigl\{\, c(t) \mid \pi(t) \ge \pi^{*} \,\bigr\},
\label{eq:budget}
\end{equation}
which is the construction of Eq.~(\ref{eq:riskcov}) restricted to the flagged set with risk
defined as one minus precision. Two conventions are needed to complete the definition of that maximum. If no acceptance
threshold on the grid attains $\pi^{*}$ the feasible set is empty, the maximum is taken to be
zero and $H(\pi^{*}) = 1$, so a variable no threshold can certify routes every flagged record
to review. Ties on the grid are never split, because a threshold cannot separate two records
carrying the same probability, and a tied block is therefore accepted or reviewed whole.

Selecting $t$ and certifying the precision it delivers on the same labels reports the best of
many thresholds as though it were the value of one. Against the human reference the budget is
therefore cross-fitted, with $t$ chosen on one half of the narratives and both the coverage it
buys and the precision it delivers measured on the other, in both directions and over repeated
splits. The in-sample value is reported beside it so that the difference is visible. The
delivered precision carries a lower confidence bound from Eq.~(\ref{eq:cp}) at the effective
accepted count, which is approximate for the reason given in Section~\ref{sec:reference}. The
budget is reported in absolute terms per year by multiplying the flagged share of the corpus by
\TexasPerYear{}, the mean annual count of analyzed narratives, and that product counts
variable-level review decisions. Where the quantity wanted is the number of documents a reader
opens, the union of the flagged sets across variables is reported instead and named as such.

The two references disagree in a direction that the structure of coded-field error predicts,
and the prediction is stated as a lemma. Let $r_v = \Pr(y = 1 \mid p = v)$ be the true
observed rate at grid value $v$, and let the coded field miss a factor the narrative states
with probability $f_v$ while never recording a factor the narrative does not state, so that
$r'_v = (1 - f_v)\,r_v$.

\begin{lemma}[One-sided reference noise]
\label{lem:noise}
For every grid value, $r'_v \le r_v$, so the reliability curve measured against the coded
field lies on or below the true curve. For every grid value at which
$(1 - f_v)\,r_v \ge v$, $|v - r'_v| \le |v - r_v|$, and if that condition holds at every
grid value carrying weight then
\begin{equation}
\mathrm{ECE}_G(y') \;\le\; \mathrm{ECE}_G(y).
\label{eq:noise}
\end{equation}
\end{lemma}

The proof is given in Appendix~\ref{app:proofs} beside the previous one. The lemma is
conditional, and its condition is that the coded field never records a factor the narrative
does not state. That is an assumption about this pair of sources rather than a property of
coded data in general, and nothing in the literature establishes it. The human reference set
supplies the empirical check, because the cell in which the coded field records a factor the
model does not flag is observable and is reported in Section~\ref{sec:res-frontier}. The
results also report the measured ratio of the two calibration errors as the companion of the
bound, and that ratio turns out to be close to one, so the bound is not where the difference
between the two references shows itself. Post-hoc recalibration is assessed as an
add-on rather than as the headline. Platt scaling and isotonic regression are fitted on one half of the narratives and scored on
the other half, over \RecalRepeats{} repeats in both directions,
\begin{equation}
q^{\mathrm{Platt}}(p) = \sigma\!\left(a + b\,\mathrm{logit}(\tilde{p})\right),\qquad
\hat{m} = \arg\min_{m \uparrow} \sum_i w_i \bigl(m(p_i) - y_i\bigr)^2,\qquad
q^{\mathrm{iso}}(p) = \hat{m}(p),
\label{eq:platt}
\end{equation}
where $\tilde{p}$ is $p$ clipped to $[\EndpointClip, 1 - \EndpointClip]$ before the logit is
taken. The clip applies in fitting and in evaluation alike, because the logit is undefined at
the endpoints and the output grid reaches them. The isotonic minimization is over non-decreasing maps $m$ and
defines a fitted function, which is then evaluated at $p$ and held constant outside the range
it was fitted on. The Platt map is fitted without a penalty, so it is a two-parameter
correction of the intercept and the slope, and temperature scaling is the special case with
$a = 0$. Splitting on narratives rather than on judgments matters here for the reason given in
Section~\ref{sec:reference}. Maps are fitted per variable wherever the reference set carries at
least \RecalMinPositives{} positive labels for that variable, and pooled otherwise, because an
analyst recalibrates one variable at a time. The decision threshold $\tau$ is also retuned per
variable to maximize weighted $F_1$ on one random half of the narratives and scored on the
other, over \ThreshRepeats{} repeats, so that a threshold failure can be told from a model
failure by the out-of-fold gain. Every quantity in this section is computed by the released
code, and Table~\ref{tab:eqmap} names the function and the result file for each equation.

%% =====================================================================================
\section{Results}
\label{sec:results}

\subsection{Population coding, cost, and the probability landscape}
\label{sec:res-population}

The second stage coded \NstageTwo{} of the \Ncorpus{} narratives in the analyzed corpus with
the full \NQuestions{}-question schema, and \NstageTwoRandom{} of them form the random
stratum on which every population estimate rests (Table~\ref{tab:run}). Coding cost
\$\StageTwoCost{} in total, which is a realized \$\CostPerThousand{} per thousand narratives,
so coding the whole corpus at that rate would cost \$\CostWholeCorpus{}. The median latency
was \MedianLatency{} seconds, the ninetieth percentile \StageTwoPNinetyLatency{} seconds, the
sustained rate about \StageTwoReqPerS{} requests per second, and the mean input
\MeanTokens{} tokens per narrative, with \StageTwoErrors{} failed calls. The fitted token
model of Eq.~(\ref{eq:cost}) has $\alpha = \TokenAlpha$, $\beta = \TokensPerQuestion$ and
$b = \TokenBchar$ tokens per character, with $R^2 = \CostModelRsq$ for the question-count model (Figure~\ref{fig:corpus}d), and Table~\ref{tab:run} also reports the direct fit at \NQuestions{} questions, whose single intercept absorbs the schema term. Its shape is the substantive finding of the run, because at
\NQuestions{} questions the schema accounts for about \SchemaTokens{} tokens per call while
a median narrative of \MedianChars{} characters contributes about \NarrativeTokens{}. Each
additional question costs a constant \TokensPerQuestion{} tokens regardless of the corpus,
so a lean screen is cheaper only in proportion to the questions it omits. Two break-even
fractions follow and they are not the same quantity. Evaluating Eq.~(\ref{eq:twostage}) with
the fitted token model at the median narrative length, with \SchemaQScreen{} questions in the
screen against \SchemaQFull{} in the full schema, predicts that screening pays only when fewer
than a fraction \BreakEvenModel{} of narratives go on to full coding. The screen as executed
cost \ScreenCostShare{} of the full call per narrative, which puts the realized break-even at
\BreakEvenFraction{}. The two differ because the lean schema carries shortened instruction
text as well as fewer questions, so its intercept is below the one the question-count model
assumes, and the executed screen is therefore cheaper than the model predicts. Under either
figure the conclusion is the same, and asking more questions in one call is the dominant
efficiency lever.

\begin{table}[pos=!htbp]
\centering
\rmfamily\small
\setlength{\tabcolsep}{4.0pt}
\caption{Corpus and run statistics.}
\label{tab:run}
\begin{threeparttable}
\begin{tabular}{>{\raggedright\arraybackslash}p{195.8pt}>{\raggedleft\arraybackslash}p{164.7pt}}
\toprule
\jevheadp{Quantity} & \jevheadp{Value} \\
\midrule
Crash records in the narrative file & 5,109,746 \\
Narratives empty & 91,175 (1.78\%) \\
Narratives analyzed ($>$40 chars) & 5,018,080 \\
Narrative length, median / p95 / p99 (chars) & 378 / 935 / 1427 \\
Stage-2 narratives coded & 195,857 \\
  of which random stratum & 150,000 \\
Model & jev-1.13.0 (195,857) \\
Mean input tokens per narrative & 3,673 \\
Median latency (s) & 0.20 \\
Token model, direct fit at 27 questions & $a$ = 3562, $b$ = 0.236/char, $R^2$ = 0.975 \\
Schema overhead per question (tokens) & 122 ($R^2$ = 0.987) \\
Cost per 1{,}000 narratives, realised (USD) & 0.1543 \\
Stage-2 spend (USD) & 30.22 \\
\bottomrule
\end{tabular}
\begin{tablenotes}[flushleft]\small\rmfamily
\item \textit{Note:}~Prices are TypeSafe list prices accessed 2026-09-17 and are quoted rather than endorsed.
\end{tablenotes}
\end{threeparttable}
\end{table}

\begin{figure}[pos=!htbp]\centering
\includegraphics[width=\textwidth]{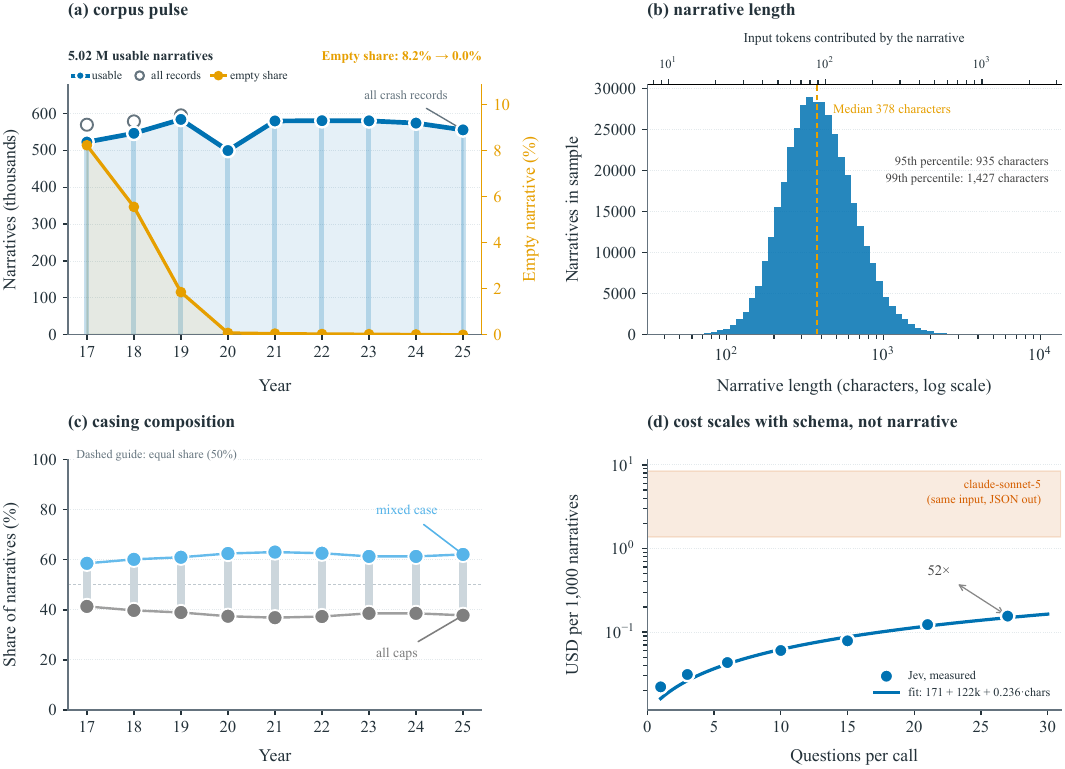}
\caption{Corpus and cost. (a) Usable narratives per year in thousands (blue) against all crash
records (hollow), with the share of records whose narrative is empty (orange, right axis)
falling from \EmptyShareFirst\% in \FirstYear{} to zero by the middle of the period.
(b) Narrative length on a logarithmic axis, median and upper percentiles marked, with the
input tokens on the upper axis. (c) Share of narratives in mixed case (light blue) and in
capitals (gray). (d) Measured cost per thousand narratives against questions per call (blue),
the fitted model of Eq.~(\ref{eq:cost}), and the list-price band of a frontier generative
model on the same input (orange).}
\label{fig:corpus}
\end{figure}

The returned probabilities pile up at the two ends of the grid in every one of the
\NgridVars{} presence variables (Figure~\ref{fig:landscape}). Most narratives receive a value at
or near the floor, and every panel turns up again at the ceiling, where the top decile holds
more mass than the decile below it in all \NCeilingUpturn{} of \NGridShapeVars{} variables. The
uncertain band between \UncertainLo{} and \UncertainHi{} holds under three percent of narratives
for every variable. The two ends do not carry equal weight everywhere. For \NMidHeavy{}
variables, namely \MidHeavyVars{}, the middle of the range holds more narratives than the
ceiling does, so the upper mode is a rise rather than a second peak. That shape is
favorable for selective prediction, because the population a threshold would route to review
is small.
The model returns its probabilities on a discrete grid, and this property was not documented
by the developer. The returned values take at most \GridDistinct{} distinct levels, running
from \GridMin{} to \GridMax{} in steps of \GridStep{}, and across \NGridAnswersMillions{}
million answers on the random stratum the value zero occurred \NGridZeros{} times and the
value one occurred \NGridOnes{} times. The same grid governs the confidence attached to
categorical and ordinal answers. The lower end of the grid has a first visible consequence in
Table~\ref{tab:prevalence}. The expected probability is the estimate a calibrated model would
match to the true prevalence, but for the rarest variables most narratives sit at the grid's
smallest value, so it largely measures the floor rather than the data. For
\FloorInflMaxVar{}, \FloorInflMax\% of the expected probability is contributed by narratives
at $p = \GridStep$. \NfloorFlagged{} variables exceed the \FloorFlagShare\% reporting
threshold, one of them by a margin that rounds away in print, and they are marked in the
table, where thresholded rates or keyword rates are the appropriate summaries. Gating behaves as designed, because the measured leakage rate of
Eq.~(\ref{eq:leak}) across the \NGatedMeasured{} gated variables ranges from \LeakMin\% to
\LeakMax\%, with the maximum for \LeakMaxVar{}, where \LeakMaxNLeaked{} of \LeakMaxNOff{}
gate-off narratives carried a substantive answer that the operator suppressed.

\begin{table}[pos=!htbp]
\centering
\rmfamily\small
\setlength{\tabcolsep}{3.0pt}
\caption{Prevalence estimates on the Stage-2 random stratum.}
\label{tab:prevalence}
\begin{threeparttable}
\begin{tabular}{l>{\raggedleft\arraybackslash}p{51.6pt}>{\raggedleft\arraybackslash}p{78.9pt}>{\raggedleft\arraybackslash}p{42.7pt}>{\raggedleft\arraybackslash}p{28.0pt}r>{\raggedleft\arraybackslash}p{71.8pt}>{\raggedleft\arraybackslash}p{45.8pt}}
\toprule
\jevhead{variable} & \jevheadp{keyword \%\newline of 5{,}018{,}080} & \jevheadp{$p>0.5$ \%\newline of 150{,}000} & \jevheadp{E[p] \%\newline of 150{,}000} & \jevheadp{floor \%} & \jevhead{hits} & \jevheadp{confirmed \%\newline of hits} & \jevheadp{uncertain \%\newline of 150{,}000} \\
\midrule
hit\_and\_run & -- & 15.863 [15.679, 16.049] & 17.384 & 0 & -- & -- & 2.02 \\
witness\_cited & -- & 6.127 [6.006, 6.249] & 9.033 & 0 & -- & -- & 2.87 \\
transported & -- & 4.997 [4.887, 5.108] & 6.778 & 2 & -- & -- & 0.42 \\
alcohol\_involved & -- & 3.635 [3.541, 3.731] & 5.498 & 4 & -- & -- & 0.38 \\
hydroplane & 0.895 & 2.352 [2.276, 2.430] & 3.623$^{\ddagger}$ & 20 & 1{,}328 & 99.77 [99.34, 99.95] & 0.41 \\
vehicle\_defect & 0.500 & 2.209 [2.135, 2.284] & 3.965 & 8 & 728 & 93.13 [91.05, 94.86] & 0.91 \\
animal\_involved & 1.997 & 2.016 [1.945, 2.088] & 3.224$^{\ddagger}$ & 25 & 2{,}981 & 83.76 [82.39, 85.07] & 0.16 \\
wrong\_way & 0.332 & 1.427 [1.368, 1.489] & 5.087 & 0 & 513 & 78.36 [74.54, 81.85] & 1.85 \\
fatigue & 1.170 & 1.399 [1.340, 1.459] & 2.653$^{\ddagger}$ & 28 & 1{,}824 & 96.77 [95.85, 97.53] & 0.15 \\
aggression & 0.566 & 1.103 [1.050, 1.157] & 3.916 & 4 & 842 & 48.46 [45.03, 51.89] & 1.11 \\
medical\_episode & 0.742 & 1.055 [1.004, 1.108] & 2.984 & 15 & 1{,}148 & 82.06 [79.71, 84.23] & 0.77 \\
phone\_use & -- & 1.038 [0.987, 1.091] & 2.463$^{\ddagger}$ & 27 & -- & -- & 0.28 \\
unbelted & -- & 0.596 [0.558, 0.636] & 2.844 & 4 & -- & -- & 0.79 \\
drug\_involved & 0.242 & 0.562 [0.525, 0.601] & 2.501 & 12 & 377 & 88.06 [84.36, 91.16] & 0.26 \\
intentional & 0.165 & 0.359 [0.329, 0.390] & 3.267 & 4 & 225 & 59.56 [52.83, 66.03] & 0.54 \\
preg\_mentioned & 0.108 & 0.116 [0.099, 0.135] & 1.179$^{\ddagger}$ & 80 & 173 & 97.11 [93.38, 99.06] & 0.01 \\
\bottomrule
\end{tabular}
\begin{tablenotes}[flushleft]\small\rmfamily
\item \textit{Note:}~Keyword \% is over all 5{,}018{,}080 narratives. Jev \% thresholds at $p>0.5$ with Clopper--Pearson intervals. E[p] is the mean probability. The floor column is the share of E[p] from narratives at the smallest probability the model can express, $p=0.01$; the double dagger marks a share above 20\%, where E[p] is inflated by the output grid and is not a prevalence estimate (Section~\ref{sec:res-population}). Confirmation \% is the share of keyword hits Jev also calls positive; uncertain \% the share with $p \in [0.3, 0.7]$.
\end{tablenotes}
\end{threeparttable}
\end{table}

\begin{figure}[pos=!htbp]\centering
\includegraphics[width=\textwidth]{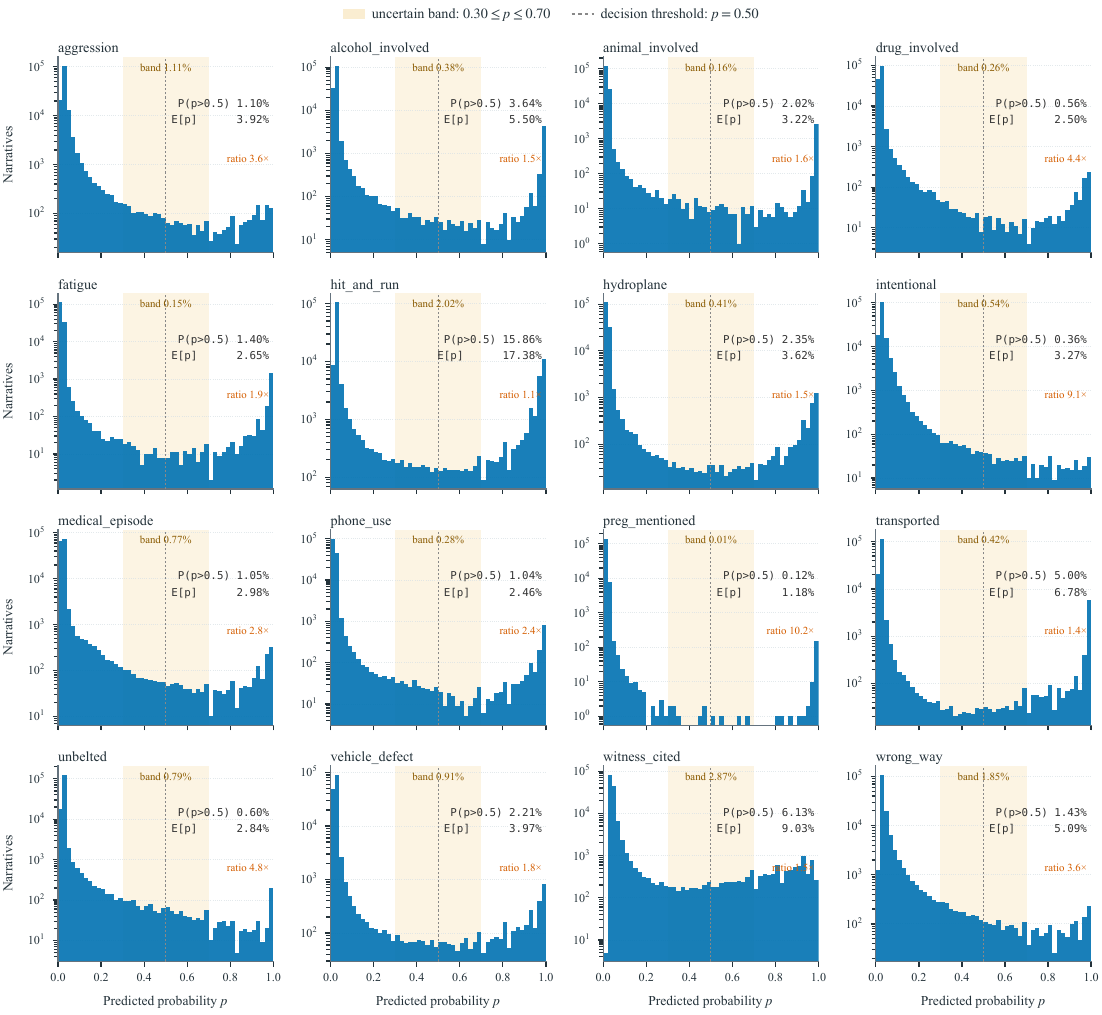}
\caption{Returned probabilities for each presence variable on the random stratum, with a
logarithmic count axis. The shaded band is the uncertain region from \UncertainLo{} to
\UncertainHi{} with its share annotated and the dashed line the decision threshold, and each
panel gives the thresholded rate, the expected probability and their ratio. Every panel
spikes at the floor and turns up again at the ceiling.}
\label{fig:landscape}
\end{figure}

\subsection{Calibration against coded fields and against human judgment}
\label{sec:res-calibration}

Table~\ref{tab:calibration} reports calibration against the coded CRIS fields for the
\NvarsCoded{} variables that have a usable reference, and Figure~\ref{fig:calibration} draws
the reliability curves on the model's own output grid with the human reference set in the
same panels. Against the coded reference the tabulated calibration error ranges from \ECEmin{} for
\ECEminVar{} to \ECEmax{} for \ECEmaxVar{}, and the calibration slope of Eq.~(\ref{eq:cox})
from \Slopemin{} to \Slopemax{}. Every calibration error reported against either reference in
this paper is the exact grid error of Eq.~(\ref{eq:ecegrid}) rather than the binned error of
Eq.~(\ref{eq:ece}). Each interval is a bootstrap of that same quantity, so every point
estimate lies inside its own interval and the two references can be compared with each other.
Slopes on both sides of one appear, so the departures are variable-specific rather than a
single systematic bias. The
curves are drawn by grid value rather than in equal-mass bins, because for a rare variable
almost all mass lands on two or three grid values and a ten-bin scheme collapses to a
handful of populated bins. The grid floor bounds what calibration is achievable. Coded base
rates run from \BaseRateMin\% to \BaseRateMax\%, and for \NbelowFloor{} variables, namely
\BelowFloorVars{}, the base rate falls below the smallest expressible probability of
\GridStep{}. Proposition~\ref{prop:floor} therefore applies to them, and the bound of
Eq.~(\ref{eq:floor}) on the attainable calibration error runs from \FloorBoundMin{} to
\FloorBoundMax{}, with the maximum for \FloorBoundMaxVar{}. The floor and the coded base rate
are marked on every panel of Figure~\ref{fig:calibration} so that the affected cases are
visible rather than inferred.

\begin{table}[pos=!htbp]
\centering
\rmfamily\small
\setlength{\tabcolsep}{3.0pt}
\caption{Calibration against coded CRIS fields (agreement).}
\label{tab:calibration}
\begin{threeparttable}
\begin{tabular}{lrr>{\raggedleft\arraybackslash}p{80pt}rrrrrrr}
\toprule
\multicolumn{3}{c}{} & \multicolumn{5}{c}{\jevhead{Error}} & \multicolumn{3}{c}{\jevhead{Cox calibration}} \\
\cmidrule(lr){4-8}\cmidrule(lr){9-11}
\jevhead{variable} & \jevhead{coded field} & \jevhead{$n$} & \jevheadp{ECE} & \jevhead{MCE} & \jevhead{Brier} & \jevhead{rel.} & \jevhead{res.} & \jevhead{slope} & \jevhead{intcpt} & \jevhead{$z$} \\
\midrule
alcohol\_involved & coded\_alcohol & 150,000 & 0.0187 [0.0180, 0.0194] & 0.022 & 0.0197 & 0.0004 & 0.0003 & 0.70 & -1.81 & -1.5 \\
drug\_involved & coded\_drug & 150,000 & 0.0199 [0.0195, 0.0203] & 0.025 & 0.0054 & 0.0004 & 0.0000 & 0.88 & -2.57 & -42.5 \\
fatigue & coded\_fatigue & 150,000 & 0.0148 [0.0144, 0.0152] & 0.034 & 0.0069 & 0.0003 & 0.0003 & 0.83 & -2.42 & -23.0 \\
animal\_involved & coded\_animal & 150,000 & 0.0149 [0.0146, 0.0152] & 0.035 & 0.0031 & 0.0003 & 0.0011 & 1.24 & -3.28 & -34.9 \\
phone\_use & coded\_phone & 150,000 & 0.0203 [0.0199, 0.0207] & 0.043 & 0.0070 & 0.0007 & 0.0000 & 0.85 & -3.58 & -27.2 \\
unbelted & coded\_unbelted & 150,000 & 0.0122 [0.0115, 0.0128] & 0.013 & 0.0186 & 0.0002 & 0.0000 & 0.51 & -2.21 & -13.1 \\
hydroplane$^{\dagger}$ & coded\_hydro\_proxy & 150,000 & 0.0166 [0.0162, 0.0171] & 0.030 & 0.0090 & 0.0003 & 0.0012 & 0.97 & -1.59 & -22.3 \\
wrong\_way & coded\_wrongway & 150,000 & 0.0455 [0.0450, 0.0460] & 0.188 & 0.0112 & 0.0048 & 0.0002 & 1.10 & -3.09 & -64.0 \\
medical\_episode & coded\_medical & 150,000 & 0.0248 [0.0244, 0.0252] & 0.037 & 0.0065 & 0.0008 & 0.0000 & 1.02 & -2.99 & -43.3 \\
\bottomrule
\end{tabular}
\begin{tablenotes}[flushleft]\small\rmfamily
\item \textit{Note:}~Coded CRIS fields are themselves incomplete, so these measure \emph{agreement} rather than accuracy; accuracy against the human reference set of Section~\ref{sec:reference} is in Table~\ref{tab:gold}. Incompleteness costs agreement rather than calibration, because a factor the narrative states and the code omits turns a correct extraction into an apparent false positive. The dagger marks a coded field that is a proxy rather than a direct match. ECE and its interval are both the exact estimator on the model's two-decimal output grid, so every point estimate lies inside its own interval. The interval is a bootstrap over crashes (B = 1{,}000), the sampling unit of the stratum.
\end{tablenotes}
\end{threeparttable}
\end{table}

\begin{figure}[pos=!htbp]\centering
\includegraphics[width=\textwidth]{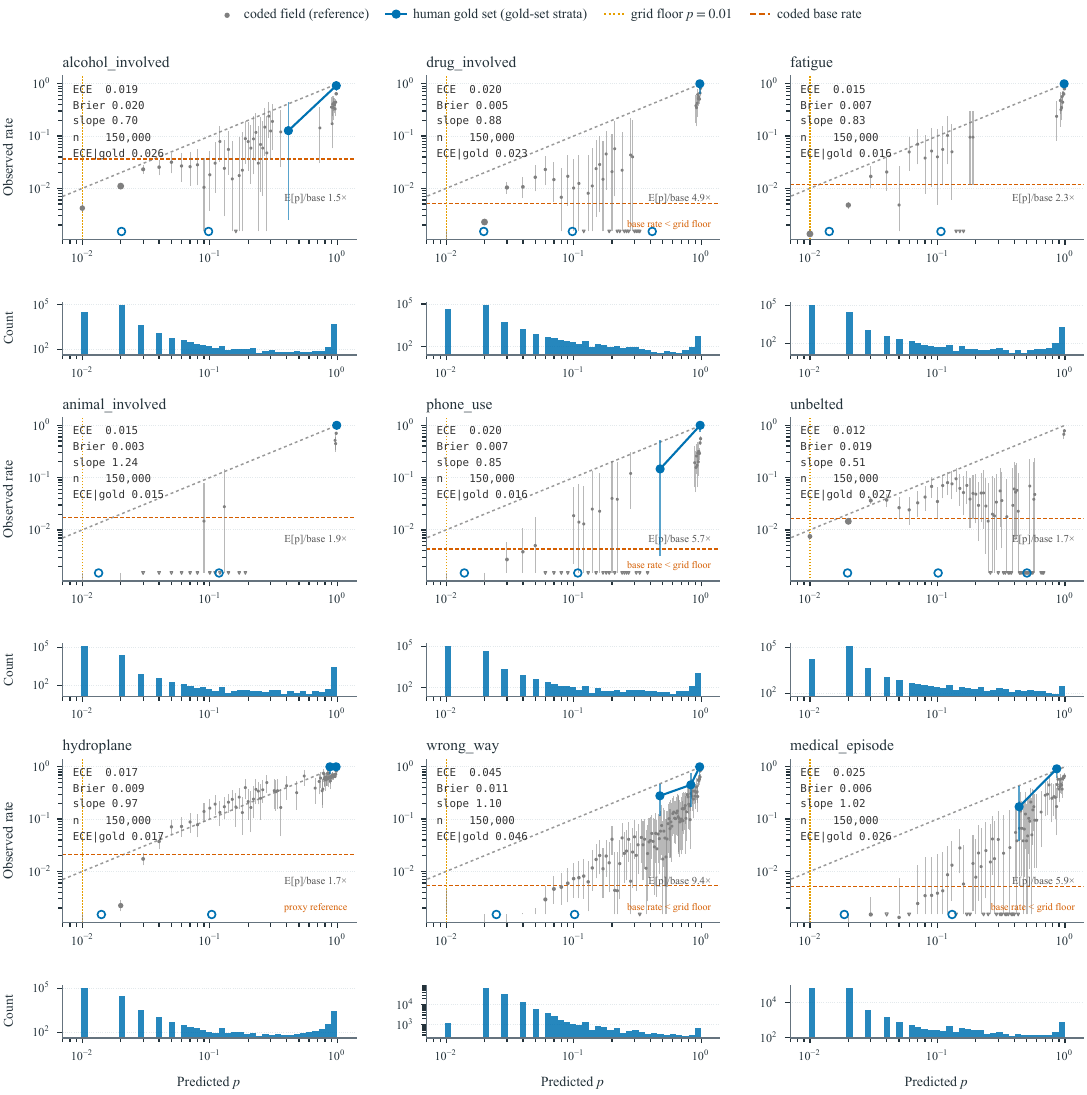}
\caption{Reliability against two references for the nine variables with a coded counterpart.
Gray markers are the observed rate in the coded CRIS field at each attainable grid value,
with exact binomial intervals and marker area proportional to support; observed rates of zero
are marked on the floor. Blue markers are the human reference set on its calibration cells,
weighted by the design weights of Section~\ref{sec:reference}, with approximate intervals at
Kish's effective sample size. Open circles are strata with no positive, and strata too thin
to support a rate are not drawn, which is why the restraint-use panel carries no filled blue
marker. Axes are logarithmic. The dotted vertical guide is the grid
floor and the dashed horizontal guide the coded base rate, and where the latter lies below
the former no allocation of grid values can be calibrated. Each panel annotates the
calibration error against both references.}
\label{fig:calibration}
\end{figure}

Against the human labels the calibration error is almost unchanged, and the difference between
the two references appears elsewhere. Table~\ref{tab:gold} scores the model over \GoldNVars{}
variables and \GoldNUsable{} usable judgments. Per variable the calibration error against the
human labels is \ECEGapMin{} to \ECEGapMax{} times the error against the coded field, with a
median ratio of \ECEGapMedian{}, and it is the larger of the two for \ECEGapNAbove{} of
\ECEGapN{} variables. The margin does not survive rounding to the three decimals
Figure~\ref{fig:calibration} prints for \ECEGapNTied{} of them, which is why a count
taken from the figure's annotations comes out lower. Once both references are weighted to the same population and measured
with the same estimator, neither one makes the model's probabilities look materially better
calibrated than the other, and Figure~\ref{fig:calibration} annotates the two errors side by
side in each panel to show it.

The two references separate sharply on agreement rather than on calibration, and that
separation is the result the rest of this paper rests on. Cohen's kappa against the coded field
is below the calibration-weighted kappa against human judgment for all \KappaGapNPositive{} of
\KappaGapN{} variables that carry a coded counterpart, by a median of \KappaGapMedian{}. The
widest case is \KappaGapMaxVar{}, which reaches \KappaGapMaxCoded{} against the coded field and
\KappaGapMaxHuman{} against human judgment, so the same extraction reads as mediocre under one
reference and as strong under the other. The narrowest agreement of all is \KappaCodedMinVar{}
at \KappaCodedMin{} against the coded field, and there the human reference agrees that the model
is weak, at \KappaCodedMinHuman{}. The mechanism is the one Lemma~\ref{lem:noise} describes,
namely one-sided omission. A factor the narrative states and the coded field does not record
turns a correct extraction into an apparent false positive. That costs the model agreement
without costing it calibration, because its probabilities stay close to the rate at which the
factor is actually stated. The practical consequence is that a coded reference
understates how faithfully the model reads a narrative, and Section~\ref{sec:res-downstream}
prices what those unrecorded factors are worth.

Calibration against the human reference is nonetheless rejected, and the failure is a level
shift rather than over-confidence. Pooled across variables the calibration error is \GoldECE{}
with a cluster-bootstrap interval of [\GoldECElo{}, \GoldECEhi{}], the Brier score is
\GoldBrier{} with a skill score of \GoldBSS{}, and the Spiegelhalter statistic of
Eq.~(\ref{eq:spieg}) is $\GoldZ$ with a bootstrap standard error of \GoldZSE{}. The calibration
slope is \GoldSlope{}, above one, so the probabilities are not extreme enough and the model
hedges toward the middle of the range, which is the opposite of the over-confidence usually
reported for neural classifiers. The probabilities also sit too high overall, with a mean of
\GoldExpectedP\% against a weighted base rate of \GoldBaseRate\% and an intercept of
$\GoldIntercept$, an over-statement of about \GoldRelExcess\% in relative terms. A safety
office summing probabilities to estimate prevalence would therefore over-count, whereas one
ranking narratives for review is unaffected, because ranking is invariant to a monotone
transformation, and that invariance is why the remedy is a recalibration rather than a
retrained model. Figure~\ref{fig:goldrel} in Appendix~\ref{app:coders} draws the same
human-referenced reliability for the seven variables without a coded counterpart.

Recalibration recovers most of the error at a small cost in labels. Table~\ref{tab:recal}
fits the maps of Eq.~(\ref{eq:platt}) on half of the narratives and scores them on the other
half in both directions, over \RecalRepeats{} repeats. Splitting on narratives rather than on
pairs matters here, because one narrative supplies up to sixteen judgments and a split over
pairs would fit the map partly on the records it is then scored against. Out of fold, the
pooled calibration error falls from \RecalRawECE{} to \RecalPlattECE{} under Platt scaling and
to \RecalIsoECE{} under isotonic regression, a factor of \RecalIsoGain{}, with an interval of
[\RecalPlattECElo{}, \RecalPlattECEhi{}] on the Platt value. The fitted map is a two-parameter
correction with an unpenalized intercept and slope, and after it the pooled slope is
\RecalPlattSlope{} against \GoldSlope{} before, so the map removes the level shift rather than
merely rescaling it.

The pooled map is a demonstration and the per-variable maps are the deployment evidence, since
an analyst recalibrates one variable at a time. Table~\ref{tab:recal} therefore carries a
per-variable block for the \RecalPerVarN{} variables whose reference labels contain at least
\RecalMinPositives{} positives. The held-out calibration error under a variable's own Platt
map runs from \RecalPerVarECELo{} to \RecalPerVarECEHi{} with a median of
\RecalPerVarECEMedian{}, against a median of \RecalPerVarRawMedian{} for the raw output, and
\RecalPerVarNImproved{} of the \RecalPerVarN{} variables improve under it. The remaining
\RecalPooledOnlyN{} variables carry too few positives to estimate their own map and can only be
recalibrated under the pooled one, which is a statement about the size of the reference set
rather than about the method. The practical claim is that a few thousand human judgments, about
\CoderHours{} hours per coder at the median response time, are enough to recalibrate the
variables that carry enough positive labels, over a corpus of \Ncorpus{} narratives.

\begin{table}[pos=!htbp]
\centering
\rmfamily\small
\setlength{\tabcolsep}{3.0pt}
\caption{Accuracy against the human reference set.}
\label{tab:gold}
\begin{threeparttable}
\begin{tabular}{lrrrrrrr>{\raggedleft\arraybackslash}p{25.4pt}r>{\raggedleft\arraybackslash}p{24.8pt}rr>{\raggedleft\arraybackslash}p{32.4pt}}
\toprule
\multicolumn{3}{c}{} & \multicolumn{6}{c}{\jevhead{Accuracy vs.\ human labels}} & \multicolumn{4}{c}{\jevhead{Calibration}} & \multicolumn{1}{c}{} \\
\cmidrule(lr){4-9}\cmidrule(lr){10-13}
\jevhead{variable} & \jevhead{$n$} & \jevhead{base \%} & \jevhead{prec.} & \jevhead{rec.} & \jevhead{$F_1$} & \jevhead{} & \jevhead{$\kappa$\jevbr  weighted} & \jevheadp{$\kappa$\newline vs.\ coded} & \jevhead{ECE} & \jevheadp{vs.\ coded} & \jevhead{gap} & \jevhead{slope} & \jevheadp{$H(0.90)$ \%\newline held out} \\
\midrule
preg\_mentioned & 107 & 0.2 & 1.00 & 1.00 & 1.00 & \jevbar{1.000}{0.000} & 1.00 & -- & 0.011 & -- & -- & sep. & 5 \\
animal\_involved & 129 & 1.9 & 1.00 & 0.98 & 0.99 & \jevbar{0.979}{0.021} & 0.99 & 0.91 & 0.015 & 0.015 & 1.0$\times$ & sep. & 0 \\
hit\_and\_run & 205 & 16.1 & 0.99 & 0.96 & 0.97 & \jevbar{0.955}{0.045} & 0.97 & -- & 0.035 & -- & -- & 2.12 & 0 \\
transported & 146 & 4.4 & 0.94 & 1.00 & 0.97 & \jevbar{0.948}{0.052} & 0.97 & -- & 0.024 & -- & -- & sep. & 0 \\
fatigue & 120 & 1.3 & 0.93 & 1.00 & 0.97 & \jevbar{0.942}{0.058} & 0.97 & 0.71 & 0.016 & 0.015 & 1.1$\times$ & sep. & 0 \\
hydroplane & 130 & 2.1 & 0.96 & 0.95 & 0.96 & \jevbar{0.929}{0.071} & 0.96 & 0.75 & 0.017 & 0.017 & 1.0$\times$ & 2.35 & 0 \\
drug\_involved & 148 & 0.4 & 0.92 & 1.00 & 0.96 & \jevbar{0.927}{0.073} & 0.96 & 0.48 & 0.023 & 0.020 & 1.1$\times$ & sep. & 0 \\
phone\_use & 123 & 0.8 & 0.93 & 0.97 & 0.95 & \jevbar{0.915}{0.085} & 0.95 & 0.48 & 0.016 & 0.020 & 0.8$\times$ & sep. & 0 \\
alcohol\_involved & 140 & 3.0 & 0.88 & 1.00 & 0.94 & \jevbar{0.895}{0.105} & 0.93 & 0.70 & 0.026 & 0.019 & 1.4$\times$ & sep. & 5 \\
aggression & 165 & 0.9 & 0.86 & 0.91 & 0.88 & \jevbar{0.806}{0.194} & 0.88 & -- & 0.029 & -- & -- & 3.23 & 11 \\
witness\_cited & 206 & 5.3 & 0.84 & 0.91 & 0.87 & \jevbar{0.792}{0.208} & 0.87 & -- & 0.055 & -- & -- & 2.15 & 19 \\
medical\_episode & 164 & 0.9 & 0.78 & 0.91 & 0.84 & \jevbar{0.741}{0.259} & 0.84 & 0.53 & 0.026 & 0.025 & 1.0$\times$ & sep. & 14 \\
\jevflag{vehicle\_defect} & 132 & 3.9 & \jevflag{0.96} & 0.54 & \jevflag{0.69} & \jevbar{0.480}{0.520} & 0.68 & -- & 0.032 & -- & -- & 1.39 & 0 \\
\jevflag{wrong\_way} & 183 & 1.0 & \jevflag{0.42} & 0.68 & \jevflag{0.52} & \jevbar{0.204}{0.796} & 0.52 & 0.37 & 0.046 & 0.045 & 1.0$\times$ & 1.55 & 72 \\
\jevflag{intentional} & 173 & 0.2 & \jevflag{0.33} & 1.00 & \jevflag{0.49} & \jevbar{0.152}{0.848} & 0.49 & -- & 0.035 & -- & -- & sep. & 68 \\
\jevflag{unbelted} & 145 & 0.2 & \jevflag{0.29} & 1.00 & \jevflag{0.45} & \jevbar{0.083}{0.917} & 0.45 & 0.10 & 0.027 & 0.012 & 2.2$\times$ & sep. & 64 \\
\textbf{pooled} & 2,416 & 2.7 & 0.90 & 0.91 & 0.91 & \jevbar{0.847}{0.153} & 0.91 & -- & 0.023 & -- & -- & 1.63 & 1 \\
\bottomrule
\end{tabular}
\begin{tablenotes}[flushleft]\small\rmfamily
\item \textit{Note:}~These are accuracy rather than agreement measures, because the reference is \GoldNUsable{} human judgments of the criteria text the model was given. Every rate is calibration-weighted to the Stage-2 random stratum under the design of Section~\ref{sec:reference}, so the base rate is the population rate and both kappa columns are population quantities. Bars scale $F_1$ from 0.40 and \jevflag{red} marks $F_1<0.70$. The gap column is ECE divided by the coded-field ECE of Table~\ref{tab:calibration}. The budget column is cross-fitted, the threshold chosen on one half of the narratives and certified on the other. ``sep.'' denotes quasi-complete separation.
\end{tablenotes}
\end{threeparttable}
\end{table}

\begin{table}[pos=!htbp]
\centering
\rmfamily\small
\setlength{\tabcolsep}{3.0pt}
\caption{Post-hoc recalibration on the human reference set, validated across narrative splits.}
\label{tab:recal}
\begin{threeparttable}
\begin{tabular}{l>{\raggedright\arraybackslash}p{112.9pt}r>{\raggedleft\arraybackslash}p{90.2pt}>{\raggedleft\arraybackslash}p{59.5pt}rr}
\toprule
\jevhead{scope} & \jevheadp{mapping} & \jevhead{$n$} & \jevheadp{ECE, held out} & \jevheadp{Brier} & \jevhead{slope} & \jevhead{intercept} \\
\midrule
pooled & none (raw model output) & 2{,}416 & 0.0231 [0.0212, 0.0244] & 0.0045 & 1.63 & -0.82 \\
pooled & Platt (logistic on the logit) & 2{,}416 & 0.0069 [0.0045, 0.0112] & 0.0036 & 0.97 & -0.06 \\
pooled & isotonic & 2{,}416 & 0.0069 [0.0045, 0.0111] & 0.0036 & 0.77 & -0.38 \\
animal\_involved & raw / Platt & 129 & 0.0146 / 0.0007 & 0.0006 / 0.0005 & -- & 0.10 \\
fatigue & raw / Platt & 120 & 0.0159 / 0.0008 & 0.0008 / 0.0004 & -- & 0.00 \\
transported & raw / Platt & 146 & 0.0240 / 0.0014 & 0.0020 / 0.0008 & -- & -0.31 \\
hydroplane & raw / Platt & 130 & 0.0169 / 0.0020 & 0.0013 / 0.0015 & 0.79 & -0.41 \\
medical\_episode & raw / Platt & 164 & 0.0260 / 0.0025 & 0.0028 / 0.0015 & 0.92 & -0.11 \\
aggression & raw / Platt & 165 & 0.0290 / 0.0034 & 0.0037 / 0.0018 & 0.93 & -0.06 \\
alcohol\_involved & raw / Platt & 140 & 0.0261 / 0.0053 & 0.0046 / 0.0030 & 0.90 & 0.08 \\
wrong\_way & raw / Platt & 183 & 0.0462 / 0.0103 & 0.0096 / 0.0059 & 0.96 & -0.11 \\
hit\_and\_run & raw / Platt & 205 & 0.0349 / 0.0114 & 0.0074 / 0.0059 & 0.93 & -0.05 \\
witness\_cited & raw / Platt & 206 & 0.0549 / 0.0194 & 0.0119 / 0.0108 & 0.96 & -0.02 \\
vehicle\_defect & raw / Platt & 132 & 0.0324 / 0.0339 & 0.0173 / 0.0185 & 0.62 & -1.20 \\
\bottomrule
\end{tabular}
\begin{tablenotes}[flushleft]\small\rmfamily
\item \textit{Note:}~Each mapping is fitted on one half of the narratives and scored on the other, both ways and over \RecalRepeats{} repeats, because an in-sample fit would be circular and a split over pairs would place one narrative on both sides of it. The row ``none'' uses the same protocol, so it is comparable to the rows below but not to the full-sample ECE of Table~\ref{tab:gold}. The Platt map is unpenalised, fitting an intercept and a slope. The pooled block is one map over all variables; the per-variable block is what an analyst deploys, and variables with fewer than \RecalMinPositives{} positive labels are omitted from it.
\end{tablenotes}
\end{threeparttable}
\end{table}

\subsection{Accuracy, selective prediction, and the review budget}
\label{sec:res-accuracy}

Against human judgment of the same criteria text, and calibration-weighted to the population
under the design of Section~\ref{sec:reference}, pooled precision is \GoldPrecision{} with an
interval of [\GoldPrecisionLo{}, \GoldPrecisionHi{}] and recall is \GoldRecall{}. Pooled
$F_1$ is \GoldFone{} with an interval of [\GoldFoneLo{}, \GoldFoneHi{}], the weighted Cohen's
kappa is \GoldKappa{}, and \GoldNStrong{} of the \GoldNVars{} variables reach an $F_1$ of
\StrongFoneThreshold{} or more (Table~\ref{tab:gold}). Every interval here resamples narratives
rather than judgments, which is the unit the reference set was drawn in. This subsection carries
the paper's only accuracy claim, and every other number in this section is agreement with coded
fields. These rates are also conditional on the pairs the coders found clear, and the bounds
that conditioning puts on them are reported at the end of this subsection.

\GoldNWeak{} variables fall below an $F_1$ of \WeakFoneThreshold{}, namely \GoldWeakVars{}, and
they do not all fail the same way. For \GoldNWeakPrecision{} of them, \GoldWeakPrecisionVars{},
the failure is precision, and \GoldWorstPrecVar{} reaches only \GoldWorstPrec{} at full recall.
For \GoldNWeakRecall{} of them, \GoldWeakRecallVar{}, the failure is the opposite, with
precision \GoldWeakRecallPrec{} and recall \GoldWeakRecallValue{}, so the model states the
factor accurately when it states it and misses more than a third of the records that carry it.
Only \GoldWeakDisputes{} of the \GoldNDisputed{} disputed pairs fall in these variables, so the
coders were essentially unanimous and these are model errors rather than ambiguous criteria,
which is why Table~\ref{tab:gold} is presented per variable. It would be natural to conclude
that these concepts are simply not recoverable from a crash narrative, but
Section~\ref{sec:res-frontier} shows that conclusion to be wrong for most of them. Two frontier
models given the same narratives and the same criteria recover \FrontierNRecovered{} of the
\GoldNWeak{} weak variables above an $F_1$ of \FrontierWeakThreshold{}, namely
\FrontierRecoveredVars{}, reaching \BestIntentionalFone{} on intentional acts,
\BestUnbeltedFone{} on restraint use and \BestVehicleDefectFone{} on vehicle defects. The
information is therefore in the narrative for those, and the failures belong to the interaction
between this model and this schema rather than to the narrative. The exception is
\FrontierNotRecoveredVars{}, where the best frontier arm reaches \BestWrongWayFone{} and no arm
clears the threshold, so that concept may genuinely be hard to state from a narrative alone.

The obvious remedy does not work, and the test is reported rather than the intuition.
Precision \GoldWorstPrec{} at full recall looks like a threshold set too low, particularly
given the upward level shift established above, so Table~\ref{tab:threshold} retunes $\tau$
per variable to maximize $F_1$ on one random half of the narratives and scores it on the
other over \ThreshRepeats{} repeats. No variable gains as much as \ThreshGainThreshold{} in
$F_1$ among the \ThreshNStable{} for which the procedure is stable, the median change is
\ThreshMedianGain{}, and the best is \ThreshBestGain{} for \ThreshBestGainVar{}. The weak
variables divide by how much evidence they carry. \ThreshWeakNThin{} of them,
\ThreshWeakThinVars{}, appear to gain as much as \ThreshWeakThinBestGain{}, but each carries
only \ThreshWeakNPosMin{} positive labels and both exceed the stability bound. The apparent
improvement is a threshold fitted to noise, and the interquartile range of the chosen threshold
in Table~\ref{tab:threshold} shows how little the two halves agreed. The other
\ThreshWeakNThick{}, \ThreshWeakThickVars{}, carry up to \ThreshWeakNPosMax{} positives and do
not gain at all, the best of them changing by \ThreshWeakThickBestGain{}. Wrong-way travel is
the clearest case, since at a tuned threshold of \ThreshWrongWayTau{} on
\ThreshWrongWayNPos{} positives its out-of-fold $F_1$ is \ThreshWrongWayOof{}, below what the
default already achieves. The default threshold of \TauGate{} is therefore kept throughout,
since the well-behaved variables are not improved by retuning either.

These failures belong to the pairing of this model with this schema rather than to the schema
alone. The frontier arms of Section~\ref{sec:res-frontier} read the same criteria text and
recover \FrontierNRecovered{} of the \GoldNWeak{} weak variables, so the wording cannot be the
whole explanation, and a rewrite of the criteria would repair the interaction rather than a
defect in the concepts. That rewrite is not attempted here, because the reference set was
drawn against the current schema and a schema tuned until its weak variables pass on the same
labels would be circular.

\begin{table}[pos=!htbp]
\centering
\rmfamily\small
\setlength{\tabcolsep}{3.0pt}
\caption{Per-variable decision threshold, retuned on the human reference set and validated out of fold.}
\label{tab:threshold}
\begin{threeparttable}
\begin{tabular}{lr>{\raggedleft\arraybackslash}p{46.6pt}>{\raggedleft\arraybackslash}p{34.4pt}>{\raggedleft\arraybackslash}p{73.3pt}>{\raggedleft\arraybackslash}p{91.7pt}rrr}
\toprule
\multicolumn{3}{c}{} & \multicolumn{6}{c}{\jevhead{Retuned threshold, out of fold}} \\
\cmidrule(lr){4-9}
\jevhead{variable} & \jevhead{$n^{+}$} & \jevheadp{$F_1$ at $\tau=0.5$} & \jevheadp{tuned $\tau^{*}$} & \jevheadp{half-sample $\tau$\newline median [IQR]} & \jevheadp{$F_1$ tuned, out of fold} & \jevhead{s.d.} & \jevhead{change} & \jevhead{stable?} \\
\midrule
unbelted & 5 & 0.45 & 0.77 & 0.75 [0.73, 0.77] & 0.87 & 0.17 & +0.42 & no \\
intentional & 5 & 0.49 & 0.66 & 0.66 [0.64, 0.70] & 0.65 & 0.22 & +0.16 & no \\
wrong\_way & 21 & 0.52 & 0.60 & 0.60 [0.40, 0.76] & 0.48 & 0.09 & -0.05 & yes \\
vehicle\_defect & 32 & 0.69 & 0.40 & 0.37 [0.15, 0.47] & 0.68 & 0.15 & -0.01 & no \\
medical\_episode & 24 & 0.84 & 0.73 & 0.70 [0.53, 0.73] & 0.87 & 0.07 & +0.03 & yes \\
witness\_cited & 41 & 0.87 & 0.42 & 0.42 [0.39, 0.50] & 0.84 & 0.05 & -0.04 & yes \\
aggression & 31 & 0.88 & 0.49 & 0.49 [0.48, 0.61] & 0.85 & 0.06 & -0.03 & yes \\
alcohol\_involved & 25 & 0.94 & 0.94 & 0.79 [0.70, 0.94] & 0.93 & 0.04 & -0.00 & yes \\
phone\_use & 15 & 0.95 & 0.66 & 0.59 [0.54, 0.66] & 0.95 & 0.05 & -0.00 & no \\
drug\_involved & 8 & 0.96 & 0.60 & 0.53 [0.46, 0.60] & 0.93 & 0.09 & -0.02 & no \\
hydroplane & 31 & 0.96 & 0.54 & 0.27 [0.24, 0.54] & 0.94 & 0.03 & -0.02 & yes \\
fatigue & 27 & 0.97 & 0.62 & 0.62 [0.52, 0.66] & 0.96 & 0.03 & -0.00 & yes \\
transported & 35 & 0.97 & 0.71 & 0.71 [0.62, 0.71] & 0.98 & 0.02 & +0.01 & yes \\
hit\_and\_run & 92 & 0.97 & 0.30 & 0.35 [0.30, 0.43] & 0.97 & 0.01 & -0.00 & yes \\
animal\_involved & 23 & 0.99 & 0.39 & 0.34 [0.28, 0.39] & 0.98 & 0.04 & -0.01 & yes \\
preg\_mentioned & 4 & 1.00 & 0.06 & 0.06 [0.05, 0.06] & 0.94 & 0.07 & -0.06 & no \\
\bottomrule
\end{tabular}
\begin{tablenotes}[flushleft]\small\rmfamily
\item \textit{Note:}~$n^{+}$ counts positive reference labels. $\tau^{*}$ maximizes weighted $F_1$ on the full reference set; the out-of-fold column picks $\tau$ on a random half and scores it on the other, both ways over 200 repeats, splitting on narratives rather than pairs, and is what retuning actually buys. The median and interquartile range are of the thresholds those half samples chose. An unstable threshold has a spread above 0.10 or fewer than 20 positives, is fitted to noise, and must not be transported.
\end{tablenotes}
\end{threeparttable}
\end{table}

The conventional selective-prediction formulation returns nothing on these variables. At
full coverage the accuracy-based selective risk of Eq.~(\ref{eq:riskcov}) against the coded
reference runs from \FullCovRiskMin\% to \FullCovRiskMax\% (Figure~\ref{fig:riskcoverage}a).
Every variable therefore sits below the \RiskTargetFive\% target usually adopted in the
selective-classification literature, \NzeroBudgetOne{} of \NvarsCoded{} sit below the
\RiskTargetOne\% target as well, and the implied review budget is zero for \NzeroBudget{} of
\NvarsCoded{} variables at the \RiskTargetFive\% target. Applied unchanged, the standard
analysis would certify a rare-event screener as needing no human review at all. The budget
over flagged records carries the information instead. Precision among flagged positives
against the coded reference ranges from \PrecMin{} for \PrecMinVar{} to \PrecMax{} for
\PrecMaxVar{}, so the review effort lies entirely in the flagged set. Ranking that set by
probability and auto-accepting the most confident traces the curves of
Figure~\ref{fig:riskcoverage}b. Those curves are drawn as steps because the output grid is
discrete, so the only operating points that exist are the attainable thresholds and the
intervals between them carry the precision of the lower one. The alcohol curve shows why that distinction
matters in practice. It reaches full precision only on a negligible share of its flags, and it
then drops
straight through the policy band to \PrecAlcoholTopBlock\% as soon as the block of records at
the highest grid value is accepted, which is \PrecAlcoholTopCoverage\% of its flags at once.
Eq.~(\ref{eq:budget}) reads those attainable points and converts them into the budgets of
Table~\ref{tab:selective}. Against the coded reference the budget
saturates, because \BudgetCodedSatN{} of the \BudgetCompN{} variables measurable against both
references require reviewing essentially every flagged record to reach \PrecTargetNinety\%
precision (Figure~\ref{fig:riskcoverage}c), and only \PrecMaxVar{} can auto-accept most of its
flags. As annual workload for Texas, that precision across all \NvarsCoded{} variables implies
\ReviewDecisionsPerYear{} variable-level review decisions per year, of which
\ReviewPerYearMax{} are for \ReviewPerYearMaxVar{} alone (Figure~\ref{fig:riskcoverage}d).
That total counts decisions rather than documents, because a narrative that more than one
variable routes to review contributes once for each of them. A reader opens such a narrative
once and answers every question it raises, so the staffing figure is the union of the flagged
sets, which is \ReviewNarrativesPerYear{} narratives per year, or \ReviewNarrativeShare\% of
the corpus.

These budgets answer how much review is needed to match the coded field rather than how much is
needed to be right, and the human reference shows the difference to be the dominant term. The
budget also has to be certified on labels that did not choose it, because the acceptance
threshold is selected from a grid and the best of many thresholds is not the value of one.
Table~\ref{tab:selective} therefore reports the budget twice, chosen on one half of the
narratives and measured on the other over \GoldBudgetNSplits{} splits, and chosen and measured
on the same half. Pooled, the held-out budget of Eq.~(\ref{eq:budget}) is \GoldBudgetNinety\%
of flagged records at \PrecTargetNinety\% precision against \GoldBudgetNinetyInSample\% in
sample, so the optimism of in-sample selection is \GoldBudgetOptimism{} percentage points at
this pooled level. The precision that budget actually delivers out of fold is
\GoldBudgetPrecHeldOut{}, with a lower confidence bound of \GoldBudgetPrecLCB{} computed at the
effective accepted count and therefore approximate.

The pooled budget is a mixture over variables rather than a policy, and Table~\ref{tab:gold}
gives the figure an agency would act on. Held out, the per-variable budgets run from
\GoldBudgetPerVarLo\% to \GoldBudgetPerVarHi\%, the largest being \GoldBudgetMaxVar{}, and
\GoldNZeroBudget{} of \GoldNVars{} variables already meet the \PrecTargetNinety\% target with
no review at all. The human-referenced budget is below the coded-field budget for all
\BudgetLowerN{} comparable variables, so the coded-field budgets of Table~\ref{tab:selective}
are an upper bound on human effort. The curve of Figure~\ref{fig:goldprec} in
Appendix~\ref{app:coders} pools the reference set across variables and is read as a summary of
that mixture rather than as an operating point.

One qualification applies to every accuracy and budget figure in this subsection. The
\ExclN{} pairs the coders left unclear or disagreed on are excluded from all of them, so each
is conditional on the pairs the coders found clear. Treating every excluded pair as a model
error lowers pooled $F_1$ to \ExclFoneError{} and raises the pooled budget to
\ExclBudgetError\%, treating them all as model-correct raises $F_1$ to \ExclFoneCorrect{}, and
routing them to review regardless of probability raises the budget to \ExclBudgetReview\%.
Those are the bounds, and Table~\ref{tab:exclusions} in Appendix~\ref{app:coders} reports the
exclusions by variable and by probability bin.

\begin{table}[pos=!htbp]
\centering
\rmfamily\small
\setlength{\tabcolsep}{3.0pt}
\caption{Selective prediction and the human-review budget (agreement).}
\label{tab:selective}
\begin{threeparttable}
\begin{tabular}{lrrr>{\raggedleft\arraybackslash}p{34.2pt}>{\raggedleft\arraybackslash}p{34.2pt}>{\raggedleft\arraybackslash}p{34.2pt}>{\raggedleft\arraybackslash}p{43.4pt}>{\raggedleft\arraybackslash}p{43.4pt}>{\raggedleft\arraybackslash}p{84.1pt}}
\toprule
\jevhead{variable} & \jevhead{AURC} & \jevhead{excess} & \jevhead{prec.} & \jevheadp{flagged \%} & \jevheadp{$H$(0.90) \%} & \jevheadp{$H$(0.95) \%} & \jevheadp{human $H$(0.90) \%\newline held out} & \jevheadp{human $H$(0.90) \%\newline in sample} & \jevheadp{review decisions/yr\newline at 90\% precision} \\
\midrule
alcohol\_involved & 0.0163 & 0.0159 & 0.712 & 3.635 & 100.0 & 100.0 & 5.1 & 3.9 & 20{,}262 \\
drug\_involved & 0.0018 & 0.0018 & 0.464 & 0.562 & 100.0 & 100.0 & 0.0 & 0.0 & 3{,}134 \\
fatigue & 0.0036 & 0.0035 & 0.653 & 1.399 & 100.0 & 100.0 & 0.1 & 0.0 & 7{,}798 \\
animal\_involved & 0.0014 & 0.0014 & 0.848 & 2.016 & 7.0 & 100.0 & 0.0 & 0.0 & 792 \\
phone\_use & 0.0025 & 0.0024 & 0.340 & 1.038 & 100.0 & 100.0 & 0.2 & 0.0 & 5{,}788 \\
unbelted & 0.0136 & 0.0132 & 0.209 & 0.596 & 100.0 & 100.0 & 63.6 & 71.0 & 3{,}323 \\
hydroplane & 0.0012 & 0.0011 & 0.719 & 2.352 & 100.0 & 100.0 & 0.0 & 0.0 & 13{,}114 \\
wrong\_way & 0.0016 & 0.0014 & 0.259 & 1.427 & 100.0 & 100.0 & 72.2 & 70.9 & 7{,}958 \\
medical\_episode & 0.0014 & 0.0013 & 0.398 & 1.055 & 100.0 & 100.0 & 14.4 & 18.9 & 5{,}880 \\
\bottomrule
\end{tabular}
\begin{tablenotes}[flushleft]\small\rmfamily
\item \textit{Note:}~Agreement against the incomplete coded fields rather than accuracy; see the note to Table~\ref{tab:calibration}. $H(\pi)$ is the share of \emph{flagged positives} a human must review to reach precision $\pi$; accuracy-based budgets are degenerate because the positive class is rare (Section~\ref{sec:res-accuracy}). The two human columns are measured against the human reference set. The first picks the threshold on one half of the narratives and certifies it on the other, and the second does both on the same half, so their difference is the optimism of in-sample selection. The last column counts \emph{variable-level review decisions} per year rather than narratives.
\end{tablenotes}
\end{threeparttable}
\end{table}

\begin{figure}[pos=!htbp]\centering
\includegraphics[width=\textwidth]{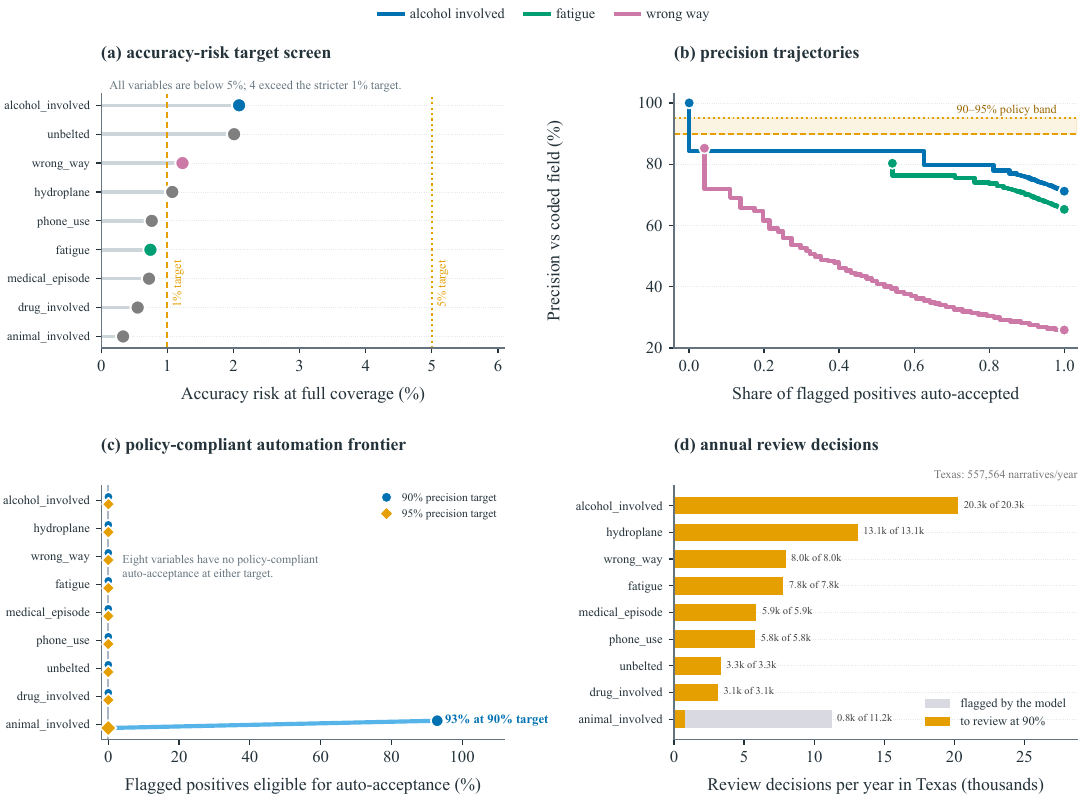}
\caption{Selective prediction against the coded reference. (a) Accuracy-based selective risk at
full coverage for each variable against the two conventional targets, with every variable
below the \RiskTargetFive\% line and four above the \RiskTargetOne\% line. (b) Precision
against the share of flagged positives auto-accepted for alcohol involvement (blue), fatigue
(green) and wrong-way travel (pink), with the \PrecTargetNinety{} to
\PrecTargetNinetyFive\% policy band shaded. The curves are steps, each interval taking the
precision of the lower attainable threshold, because any coverage inside it means accepting a
block of records tied at one grid value. No curve rests inside the band. (c) The share of
flagged positives a variable could auto-accept while holding each target, the complement of
the budget of Eq.~(\ref{eq:budget}), zero for eight of nine variables at the lower target and
for all nine at the higher. (d) The variable-level review decisions this implies for Texas
each year (orange) against the volume the model flags (gray). Precision here inherits the
omissions of the coded fields.}
\label{fig:riskcoverage}
\end{figure}

\subsection{Frontier comparison and value added over the coded fields}
\label{sec:res-frontier}

Two frontier generative models were scored on the same \CoderNPairs{} narrative-variable pairs
against the same human labels with the same weights and metric code
(Table~\ref{tab:frontier}). Calibration turns out to be a property of the individual model
rather than of the paradigm, and that is the result which generalizes. The first frontier arm
fails in the same direction as the typed model, rejected by the Spiegelhalter test at
$z = \FableZ$ and carrying a slope of \FableSlope{} against \GoldSlope{}, so both over-state
prevalence. The second behaves differently, with a slope of \GPTSlope{}, a statistic of $\GPTZ$
close to zero, and a calibration error of \GPTECE{} against \FableECE{} and \GoldECE{} for the
other two. It elicits only \GPTDistinctP{} distinct values against \FableDistinctP{} for the
first arm, so a coarse output grid is not by itself an obstacle to calibration. Calibration
cannot therefore be inferred from the vendor or the interface, and the best-calibrated arm is
not the most accurate one, which is why the audit cannot be skipped by buying a better
model.

On accuracy the frontier arms do not agree with each other. The first reaches an $F_1$ of
\FableFone{} against \GoldFone{} for the typed model, a paired difference of \DiffJevFable{}
with an interval of [\DiffJevFableLo{}, \DiffJevFableHi{}] over \PairedN{} judgments, which is a
material advantage. The second reaches \GPTFone{} and is not distinguishable from the typed
model, because its paired difference of \DiffJevGPT{} has an interval of
[\DiffJevGPTLo{}, \DiffJevGPTHi{}] that straddles zero, and it gets there by trading recall for
precision, \GPTRecall{} against \GPTPrecision{}. The two frontier models differ from each other
by \DiffFableGPT{}, more than either differs from the typed model, so the choice among frontier
models matters more than the choice between paradigms. Which variables a model fails on is
idiosyncratic. At least one frontier arm recovers \FrontierNRecovered{} of the typed model's
\GoldNWeak{} weak variables, the second arm falls below an $F_1$ of \FrontierWeakThreshold{} on
\GPTWeakVars{}, where the typed model is sound on \GPTNWeakJevSound{}, and the first is weak
only on \FableWeakVars{}. This is the attribute-specific pattern that
\citet{bharati2026benchmarking} report across six models, reproduced here against human labels
rather than coded fields.

Cost is reported separately from accuracy, because the two cannot be measured on the same arms
here. The benchmarked arms ran outside this pipeline and their outputs carry no token counts, so
no statement in this paper attaches a cost to either of them. What can be priced is a scenario,
applying the list prices of three commercially available tiers to the token counts this study
measured. A mid-tier model would cost \$\FrontierCostSonnet{} per thousand narratives against
\$\CostPerThousand{} for the typed model, a factor of \FrontierCostRatio{}, with the largest
tier at \FrontierCostRatioOpus{} and the smallest at \FrontierCostRatioHaiku{}. Projected to the
whole corpus that is about \$\CostWholeCorpusSonnet{} against \$\CostWholeCorpus{}, which
separates a routine analysis from one that does not happen. The scenario prices tiers rather
than the arms of Table~\ref{tab:frontier}, and the two must not be read as one curve.

\begin{table}[pos=!htbp]
\centering
\rmfamily\small
\setlength{\tabcolsep}{4.0pt}
\caption{Jev against frontier generative models on the same human reference labels.}
\label{tab:frontier}
\begin{threeparttable}
\begin{tabular}{lrrrrrrrrrrr}
\toprule
\multicolumn{2}{c}{} & \multicolumn{5}{c}{\jevhead{Agreement with human labels}} & \multicolumn{4}{c}{\jevhead{Calibration}} & \multicolumn{1}{c}{} \\
\cmidrule(lr){3-7}\cmidrule(lr){8-11}
\jevhead{model} & \jevhead{$n$} & \jevhead{prec.} & \jevhead{rec.} & \jevhead{$F_1$} & \jevhead{} & \jevhead{$\kappa$} & \jevhead{ECE} & \jevhead{Brier} & \jevhead{slope} & \jevhead{$z$} & \jevhead{prob.} \\
\midrule
Jev 1.13 & 2,416 & 0.902 & 0.915 & 0.908 & \jevbar{0.541}{0.459} & 0.906 & 0.0231 & 0.0045 & 1.63 & \jevflag{-5.3} & native \\
claude-fable-5-1 & 2,416 & 0.959 & \jevbest{0.976} & \jevbest{0.967} & \jevbar{0.837}{0.163} & \jevbest{0.966} & 0.0282 & \jevbest{0.0025} & 2.55 & \jevflag{-6.0} & elicited \\
gpt-5.6-sol & 2,416 & \jevbest{0.970} & 0.814 & 0.885 & \jevbar{0.427}{0.573} & 0.883 & \jevbest{0.0027} & 0.0056 & 0.98 & \jevflag{-0.3} & elicited \\
\bottomrule
\end{tabular}
\begin{tablenotes}[flushleft]\small\rmfamily
\item \textit{Note:}~All arms answer the same narratives and variables under the same criteria text, scored against the same labels with the same code. \jevbest{Bold} is the leader per metric and bars scale $F_1$ from 0.80. The last column states whether the model returns a probability natively or was asked to state one. Every $z$ is \jevflag{red} because every arm is rejected, and they fail in opposite directions, since a slope above one with $z<0$ over-states prevalence and a slope below one with $z>0$ under-states it. Cost is not measured here, and Figure~\ref{fig:frontier} prices a comparable tier at list rates. No arm is held out from the reference set.
\end{tablenotes}
\end{threeparttable}
\end{table}

Crash-level agreement with the coded fields runs from \Agreemin\% to \Agreemax\%, but at these
base rates agreement is dominated by shared negatives, so Cohen's kappa is the informative
summary (Table~\ref{tab:discrepancy}). It ranges from \Kappamin{} for \KappaminVar{} to
\Kappamax{} for \KappamaxVar{}, from slight to almost perfect on the Landis and Koch scale.
Figure~\ref{fig:agreement}c plots how often each source flags a factor the other does not. Seven
of the nine variables fall below the diagonal, so the narrative asserts these factors more often
than the coded fields record them, and alcohol involvement sits on it. Restraint use is the one
variable above, because the restraint field records an unbelted occupant in many crashes whose
narrative says nothing about seat belts. That pattern is what the study was built to detect, and
it is also why the coded fields cannot serve as an accuracy reference, since a narrative-only
flag may be a model error or a coded omission.

Which of the two it is can only be settled on the human reference set. Its narratives that the
model flags and the coded field records as negative are \NarrOnlyShare\% of all its flags,
\NarrOnlyN{} judgments over \NarrOnlyNVars{} variables. The coders confirm \NarrOnly\% of them,
with an interval of \NarrOnlyLo{} to \NarrOnlyHi\%, against \CodedConf\% where the coded field
agrees. Rather more than half of the narrative-only region is therefore a factor the narrative
states and the coded field does not record, which is a statement about what the two sources
contain rather than about what happened in the crash. That rate divides sharply by the model's
accuracy on the variable. Restricted to variables the model codes well it is \NarrOnlyStrong\%
on \NarrOnlyStrongN{} judgments, with an interval of \NarrOnlyStrongLo{} to
\NarrOnlyStrongHi\%, and restricted to the weak variables it is \NarrOnlyWeak\% on
\NarrOnlyWeakN{}, with an interval of \NarrOnlyWeakLo{} to \NarrOnlyWeakHi\%. Where the model is accurate a narrative-only flag is usually a factor the
coded field omits, and where it is weak the flag is usually the model, so the value-added
reading is licensed per variable by that variable's accuracy.

Against the same coded reference the typed model recovers more of the coded positives than
keyword rules do, with precision moving in both directions (Figure~\ref{fig:agreement}b,
Table~\ref{tab:baseline}). Recall rises for every variable, by the largest margins for
hydroplaning and wrong-way travel, while precision rises for animal involvement, phone use and
alcohol involvement and falls elsewhere, most for restraint use. The paired McNemar test of
Eq.~(\ref{eq:mcnemar}) rejects equal correctness for every variable except fatigue. Keyword
rules therefore remain a strong and nearly free baseline where the vocabulary is closed, which
matches the Arkansas finding of \citet{bharati2026benchmarking}, and the typed model's advantage
concentrates in variables whose expression is paraphrastic. Figure~\ref{fig:frontier} places the
two routes on the cost axis of the scenario above, with keyword rules at a macro-averaged $F_1$
near one half and the typed model somewhat above it at a price two orders of magnitude
higher.

\begin{table}[pos=!htbp]
\centering
\rmfamily\small
\setlength{\tabcolsep}{4.0pt}
\caption{Discrepancy taxonomy against coded fields (agreement).}
\label{tab:discrepancy}
\begin{threeparttable}
\begin{tabular}{lrrrrrrr}
\toprule
\jevhead{variable} & \jevhead{$n$} & \jevhead{agreement \%} & \jevhead{$\kappa$} & \jevhead{both} & \jevhead{narr.\ only} & \jevhead{code only} & \jevhead{neither} \\
\midrule
alcohol\_involved & 150,000 & 97.92 & 0.702 & 3,882 & 1,571 & 1,556 & 142,991 \\
drug\_involved & 150,000 & 99.45 & 0.484 & 391 & 452 & 374 & 148,783 \\
fatigue & 150,000 & 99.26 & 0.707 & 1,369 & 729 & 387 & 147,515 \\
animal\_involved & 150,000 & 99.67 & 0.910 & 2,565 & 459 & 37 & 146,939 \\
phone\_use & 150,000 & 99.24 & 0.477 & 529 & 1,028 & 118 & 148,325 \\
unbelted & 150,000 & 97.99 & 0.103 & 187 & 707 & 2,307 & 146,799 \\
hydroplane$^{\dagger}$ & 150,000 & 98.93 & 0.753 & 2,536 & 992 & 619 & 145,853 \\
wrong\_way & 150,000 & 98.77 & 0.371 & 554 & 1,587 & 256 & 147,603 \\
medical\_episode & 150,000 & 99.28 & 0.534 & 629 & 953 & 131 & 148,287 \\
\bottomrule
\end{tabular}
\begin{tablenotes}[flushleft]\small\rmfamily
\item \textit{Note:}~Agreement against the incomplete coded fields rather than accuracy; see the note to Table~\ref{tab:calibration}. The dagger marks a coded field that is a proxy rather than a direct match. ``Narrative only'' is not an error rate, because a factor stated in the narrative and absent from the coded field is what this study sets out to find.
\end{tablenotes}
\end{threeparttable}
\end{table}

\begin{table}[pos=!htbp]
\centering
\rmfamily\small
\setlength{\tabcolsep}{4.0pt}
\caption{Keyword baseline versus Jev on the same coded reference (agreement).}
\label{tab:baseline}
\begin{threeparttable}
\begin{tabular}{lrrrrrrr}
\toprule
\multicolumn{1}{c}{} & \multicolumn{3}{c}{\jevhead{Keyword rules}} & \multicolumn{3}{c}{\jevhead{Jev 1.13}} & \multicolumn{1}{c}{} \\
\cmidrule(lr){2-4}\cmidrule(lr){5-7}
\jevhead{variable} & \jevhead{prec.} & \jevhead{rec.} & \jevhead{$\kappa$} & \jevhead{prec.\ } & \jevhead{rec.\ } & \jevhead{$\kappa$\ } & \jevhead{McNemar $p$} \\
\midrule
medical\_episode & 0.425 & 0.642 & 0.509 & 0.398 & 0.828 & 0.534 & $<$0.001 \\
drug\_involved & 0.576 & 0.284 & 0.378 & 0.464 & 0.511 & 0.484 & $<$0.001 \\
hydroplane & 0.846 & 0.356 & 0.495 & 0.719 & 0.804 & 0.753 & $<$0.001 \\
fatigue & 0.675 & 0.701 & 0.684 & 0.653 & 0.780 & 0.707 & 0.960 \\
wrong\_way & 0.388 & 0.246 & 0.298 & 0.259 & 0.684 & 0.371 & $<$0.001 \\
animal\_involved & 0.751 & 0.860 & 0.798 & 0.848 & 0.986 & 0.910 & $<$0.001 \\
alcohol\_involved & 0.693 & 0.692 & 0.681 & 0.712 & 0.714 & 0.702 & $<$0.001 \\
phone\_use & 0.238 & 0.672 & 0.348 & 0.340 & 0.818 & 0.477 & $<$0.001 \\
unbelted & 0.589 & 0.034 & 0.063 & 0.209 & 0.075 & 0.103 & $<$0.001 \\
\bottomrule
\end{tabular}
\begin{tablenotes}[flushleft]\small\rmfamily
\item \textit{Note:}~Agreement against the incomplete coded fields rather than accuracy; see the note to Table~\ref{tab:calibration}. McNemar tests paired per-narrative correctness. The frontier arm is scored against human labels in Table~\ref{tab:frontier}.
\end{tablenotes}
\end{threeparttable}
\end{table}

\begin{figure}[pos=!htbp]\centering
\includegraphics[width=\textwidth]{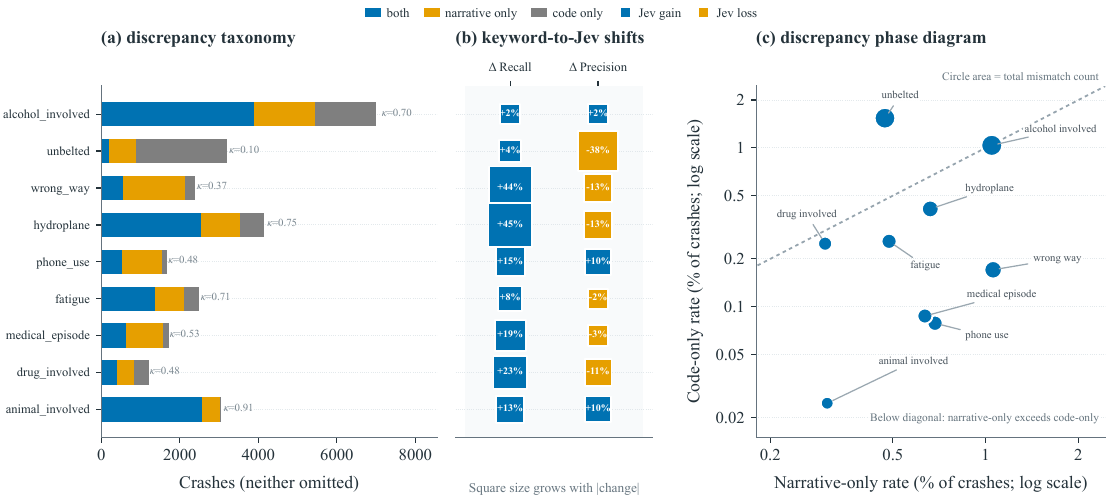}
\caption{Agreement with the coded fields. (a) Discrepancy taxonomy per variable, with crashes
flagged by both sources (blue), by the narrative only (orange) and by the code only (gray),
and Cohen's kappa beside each bar. (b) The change in recall and in precision against the same
coded reference when keyword rules are replaced by the typed model; the square grows with the
size of the change, gains in blue and losses in orange. (c) The rate at which each source
flags a factor the other does not, logarithmic on both axes, with the dashed diagonal marking
equal rates and marker area the total mismatch count. Narrative-only cases are candidates for
information the coded fields omit rather than errors.}
\label{fig:agreement}
\end{figure}

\begin{figure}[pos=!htbp]\centering
\includegraphics[width=0.85\textwidth]{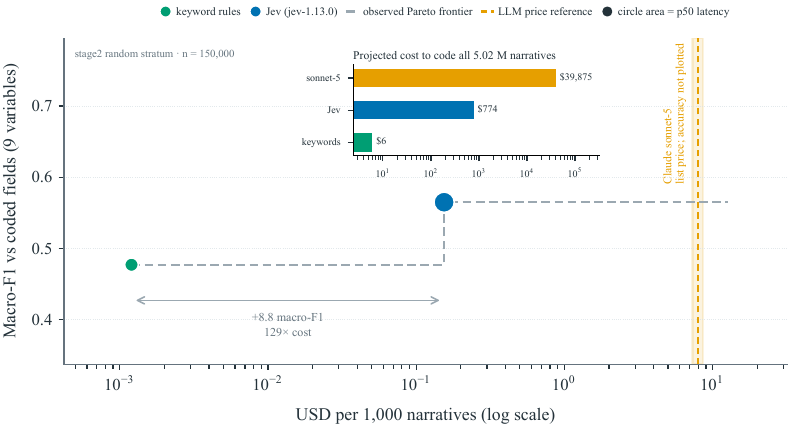}
\caption{Cost per thousand narratives against macro-averaged $F_1$ with respect to the coded
fields over the same nine variables, for keyword rules (green) and the typed model (blue),
with marker area proportional to median latency and the dashed step the observed Pareto
frontier. The vertical orange line is the list price of a mid-tier frontier model applied to
the measured token counts. Frontier-arm accuracy is not plotted, because it was measured
against the human reference set rather than the coded fields. The inset projects the cost of
coding the full corpus by each route, which none of them coded.}
\label{fig:frontier}
\end{figure}

\subsection{What changes in a safety diagnosis when narrative variables are added}
\label{sec:res-downstream}

The preceding subsections establish that the model reads narratives faithfully for most
variables and that the coded contributing-factor fields omit a good deal of what the narratives
state. Neither result says how much a transportation agency's picture of its own crashes would
change if the narrative variables were used alongside the coded fields, and that question is
the one a safety office actually asks. This subsection answers it on the Stage-2 random
stratum, which contains \DownNSeverity{} injury or fatal crashes of which \DownNFatal{} are
fatal, using only data already produced and no further model calls.

The quantity is the number of injury and fatal crashes that would be attributed to each factor
under each of two sources. Under the coded fields alone it is the count of such crashes whose
coded field records the factor. Under the two sources together it is that count plus the
expected number of crashes whose coded field is negative and whose narrative states the factor,
obtained by summing the recalibrated probability of Section~\ref{sec:res-calibration} over
those crashes. Summing a calibrated probability rather than counting model flags is what makes
the second term an estimate of a population total, because a calibrated probability already
carries the model's error rate inside it. The recalibration map is fitted per variable wherever
the reference set supports one. The interval on the second term comes from resampling
the reference narratives and refitting the map, since the map is estimated on
\GoldNNarratives{} narratives and the stratum holds \NstageTwoRandom{}, so the map is where the
uncertainty lives.

Table~\ref{tab:downstream} gives the result for all \DownNVars{} mapped factors, scaled to
Texas per year. The coded fields attribute \DownCodedTotal{} injury and fatal crashes per year
across the nine factors, and the narrative variables add \DownAddedTotal{} more. The relative
increase runs from \DownRelMin\% for \DownRelMinVar{} to \DownRelMax\% for \DownRelMaxVar{},
with a median of \DownRelMedian\%. The largest single addition in absolute terms is
\DownAddedMax{} crashes per year for \DownAddedMaxVar{}. The largest in relative terms is
\DownRelMaxVar{}, where the coded fields record \DownTopCoded{} injury and fatal crashes per
year and the narratives add \DownTopAdded{} more, with an interval of \DownTopAddedLo{} to
\DownTopAddedHi{}. A factor that the coded fields make the smallest of the nine therefore
becomes comparable to several of the others, which is a change in the diagnosis rather than
only in the counts. Ranking the nine factors by attributed crashes, \DownNRankMove{} of them
change place when the narrative variables are added.

The confirmation rates of Section~\ref{sec:res-frontier} decide how much of that change can be
believed, and they differ enough across factors that the answer is not uniform. For
\DownNConfirmHigh{} of the \DownNVars{} factors a coder confirmed at least
\DownConfirmHighThreshold{} of the narrative-only flags, and \DownConfirmHiVar{} is the
strongest at \DownConfirmHi{}, so for those the added count is evidence about the crashes rather
than about the model. For \DownNConfirmLow{} of them the coder confirmed fewer than half, and
\DownConfirmLoVar{} is the weakest at \DownConfirmLo{}, so most of what the model adds there is
model error and the increase is not a finding. The pattern is the one
Section~\ref{sec:res-accuracy} would predict, since the factors with low confirmation are the
factors the model codes poorly. A safety office can therefore adopt the narrative variables one
factor at a time, on evidence it already holds, rather than as a single decision. Read that
way, adding calibrated narrative variables to the coded fields raises the attributed burden of
the well-measured factors by between \DownRelMin\% and \DownRelMax\%. \DownRelMaxVar{} in
particular moves from a minor entry in the coded record to a factor of a size that would
ordinarily attract a countermeasure program.

\begin{table}[pos=!htbp]
\centering
\rmfamily\small
\setlength{\tabcolsep}{3.0pt}
\caption{Injury and fatal crashes attributed to each factor, from the coded fields alone and with calibrated narrative variables added.}
\label{tab:downstream}
\begin{threeparttable}
\begin{tabular}{lr>{\raggedleft\arraybackslash}p{32.8pt}>{\raggedleft\arraybackslash}p{68.2pt}>{\raggedleft\arraybackslash}p{51.4pt}>{\raggedleft\arraybackslash}p{40.9pt}>{\raggedleft\arraybackslash}p{78.1pt}}
\toprule
\jevhead{factor} & \jevhead{coded field} & \jevheadp{coded\newline per year} & \jevheadp{added by narratives\newline per year} & \jevheadp{95\% interval} & \jevheadp{increase \%} & \jevheadp{narrative-only flags\newline confirmed} \\
\midrule
alcohol\_involved & coded\_alcohol & 8{,}531 & 1{,}784 & [1{,}530, 2{,}649] & 21 & 0.71 \\
medical\_episode & coded\_medical & 1{,}918 & 1{,}644 & [1{,}371, 2{,}150] & 86 & 0.60 \\
phone\_use & coded\_phone & 732 & 1{,}334 & [1{,}091, 1{,}522] & 182 & 0.90 \\
wrong\_way & coded\_wrongway & 1{,}543 & 1{,}248 & [924, 2{,}791] & 81 & 0.29 \\
hydroplane$^{\dagger}$ & coded\_hydro\_proxy & 2{,}959 & 1{,}154 & [932, 1{,}437] & 39 & 0.88 \\
unbelted & coded\_unbelted & 6{,}364 & 1{,}049 & [800, 1{,}423] & 16 & 0.15 \\
drug\_involved & coded\_drug & 1{,}833 & 998 & [817, 1{,}317] & 54 & 0.52 \\
fatigue & coded\_fatigue & 2{,}476 & 910 & [792, 1{,}020] & 37 & 0.89 \\
animal\_involved & coded\_animal & 1{,}702 & 625 & [509, 750] & 37 & -- \\
\bottomrule
\end{tabular}
\begin{tablenotes}[flushleft]\small\rmfamily
\item \textit{Note:}~Counts are estimated on the Stage-2 random stratum and scaled to Texas per year. The added column is the expected number of injury or fatal crashes whose coded field is negative and whose narrative states the factor, summed over the recalibrated probability of Section~\ref{sec:res-calibration}, so it is a calibrated total rather than a count of flags. Intervals resample the reference narratives and refit the map. The last column is the share of narrative-only flags a coder confirmed, from Table~\ref{tab:gold}, which disciplines the added column rather than feeding it. The dagger marks a coded field that is a proxy rather than a direct match.
\end{tablenotes}
\end{threeparttable}
\end{table}

%% =====================================================================================
\section{Discussion}
\label{sec:discussion}

The central answer of this study is that adding calibrated narrative variables to a state's
coded contributing-factor fields changes the safety diagnosis those fields support. Across the
\DownNVars{} factors that carry a coded counterpart, the narratives add \DownAddedTotal{}
injury and fatal crashes per year to the \DownCodedTotal{} the coded fields record, a median
increase of \DownRelMedian\%, and \DownNRankMove{} of the nine factors change rank. The
strongest single case is \DownRelMaxVar{}, which the coded fields make the smallest of the nine
and which the narratives raise by \DownRelMax\%. That result rests on two prior ones that this paper also establishes. A typed decision model
can code narratives at a cost low enough for a state to run the schema over a population rather
than a sample. Its probabilities are usable for ranking but not for counting until they are
recalibrated against a few thousand human judgments. Once that
recalibration is in hand, the human review a narrative-derived variable needs can be stated per
variable and per year, which is what makes the added counts actionable rather than merely
interesting. No comparable study has priced that change against the coded record. Safety
diagnoses built on coded contributing factors alone, of the kind the crash data-quality
literature describes \citep{imprialou2019quality}, therefore understate the factors whose
evidence sits mainly in prose, and the size of the understatement is measurable rather than
presumed. The closest prior study is the Arkansas
benchmark of \citet{bharati2026benchmarking}, which evaluated six frontier generative models
on some four thousand matched fatal crashes over six attributes the database already codes. That
study found a keyword baseline comparable to the best model, with differences across
attributes larger than differences across models. The present results reproduce both
findings on a different corpus and a different model class. Keyword rules remain a strong and
nearly free baseline for closed-vocabulary variables, and the spread of kappa across
variables against the coded fields is wider than any difference between methods. What this
study adds is the layer that agreement measures cannot supply. It measures the probabilities
rather than the labels, and it quantifies how much the administrative reference understates the
model's fidelity to the narrative. It then converts the result into a review budget defined over the records an agency
would read, which is the operational form of the human review that the Arkansas authors
recommended. The paragraphs that follow take each of the study's findings in turn and place
it against the closest published work.

The calibration finding sits between two literatures that have not been connected. Work on verbalized confidence in generative models reports over-confidence \citep{xiong2024express} and shows that calibration depends on the elicitation format \citep{kadavath2022know,tian2023just}. Neither of the frontier arms audited
here behaved in that way. The second arm elicited its probabilities on only \GPTDistinctP{}
distinct values and was nonetheless the best calibrated of the three, at a slope of \GPTSlope{} and a
statistic of $\GPTZ$, which says that a coarse verbalized scale is not by itself a source of
over-confidence. The typed model and the first frontier arm failed in the other direction
instead, with slopes above one and a level too high. That is the pattern the clinical prediction-model literature describes as a failure
of calibration in the large that survives a reasonable slope
\citep{van2016a,steyerberg2010assessing}. The hierarchy of \citet{van2016a} places mean calibration first, then weak calibration, which
requires an intercept of zero and a slope of one together, then moderate and strong
calibration. The typed model fails at the first level, with a mean probability of
\GoldExpectedP\% against a prevalence of \GoldBaseRate\%, and it fails at the second as well,
because its slope is \GoldSlope{} rather than one. Describing the remedy as a one-parameter
correction would therefore be wrong. The Platt map of Eq.~(\ref{eq:platt}) fits an intercept
and a slope, which is what a weak-calibration failure requires, and after it the pooled slope
is \RecalPlattSlope{}. The result also extends the observation of \citet{zhou2024reliable} that larger and more
instructable models become less reliable, because the more accurate frontier arm was the one
rejected on calibration. Its slope of \FableSlope{} lies further from one than the typed
model's, even though its accuracy is higher. The arm that is best calibrated here, at a slope
of \GPTSlope{} and a calibration error of \GPTECE{}, is neither the most accurate arm nor the
one with the finest output grid. Calibration is therefore a property of the model and the
elicitation together rather than of size, price or interface, and an audit against an external
reference is the only way to learn both its presence and its direction.

The reference-noise finding connects the injury-surveillance auto-coding line with the crash
data-quality literature. The injury-surveillance auto-coding line classified narratives into cause codes with Bayesian and other learners \citep{lehto2009bayesian,bertke2016comparison} and routed low-confidence records to human review \citep{marucciwellman2015practical,marucciwellm2017classifyin},
and the rare-category study of \citet{nanda2018improving} found that filtering rather than
more training data rescued the rare classes. The present study inverts the direction of comparison, scoring the narrative-derived
probability against the coded field and then against human judgment of the narrative. The
result is not the one the design anticipated, and it is worth stating plainly. The two
references give almost the same calibration error, a median ratio of \ECEGapMedian{} with a
range of \ECEGapMin{} to \ECEGapMax{}, so a coded reference does not flatter the model's
probabilities. What an incomplete reference penalizes instead is the agreement statistic.
Cohen's kappa against the coded field falls below the weighted kappa against human judgment for
all
\KappaGapNPositive{} comparable variables, by a median of \KappaGapMedian{}, and
\KappaGapMaxVar{} reads as \KappaGapMaxCoded{} against the coded record and
\KappaGapMaxHuman{} against human judgment of the same narratives. The mechanism is omission in
the coded fields, whose incompleteness the crash data-quality literature lists among its most
serious problems \citep{imprialou2019quality}, and which this study measures directly as the
excess of narrative-only over code-only discrepancies. Under-reporting studies have measured
the same omission at the level of whole crashes and injured people
\citep{elvik1999incomplete,alsop2001underreporting,watson2015underreporting}.
Lemma~\ref{lem:noise} gives that structure its formal statement, and its empirical content
here is narrower than the lemma allows. Omission moves the agreement statistic rather than the
calibration statistic, because a probability close to the rate at which a factor is stated
stays close to it whether or not a clerk recorded the factor. \citet{arteaga2025a} used a language model to expose under-reporting from the same
kind of narrative evidence. The confirmed rate of \NarrOnly\% in the narrative-only cell
reported here measures how much of that apparent under-reporting is a factor the coded field
omits rather than a model error. Its split by variable accuracy shows that the answer depends
on the model as much as on the field, and Section~\ref{sec:res-downstream} converts it into
crashes per year.

The resolution floor is a general property of any model that reports probabilities on a
discrete grid, and the calibration literature has not treated it because that literature is
written for continuous outputs. Reliability diagrams and their binning have received close
attention, from the consistency bands of \citet{brocker2007diagrams} to the binning-free
constructions of \citet{dimitriadis2021stable} and the regression view of
\citet{gneiting2023regression}. \citet{kumar2019verified} showed that verifiable calibration
error requires a discretized output, and a model that discretizes its own output supplies that
property for free. The exact error of Eq.~(\ref{eq:ecegrid}) therefore needs no binning
parameter, and equal-mass binning of the kind recommended by \citet{nixon2019measuring} stops
working because the mass concentrates on a handful of attainable values.
Proposition~\ref{prop:floor} adds the constraint that a discrete grid imposes. Wherever the
smallest positive value exceeds a variable's base rate, no assignment of grid values is
calibrated and the mean probability over-states prevalence by at least the gap. That check
takes one line, and any analyst using such a model should run it before trusting summed
probabilities. The floor was invisible in the developer's documentation, and it was found only
because the audit inspected the distribution of returned values rather than their accuracy
alone.

The review budget over flagged records connects selective prediction to rare-event screening.
The risk-coverage machinery of \citet{elyaniv2010foundations} and \citet{geifman2017selective}
assumes that abstention trades coverage against error, and on a balanced problem it does. On
a variable with a base rate of a few percent, a model that answers everything already has an
error rate below conventional targets, purely because the negative class dominates. The
analysis then certifies a system nobody should trust, which is the objection the
precision-recall literature raised against error rates on imbalanced data
\citep{davis2006prroc,saito2015prplot}. Pointing the construction at the decision actually
being made, which is whether to read a flagged record, restores the trade-off. Nobody reviews
the records a model calls negative, so defining coverage over flagged positives and risk as
one minus precision produces a number an agency can staff against. The difference between the
coded-field budget, which saturates for eight of nine variables, and the human-referenced
budget of \GoldBudgetNinety\% is the reference effect seen in the agreement statistics, now
expressed in hours. A budget selected and certified on the same labels is optimistic, which is
why the value reported here is cross-fitted and carries a lower bound on the precision it
delivers. That bound is a finite-sample statement about this reference set rather than a
guarantee. The distribution-free risk-control literature supplies guarantees of the stronger
kind, with risk-controlling prediction sets \citep{bates2021distribution}, conformal risk
control \citep{angelopoulos2024conformal} and the Learn then Test framework for selecting
among many candidate thresholds \citep{angelopoulos2025learn}. Applying that machinery here
would require an exchangeable calibration sample and a fixed error level chosen in advance,
and it would return a threshold whose risk is controlled with high probability rather than an
estimate with an interval. Section~\ref{sec:res-accuracy} reports the latter, and the
construction is not specific to crash narratives, since it applies to any screening task in
which the positive class is rare and the cost is concentrated in confirming positives.

The schema and the cost model transfer to other states more readily than the reference
mapping does. The schema asks about the narrative rather than about Texas, so the same
questions can be posed to any state's narrative field, with the criteria text serving as the
complete specification of each variable for a model, a coder or a reviewer alike. The cost
model transfers more strongly still, because it is dominated by schema size rather than by
narrative length. The per-narrative cost of this schema is therefore approximately the same
in any corpus of comparable prose length, and the break-even fraction of
Eq.~(\ref{eq:twostage}) is a property of the schema. What does not transfer is the coded-field
reference. Every state codes contributing factors differently, and the mapping of
Table~\ref{tab:coded-map} had to be corrected twice against the actual value sets before it
was usable. The one-sided error that Lemma~\ref{lem:noise} describes will also have a
different rate in every field of every state. An agency adopting the pipeline should
therefore expect to reuse the schema and to refit the recalibration map on its own human
judgments. It should also expect to rebuild the reference mapping from its own data
dictionary, and the released code is organized so that those three steps are separable.

Transfer is a hypothesis here rather than a finding, because this study covers one state, one
model version and one schema, and it is stated with the conditions under which it would fail.
The cost model would fail where narratives are much longer than the \MedianNarrativeChars{}
characters of the median here, since the text term would then stop being small beside the
schema term. The schema would fail where officers write to a different convention, because the
criteria text was written against Texas shorthand and a construct such as an intentional act is
recognized through the phrases a reporting culture happens to use. The recalibration map would
fail wherever the prevalence differs, since a map fitted at one base rate moves the
probabilities toward that rate and not toward another. The accuracy results would fail on any
later model version, because the audit measures a model and a schema together rather than a
capability. Each of these is testable in one state with a few hundred labeled narratives, which
is the cost this study reports, and until that test is run the transfer claim should be read as
a design expectation.

%% =====================================================================================
\section{Conclusion}
\label{sec:conclusion}

This study set out to answer two questions that crash-narrative research has left to the
reader, namely how much human checking a narrative-derived variable needs and what changes in a
safety diagnosis when such variables are added to the coded record. Answering them required
treating narrative coding as a set of typed, gated decisions with probabilities attached and
auditing the probabilities against an external reference instead of assuming them. It also
required running that decomposition at scale, with a first-stage screen over \StageOneCoded{}
narratives and full coding of \NstageTwo{} rather than a few thousand records. The typed model
coded those \NstageTwo{} Texas narratives with a \NQuestions{}-question schema at
\$\CostPerThousand{} per thousand, which projects to \$\CostWholeCorpus{} for the full corpus
of \Ncorpus{} narratives at the same rate. The cost model showed that spend is governed by the
size of the question schema rather than by the length of the text, which makes batching
questions the dominant efficiency lever and sets a break-even fraction of \BreakEvenFraction{}
below which two-stage screening pays. The model returns probabilities
on a two-decimal grid whose floor exceeds the coded-field base rate of \NbelowFloor{} target
variables,
so for those variables no allocation of grid values can be calibrated and the mean
probability is not a prevalence estimate. And risk-coverage analysis imported unchanged from
balanced settings certifies rare-event screeners as needing no review, so the budget has to
be defined over the records the model flags.

Measured against human judgment of the same criteria, the model reaches an $F_1$ of
\GoldFone{} and a weighted kappa of \GoldKappa{} over \GoldNUsable{} usable consensus labels,
but its probabilities are not calibrated. The Spiegelhalter statistic is \GoldZ{} and the mean
probability lies above the prevalence, so the model fails calibration in the large, and its
slope of \GoldSlope{} means it fails weak calibration as well. The two references agree on the
calibration error, at a median ratio of \ECEGapMedian{}, and they disagree on agreement, where
the coded field reads lower than human judgment for every comparable variable by a median of
\KappaGapMedian{} in kappa. The coded reference is incomplete in one direction, and what that
incompleteness costs is the appearance of agreement rather than the appearance of calibration.
The same \GoldNUsable{} judgments then recalibrate the model to an out-of-fold calibration
error of \RecalIsoECE{} with a slope of \RecalPlattSlope{}, which is the operationally
important result. For a few thousand judgments a safety office obtains probabilities it can
sum, threshold and budget against over a corpus it could never label by hand. Adding those
probabilities to the coded contributing-factor fields raises the injury and fatal crashes
attributed to the nine mapped factors by \DownAddedTotal{} per year and moves
\DownNRankMove{} of them in rank, which is the answer to the second question. The frontier
comparison showed that two generative models from different vendors do not agree with each
other about how much better than the typed model they are. It also showed that calibration is a
property of the individual model rather than of the paradigm, so the transferable part of this
work is the audit rather than the model. The contributions are therefore a gated typed
decomposition whose design rules were measured, a large-scale run with a measured cost model,
and an independent calibration and selective-prediction audit against two references. To these
the paper adds a review budget defined over flagged records and certified out of fold, a
downstream measurement of what calibrated narrative variables change in a safety diagnosis, and
the released schema, code and outputs.

\limitationsparagraph The study has boundaries that the reader should hold in mind, and this
paragraph states them. The model is closed and commercial, its training data are undisclosed, and its
behavior may change without notice. The pinned identifier and the recorded per-call
identifier make such a change detectable but not preventable, so a reader cannot rerun the
analysis if the vendor withdraws the model, and audits of commercial systems should be
treated as perishable. The coded CRIS fields used as the large reference are incomplete, so
every statistic computed against them is an agreement measure bounded by the reference's own
error. The accuracy claims rest on a human reference set that is smaller by three orders of
magnitude and that is defined only on cases the coders found clear, because disputes were
excluded rather than adjudicated. Its pairwise reliability is estimated on \CoderNCalib{}
triple-coded pairs, and its stratification on the model's probability rather than on the
positive class leaves the rarest variables with a handful of positive labels, as few as
\ThreshWeakNPosMin{} for \ThreshWeakNThin{} of the \GoldNWeak{} that fail. Per-variable
precision, any threshold fitted to it, and the calibration error at the floor stratum are
correspondingly imprecise for those variables, and \RecalPooledOnlyN{} variables carry too few
positives to support a recalibration map of their own. The weights that make the reference set
speak for the population are calibration weights rather than known inclusion probabilities,
because the frame was drawn by a greedy cell-filling rule and then screened. Their validity
therefore rests on labeling being uninformative about the outcome within a cell. The recalibration maps
were fitted on \CoderN{} coders' judgment of Texas narratives and should be refitted rather
than transported. The downstream counts of Section~\ref{sec:res-downstream} inherit every one
of these boundaries, and for the \DownNConfirmLow{} factors whose narrative-only flags the
coders confirmed less than half the time they should be read as an upper bound rather than as
an estimate. The frontier arms ran with more context and more computation per judgment than the
typed model, and their cost was not metered at all, so no cost in this paper describes them and
the price comparison is a scenario over commercially listed tiers. The analysis covers one state, narratives are written by officers
with agency-specific conventions and are themselves incomplete records of the crash, and
residual identifiers survive the data owner's de-identification pass at the rates reported in
Section~\ref{sec:corpus}, which is why no narrative is reproduced here.

Four lines of work follow from those boundaries. The first is a case-by-case adjudication of
the discrepancies between narrative-derived and coded values, sampled on the discrepancies
themselves rather than on a probability-stratified frame. That design would narrow the wide
intervals on the narrative-only cell and attribute each discrepancy to the narrative, the
code or the model. The second is a revised schema for the
\GoldNWeak{} variables that failed, namely \GoldWeakVars{}, with their criteria rewritten after
error analysis. The revision would be scored against a fresh reference
set drawn against the revised schema, so that the repair is validated on labels that did not
select it. The third is replication in other states, which would test whether the schema and the cost
model transfer as Section~\ref{sec:discussion} predicts. The fourth is a repeated audit across model versions, with the same reference set
and the same metric code. It is the only way to learn whether the calibration and the review
budgets reported here are properties of the model or of one release, and the released code is
designed to make it routine.

%% Appendices: placed directly after the Conclusion and before the declarations and references.
\appendix
\renewcommand{\thetable}{\thesection.\arabic{table}}
\renewcommand{\thefigure}{\thesection.\arabic{figure}}
\renewcommand{\theequation}{\thesection.\arabic{equation}}

% an unnumbered heading, so a reader can see where the appendices begin
\section*{Appendices}

\section{The question schema}
\label{app:schema}
\setcounter{table}{0}\setcounter{figure}{0}\setcounter{equation}{0}

Table~\ref{tab:schema-full} reproduces every question in the schema as it was sent to the
model, with its type, its gate where it has one, and the complete instruction and criteria
text. The text is the entire specification of each variable, because the model sees nothing
else, and the human coders of Section~\ref{sec:reference} and the frontier models of
Section~\ref{sec:res-frontier} received the same text verbatim. The schema was produced in two steps that the released files record. A first version was written from the design rules of Section~\ref{sec:schema} and
run on a development sample of \SchemaValidateCompared{} narratives, and the error analysis of
that run produced the exclusion clause of the aggression question, whose effect is reported
in Section~\ref{sec:schema}. The file reproduced here is the second version, which is the one
used for the second-stage run, the reference set and the frontier arms, and it is released
with the code together with the lean first-stage schema of \NLeanQuestions{} presence
questions. Instructions refer to the narrative by its key in the state object, and criteria
for a presence question name the boundary cases that count and do not count, which is where
most of the design effort went. The residual-identifier question is included in the file
because it ships with the schema, but it was run as a separate display screen and was not part
of the population coding.

\begin{small}\rmfamily
\setlength{\tabcolsep}{3.0pt}
\begin{longtable}{lrr>{\raggedright\arraybackslash}p{267.2pt}}
\caption{The full question schema, version 1.1, as sent to the model.}\label{tab:schema-full}\\
\toprule
\jevhead{id} & \jevhead{type} & \jevhead{gate} & \jevheadp{instruction and criteria} \\
\midrule
\endfirsthead
\multicolumn{4}{@{}l}{\rmfamily\small\textit{Table \ref{tab:schema-full}, continued}}\\
\toprule
\jevhead{id} & \jevhead{type} & \jevhead{gate} & \jevheadp{instruction and criteria} \\
\midrule
\endhead
\bottomrule
\multicolumn{4}{@{}r}{\rmfamily\small\textit{continued on the next page}}\\
\endfoot
\bottomrule
\multicolumn{4}{@{}p{\linewidth}@{}}{\rmfamily\small\vspace{2pt}\textit{Note:}~Text is reproduced from \texttt{schemas/crash\_factors\_v1\_1.json} without abbreviation. A gated question is interpreted only when its gate exceeds $\tau_g = 0.5$, otherwise it takes its no-match option. \texttt{pii\_residual} ships in the same file but was run separately as the display screen, outside the Stage 1 and Stage 2 calls.}\\
\endlastfoot
\texttt{preg\_mentioned} & noul & -- & The 'narrative' states that a person involved in the crash was pregnant. \newline \textit{true}: Any explicit statement that a driver, passenger, pedestrian, or other involved person is pregnant, e.g. 'was 6 months pregnant', 'pregnant driver', 'expecting'. \newline \textit{false}: No pregnancy is stated. Mentions of children, infants, or car seats do not count. \\
\texttt{medical\_episode} & noul & -- & The 'narrative' states that a driver experienced a medical event (seizure, fainting, blackout, diabetic episode, heart attack, stroke, or similar) before or during the crash. \newline \textit{true}: A medical event of the driver is described as occurring before or during the crash, whether confirmed or suspected by the officer. \newline \textit{false}: No medical event of the driver is described. Injuries caused by the crash, or transport to a hospital for crash injuries, do not count. \\
\texttt{drug\_involved} & noul & -- & The 'narrative' states or suspects that a driver was under the influence of drugs or medication (not alcohol). \newline \textit{true}: Any statement that a driver used, was impaired by, or was suspected of using illegal drugs, marijuana, prescription medication, or other substances other than alcohol. \newline \textit{false}: No drug or medication involvement of a driver is mentioned. Alcohol-only impairment does not count. \\
\texttt{alcohol\_involved} & noul & -- & The 'narrative' states or suspects that a driver had been drinking alcohol or was intoxicated by alcohol. \newline \textit{true}: Any statement that a driver had been drinking, smelled of alcohol, was DWI or DUI, or was suspected of alcohol intoxication. \newline \textit{false}: No alcohol involvement of a driver is mentioned. \\
\texttt{hydroplane} & noul & -- & The 'narrative' states that a vehicle hydroplaned or lost traction because of water on the roadway. \newline \textit{true}: The word hydroplane appears, or the narrative says a vehicle lost traction, skidded, or lost control because of rain, wet roadway, or standing water. \newline \textit{false}: No loss of traction due to water is described. Wet conditions mentioned without loss of control do not count. \\
\texttt{vehicle\_defect} & noul & -- & The 'narrative' states that a mechanical failure of a vehicle (tire blowout, brake failure, steering failure, or similar) contributed to the crash. \newline \textit{true}: A tire blowout, tread separation, brake failure, steering failure, stuck accelerator, or other mechanical defect is described as contributing to the crash. \newline \textit{false}: No mechanical failure is described. Damage caused by the crash does not count. \\
\texttt{fatigue} & noul & -- & The 'narrative' states that a driver fell asleep, dozed off, was drowsy, or was fatigued. \newline \textit{true}: A driver is described as asleep, falling asleep, nodding off, drowsy, tired, or fatigued before the crash. \newline \textit{false}: No sleep or fatigue of a driver is mentioned. \\
\texttt{aggression} & noul & -- & The 'narrative' describes road rage, aggressive driving directed at another road user, or a physical or verbal confrontation between people involved. \newline \textit{true}: Road rage, brake-checking, chasing, threatening, arguing, fighting, or a confrontation before or after the crash is described. \newline \textit{false}: No aggression or confrontation is described. Ordinary traffic violations without hostility do not count. A police pursuit, a traffic stop, an officer chasing or stopping a vehicle, and a driver evading or fleeing from police do not count. A driver leaving the scene of the crash does not count by itself. \\
\texttt{intentional} & noul & -- & The 'narrative' states or suspects that a vehicle was used intentionally to cause the crash (assault with a vehicle, deliberate ramming, or a suicide attempt). \newline \textit{true}: The crash is described as deliberate, intentional, on purpose, an assault with a vehicle, or a suicide attempt. \newline \textit{false}: The crash is described as unintentional, or intent is not mentioned. \\
\texttt{hit\_and\_run} & noul & -- & The 'narrative' states that a driver left the scene of the crash without stopping to provide information. \newline \textit{true}: A driver fled, left the scene, failed to stop and render aid, or drove away without exchanging information. \newline \textit{false}: All drivers remained at the scene, or leaving the scene is not mentioned. \\
\texttt{animal\_involved} & noul & -- & The 'narrative' states that an animal on or near the roadway was involved in the crash or in the driver's evasive action. \newline \textit{true}: A vehicle struck an animal, or swerved or braked to avoid an animal. \newline \textit{false}: No animal is mentioned as involved. \\
\texttt{wrong\_way} & noul & -- & The 'narrative' states that a vehicle was traveling the wrong way on a roadway, ramp, or lane designated for the opposite direction. \newline \textit{true}: A vehicle is described as going the wrong way, wrong direction, or against traffic on a one-way road, divided highway, or ramp. \newline \textit{false}: No wrong-way travel is described. Crossing the center line during a loss of control does not count. \\
\texttt{phone\_use} & noul & -- & The 'narrative' states that a driver was using a phone or mobile device at the time of the crash. \newline \textit{true}: A driver is described as texting, talking on a phone, looking at a phone, or using a mobile device. \newline \textit{false}: No phone or device use is mentioned. General 'distracted' or 'inattention' without a device does not count. \\
\texttt{unbelted} & noul & -- & The 'narrative' states that a vehicle occupant was not wearing a seat belt. \newline \textit{true}: An occupant is described as unrestrained, not wearing a seat belt, or unbelted. \newline \textit{false}: Seat belt use is not mentioned, or occupants are described as belted. \\
\texttt{transported} & noul & -- & The 'narrative' states that a person was transported to a hospital or medical facility. \newline \textit{true}: A person was transported by EMS, ambulance, helicopter, or private vehicle to a hospital, medical center, or clinic. \newline \textit{false}: No transport to a medical facility is mentioned, or people refused transport. \\
\texttt{witness\_cited} & noul & -- & The 'narrative' cites a statement from a witness who was not a driver or occupant of the involved vehicles. \newline \textit{true}: A witness, bystander, or independent observer's statement is described. \newline \textit{false}: Only driver, passenger, or officer statements appear, or no statements appear. \\
\texttt{medical\_type} & choice & \texttt{medical\_episode} & Which type of driver medical event does the 'narrative' describe as occurring before or during the crash? \newline \textit{seizure}: Seizure or epileptic episode. \newline \textit{loss\_of\_consciousness}: Fainting, passing out, blacking out, syncope, or becoming unresponsive without a stated cause. \newline \textit{diabetic}: Diabetic episode, low blood sugar, hypoglycemia, or insulin-related event. \newline \textit{cardiac}: Heart attack, cardiac arrest, chest pain, or other heart event. \newline \textit{stroke}: Stroke or stroke-like symptoms. \newline \textit{other\_medical}: Another medical condition or episode of the driver, such as dementia, dizziness, or an unspecified medical emergency. \newline \textit{none}: No driver medical event is described. \\
\texttt{drug\_class} & choice & \texttt{drug\_involved} & Which category of drug or medication does the 'narrative' associate with a driver? Choose based only on what the narrative names. \newline \textit{marijuana}: Marijuana, THC, cannabis, or weed. \newline \textit{stimulant}: Methamphetamine, cocaine, amphetamines, or other stimulants. \newline \textit{opioid}: Heroin, fentanyl, oxycodone, hydrocodone, or other opioids. \newline \textit{sedative\_prescription}: Xanax, benzodiazepines, sleeping pills, muscle relaxers, or other prescription sedatives. \newline \textit{unspecified\_drug}: Drugs, narcotics, a controlled substance, or medication mentioned without naming the type. \newline \textit{none}: No drug or medication is mentioned. \\
\texttt{preg\_role} & choice & \texttt{preg\_mentioned} & What was the role of the pregnant person in the crash according to the 'narrative'? \newline \textit{driver}: The pregnant person was driving one of the vehicles. \newline \textit{passenger}: The pregnant person was a passenger in one of the vehicles. \newline \textit{pedestrian\_or\_other}: The pregnant person was a pedestrian, cyclist, or otherwise not inside a vehicle. \newline \textit{no\_pregnancy}: The narrative does not state that anyone was pregnant. \\
\texttt{preg\_outcome} & choice & \texttt{preg\_mentioned} & What does the 'narrative' state about the condition of the pregnant person after the crash? \newline \textit{no\_complaint}: The pregnant person reported no injury or the narrative gives no complaint. \newline \textit{pain\_or\_evaluation}: The pregnant person complained of pain, tightness, or discomfort, or was checked by EMS but not transported. \newline \textit{transported}: The pregnant person was transported to a hospital or medical facility as a precaution or for injury. \newline \textit{fetal\_harm}: The narrative states harm to, or loss of, the unborn child. \newline \textit{no\_pregnancy}: The narrative does not state that anyone was pregnant. \\
\texttt{preg\_stage} & choice & \texttt{preg\_mentioned} & What stage of pregnancy does the 'narrative' state for the pregnant person? \newline \textit{early}: Up to 13 weeks, or 1 to 3 months pregnant, or described as first trimester or early pregnancy. \newline \textit{mid}: 14 to 27 weeks, or 4 to 6 months pregnant, or described as second trimester. \newline \textit{late}: 28 weeks or more, or 7 to 9 months pregnant, or described as third trimester, due soon, or full term. \newline \textit{stage\_not\_stated}: Pregnancy is stated but no weeks, months, or trimester is given. \newline \textit{no\_pregnancy}: The narrative does not state that anyone was pregnant. \\
\texttt{animal\_type} & choice & \texttt{animal\_involved} & Which kind of animal does the 'narrative' describe as involved in the crash? \newline \textit{deer}: Deer. \newline \textit{dog}: Dog. \newline \textit{livestock}: Cow, cattle, horse, goat, or other livestock. \newline \textit{feral\_hog}: Hog, feral hog, wild boar, or pig. \newline \textit{other\_animal}: Any other animal, such as coyote, cat, or an unnamed animal. \newline \textit{none}: No animal is described as involved. \\
\texttt{defect\_type} & choice & \texttt{vehicle\_defect} & Which kind of vehicle mechanical failure does the 'narrative' describe as contributing to the crash? \newline \textit{tire}: Tire blowout, flat tire, or tread separation. \newline \textit{brakes}: Brake failure or brakes not working. \newline \textit{steering\_or\_suspension}: Steering failure, wheel came off, axle or suspension failure. \newline \textit{other\_mechanical}: Another mechanical failure such as engine stall, stuck accelerator, or trailer hitch failure. \newline \textit{none}: No mechanical failure is described. \\
\texttt{surface\_narr} & choice & -- & What roadway surface condition does the 'narrative' state at the time of the crash? \newline \textit{dry}: The roadway is described as dry. \newline \textit{wet}: The roadway is described as wet, rainy, or having standing water. \newline \textit{ice\_or\_snow}: The roadway is described as icy, frozen, slick from ice, or snowy. \newline \textit{other\_surface}: Another surface hazard is described, such as gravel, oil, mud, sand, or debris. \newline \textit{not\_stated}: The narrative does not describe the surface condition. \\
\texttt{narrative\_detail} & score & -- & How much crash-relevant detail does the 'narrative' contain? \newline \textit{level 0}: Boilerplate or a single sentence with no description of how the crash happened. \newline \textit{level 1}: A basic sequence of vehicle movements and the collision, without statements or contributing factors. \newline \textit{level 2}: A detailed account including driver or witness statements, contributing factors, or officer findings. \\
\texttt{emotion\_intensity} & score & -- & How intense is the emotional state of any driver described in the 'narrative' before or after the crash? \newline \textit{level 0}: No emotional state is described. \newline \textit{level 1}: A driver is described as upset, frustrated, crying, shaken, or scared. \newline \textit{level 2}: A driver is described as angry, hostile, yelling, threatening, or violent. \\
\texttt{injury\_narr} & score & -- & What is the most severe injury outcome stated in the 'narrative' for any person? \newline \textit{level 0}: No injury is mentioned, or the narrative says no one was injured. \newline \textit{level 1}: A person complained of pain or possible injury but was not transported. \newline \textit{level 2}: A person had a visible injury or was transported to a hospital. \newline \textit{level 3}: A person was pronounced dead or the narrative describes a fatality. \\
\texttt{pii\_residual} & noul & -- & The 'narrative' contains a person's full name, date of birth, driver license or ID number, home address, or phone number that is not already replaced by a placeholder. \newline \textit{true}: An actual name, date of birth, license or ID number, street address, or phone number appears as literal text. \newline \textit{false}: No such identifier appears, or every identifier has been replaced by a placeholder such as [NAME] or [PHONE]. \\
\end{longtable}
\end{small}

\section{Keyword families and the coded-field mapping}
\label{app:keywords}
\setcounter{table}{0}\setcounter{figure}{0}\setcounter{equation}{0}

The rule baseline of Section~\ref{sec:reference} is a set of case-insensitive regular
expressions, one family per variable, and Table~\ref{tab:keywords} lists them as they appear
in the scripts that apply them. The first ten families were written to over-find rather than
to be precise, because their original purpose was to enrich the rare classes of the
second-stage frame and to provide a population keyword rate over all \Ncorpus{} narratives.
They were tuned only on the development sample and never on any evaluated narrative. The
last three families, for alcohol involvement, phone use and restraint use, were added so that
the baseline covers every variable that has a coded reference, since comparing a
macro-averaged score over different variable sets would not be a comparison. Those three
were scored on the working set of narratives only. The residual-identifier patterns of
Section~\ref{sec:corpus} are of the same kind and are listed in the released code beside
them. Table~\ref{tab:coded-map} gives the mapping from each narrative variable to the coded
CRIS fields that define its reference, which is implemented in the released join script over
the crash-unit-person extract. Each reference is true for a crash when the rule holds on any
person or unit row of that crash, because a contributing factor is recorded per unit and a
restraint or test result per person. The hydroplaning reference is a proxy built from the
loss-of-control, surface and weather codes, because CRIS has no hydroplaning field. The
final row is the opposite-direction collision type that the original plan proposed for
wrong-way travel, retained only as a sensitivity check after it proved a factor of
\WrongWayProxyRatio{} too broad.

\begin{small}\rmfamily
\setlength{\tabcolsep}{3.0pt}
\begin{longtable}{l>{\raggedright\arraybackslash}p{71.2pt}>{\raggedright\arraybackslash}p{44.5pt}>{\raggedright\arraybackslash}p{289.2pt}}
\caption{Regular-expression families of the rule baseline.}\label{tab:keywords}\\
\toprule
\jevhead{family} & \jevheadp{narrative variable} & \jevheadp{scope} & \jevheadp{regular expression (case-insensitive)} \\
\midrule
\endfirsthead
\multicolumn{4}{@{}l}{\rmfamily\small\textit{Table \ref{tab:keywords}, continued}}\\
\toprule
\jevhead{family} & \jevheadp{narrative variable} & \jevheadp{scope} & \jevheadp{regular expression (case-insensitive)} \\
\midrule
\endhead
\bottomrule
\multicolumn{4}{@{}r}{\rmfamily\small\textit{continued on the next page}}\\
\endfoot
\bottomrule
\multicolumn{4}{@{}p{\linewidth}@{}}{\rmfamily\small\vspace{2pt}\textit{Note:}~The first ten families apply to all narratives and provide the population keyword rates of Table~\ref{tab:prevalence}. The last three were added so the baseline covers every variable with a coded reference, and were scored on the working set only. Patterns are read from the scripts that apply them.}\\
\endlastfoot
preg & preg\_mentioned & all narratives & \texttt{pregnan|\allowbreak{}unborn|\allowbreak{}fetus|\allowbreak{}fetal|\allowbreak{}trimester|\allowbreak{}weeks along|\allowbreak{}miscarr} \\
medical & medical\_episode & all narratives & \texttt{seizure|\allowbreak{}diabet|\allowbreak{}passed out|\allowbreak{}blacked out|\allowbreak{}fainted|\allowbreak{}heart attack|\allowbreak{}stroke|\allowbreak{}unconscious|\allowbreak{}medical (episode|\allowbreak{}condition|\allowbreak{}emergency|\allowbreak{}event)|\allowbreak{}dementia|\allowbreak{}alzheimer} \\
drug & drug\_involved & all narratives & \texttt{marijuana|\allowbreak{}methamphetamine|\allowbreak{}meth |\allowbreak{}cocaine|\allowbreak{}fentanyl|\allowbreak{}xanax|\allowbreak{}heroin|\allowbreak{}pills|\allowbreak{}narcotic|\allowbreak{}controlled substance|\allowbreak{}under the influence of (a )?drug} \\
hydro & hydroplane & all narratives & \texttt{hydroplan} \\
defect & vehicle\_defect & all narratives & \texttt{tire blow|\allowbreak{}blowout|\allowbreak{}blew out|\allowbreak{}brake fail|\allowbreak{}brakes fail|\allowbreak{}mechanical fail|\allowbreak{}steering fail} \\
rage & aggression & all narratives & \texttt{road rage|\allowbreak{}angry|\allowbreak{}anger|\allowbreak{}enraged|\allowbreak{}argument|\allowbreak{}fight|\allowbreak{}confront|\allowbreak{}brandish|\allowbreak{}assault} \\
fatigue & fatigue & all narratives & \texttt{fell asleep|\allowbreak{}asleep|\allowbreak{}drowsy|\allowbreak{}fatigue|\allowbreak{}nodded off} \\
intent & intentional & all narratives & \texttt{suicid|\allowbreak{}intentional|\allowbreak{}deliberate|\allowbreak{}on purpose|\allowbreak{}rammed} \\
wrongway & wrong\_way & all narratives & \texttt{wrong way|\allowbreak{}wrong direction} \\
animal & animal\_involved & all narratives & \texttt{deer|\allowbreak{}hog|\allowbreak{}cow|\allowbreak{}horse|\allowbreak{}dog|\allowbreak{}coyote|\allowbreak{}cattle} \\
alcohol & alcohol\_involved & working set & \texttt{\textbackslash{}balcohol|\allowbreak{}intoxicat|\allowbreak{}\textbackslash{}bdwi\textbackslash{}b|\allowbreak{}\textbackslash{}bdui\textbackslash{}b|\allowbreak{}drunk|\allowbreak{}\textbackslash{}bbac\textbackslash{}b|\allowbreak{}been drinking|\allowbreak{}odor of an alcoholic|\allowbreak{}open container|\allowbreak{}breathalyz|\allowbreak{}field sobriety} \\
phone & phone\_use & working set & \texttt{cell ?phone|\allowbreak{}cellular|\allowbreak{}mobile (phone|\allowbreak{}device)|\allowbreak{}texting|\allowbreak{}\textbackslash{}btext(ed|\allowbreak{}ing)?\textbackslash{}b|\allowbreak{}on (the|\allowbreak{}his|\allowbreak{}her|\allowbreak{}their) phone|\allowbreak{}looking at (his|\allowbreak{}her|\allowbreak{}their) phone} \\
unbelted & unbelted & working set & \texttt{not wearing (a )?(seat ?belt|\allowbreak{}safety belt)|\allowbreak{}unbelted|\allowbreak{}unrestrained|\allowbreak{}no (seat ?belt|\allowbreak{}restraint)|\allowbreak{}seat ?belt was not|\allowbreak{}without a seat ?belt} \\
\end{longtable}
\end{small}

\begin{small}\rmfamily
\setlength{\tabcolsep}{3.0pt}
\begin{longtable}{>{\raggedright\arraybackslash}p{73.0pt}r>{\raggedright\arraybackslash}p{203.7pt}>{\raggedright\arraybackslash}p{70.6pt}}
\caption{Mapping from narrative variables to coded CRIS reference fields.}\label{tab:coded-map}\\
\toprule
\jevheadp{narrative variable} & \jevhead{reference column} & \jevheadp{CRIS fields and values} & \jevheadp{aggregation to crash level} \\
\midrule
\endfirsthead
\multicolumn{4}{@{}l}{\rmfamily\small\textit{Table \ref{tab:coded-map}, continued}}\\
\toprule
\jevheadp{narrative variable} & \jevhead{reference column} & \jevheadp{CRIS fields and values} & \jevheadp{aggregation to crash level} \\
\midrule
\endhead
\bottomrule
\multicolumn{4}{@{}r}{\rmfamily\small\textit{continued on the next page}}\\
\endfoot
\bottomrule
\multicolumn{4}{@{}p{\linewidth}@{}}{\rmfamily\small\vspace{2pt}\textit{Note:}~Each reference is true for a crash when the rule holds on any of its person or unit rows. The hydroplaning reference is a proxy built from surface, weather and loss-of-control codes, because CRIS has no hydroplaning field. The final row is the opposite-direction collision type the original plan proposed for wrong-way travel, retained only as a sensitivity check.}\\
\endlastfoot
alcohol\_involved & \texttt{coded\_alcohol} & \texttt{Contrib\_Factr\_1\_ID} or \texttt{Contrib\_Factr\_2\_ID} matches ``Under Influence - Alcohol'' or ``Had Been Drinking''; or \texttt{Prsn\_Alc\_Rslt\_ID} = 1; or \texttt{Prsn\_Bac\_Test\_Rslt} $>$ 0 & any person or unit row of the crash \\
drug\_involved & \texttt{coded\_drug} & contributing factor matches ``Under Influence - Drug''; or \texttt{Prsn\_Drg\_Rslt\_ID} = 1 & any row \\
fatigue & \texttt{coded\_fatigue} & contributing factor matches ``Fatigued Or Asleep'' & any row \\
medical\_episode & \texttt{coded\_medical} & contributing factor matches ``Ill (Explain In Narrative)'' & any row \\
animal\_involved & \texttt{coded\_animal} & contributing factor matches ``Animal On Road''; or \texttt{Harm\_Evnt\_ID} contains ``Animal'' & any row \\
phone\_use & \texttt{coded\_phone} & contributing factor matches ``Cell/Mobile Device Use'' & any row \\
unbelted & \texttt{coded\_unbelted} & \texttt{Prsn\_Rest\_ID} equals ``None'' (the only unbelted code; ``Not Applicable'' excluded) & any person row \\
hydroplane & \texttt{coded\_hydro\_proxy} & \texttt{Othr\_Factr\_ID} contains ``Lost Control Or Skidded'' and (\texttt{Surf\_Cond\_ID} contains ``Wet'' or ``Standing Water'', or \texttt{Wthr\_Cond\_ID} contains ``Rain'') & any row; proxy reference \\
wrong\_way & \texttt{coded\_wrongway} & contributing factor matches ``Wrong Way - One Way Road'', ``Wrong Side - Not Passing'' or ``Wrong Side - Approach'' & any row \\
(sensitivity only) & \texttt{coded\_wrongway\_proxy} & \texttt{FHE\_Collsn\_ID} contains ``Opposite Direction'' & any row; not used as a reference \\
\end{longtable}
\end{small}

\section{Coder protocol and additional reliability diagrams}
\label{app:coders}
\setcounter{table}{0}\setcounter{figure}{0}\setcounter{equation}{0}

The human reference set of Section~\ref{sec:reference} was produced under a written protocol
that this appendix summarizes, and Figures~\ref{fig:goldrel} and \ref{fig:goldprec} draw the
two human-referenced diagrams that the main text refers to. The protocol defined the task as
one narrow question repeated, whether the narrative states the factor, under the governing
rule to label what the narrative states rather than what probably happened. A crash that an
experienced analyst would read as alcohol-involved is therefore labeled negative if the
officer did not write it down. That rule inverts the usual assumption that more domain
expertise is better, because the dominant risk is over-inference. Coders who know crashes
may fill in what the officer omitted, agree with each other, and produce a reference that is
confidently wrong in the direction the model might also be wrong. The panel was therefore
mixed by design, with one crash-data practitioner who reads officer shorthand, one
transportation-safety researcher familiar with the constructs, and one careful non-specialist
reader who cannot over-infer and who serves as the check on the other two. Coders were
required to be independent of the modeling work, never to have seen the model's outputs or
helped write the criteria, and not to discuss items while coding. All three first coded a
shared calibration block, which was reviewed together once to catch inference, ambiguous
criteria and shorthand the non-specialist could not read, after which no further discussion
took place. Each coder then labeled an allocation of narratives determined by position in a
fixed roster, with every narrative coded by two of the three and the calibration block by all
three. The roster rule was deterministic and fixed before any narrative was seen. Narratives
were held in the order the export produced, and the calibration block occupied the first
positions and was assigned to every coder. Each later position went to the two coders whose
index matched that position under a fixed rotation, so no coder chose a narrative and no
narrative was reassigned after labeling began. Each judgment carried the coder identifier, the
time taken and a yes, no or unclear verdict. Reliability was computed from the double- and
triple-coded pairs as reported in Section~\ref{sec:reference}, majority labels were formed, and
ties and unclear verdicts were recorded as disputed and excluded.

The calibration weights of Section~\ref{sec:reference} are built from the frame rather than
from the labels, and their construction is recorded here in full. The population is the
second-stage random stratum. For each of the \FrontierNVars{} presence variables the population
is divided into the cells \GoldBinsText{} and the labeled pairs falling in each cell are scaled
so that their weights sum to the population count of that cell. \DesignNActivePairs{} of the
pairs entered through the active route, with assignment probability one, and
\DesignNControlPairs{} through the control route, with assignment probability
\GoldNegControls{} divided by the narrative's count of low-probability variables. Cells holding
fewer than \DesignMinCell{} labeled pairs were merged with an adjacent cell of the same
variable before the scaling, which affected \DesignNCollapsed{} variables, and the merges are
listed in the released design report. The weights span a factor of \DesignWeightRatio{} and
give a design effect of \DesignDeff{}, with per-variable effective sample sizes running from
\DesignNEffVarLo{} to \DesignNEffVarHi{}. The cluster bootstrap resamples narratives within the
\DesignNStrata{} draw routes the frame records, which are the cells the greedy rule drew each
narrative to fill.

\begin{small}\rmfamily
\setlength{\tabcolsep}{4.0pt}
\begin{longtable}{l>{\raggedright\arraybackslash}p{186.9pt}rrrr}
\caption{Unresolved pairs and the sensitivity of the headline results to them.}\label{tab:exclusions}\\
\toprule
\jevhead{block} & \jevheadp{item} & \jevhead{count} & \jevhead{$F_1$} & \jevhead{ECE} & \jevhead{$H(0.90)$ \%} \\
\midrule
\endfirsthead
\multicolumn{6}{@{}l}{\rmfamily\small\textit{Table \ref{tab:exclusions}, continued}}\\
\toprule
\jevhead{block} & \jevheadp{item} & \jevhead{count} & \jevhead{$F_1$} & \jevhead{ECE} & \jevhead{$H(0.90)$ \%} \\
\midrule
\endhead
\bottomrule
\multicolumn{6}{@{}r}{\rmfamily\small\textit{continued on the next page}}\\
\endfoot
\bottomrule
\multicolumn{6}{@{}p{\linewidth}@{}}{\rmfamily\small\vspace{2pt}\textit{Note:}~The \ExclN{} excluded pairs are those the coders marked unclear or disagreed on. They are dropped from every rate in Section~\ref{sec:res-accuracy}, so those rates are conditional on the pairs the coders found clear. The sensitivity block replaces that conditioning with three assumptions. Routing the excluded pairs to review leaves $F_1$ undefined, because a record sent to a human is not scored as a model decision.}\\
\endlastfoot
by variable & hydroplane & 9 & -- & -- & -- \\
by variable & hit\_and\_run & 7 & -- & -- & -- \\
by variable & phone\_use & 7 & -- & -- & -- \\
by variable & transported & 6 & -- & -- & -- \\
by variable & witness\_cited & 4 & -- & -- & -- \\
by variable & alcohol\_involved & 4 & -- & -- & -- \\
by variable & vehicle\_defect & 3 & -- & -- & -- \\
by variable & aggression & 3 & -- & -- & -- \\
by variable & medical\_episode & 2 & -- & -- & -- \\
by variable & wrong\_way & 2 & -- & -- & -- \\
by variable & fatigue & 1 & -- & -- & -- \\
by variable & drug\_involved & 1 & -- & -- & -- \\
by probability bin & [0.05,0.3) & 9 & -- & -- & -- \\
by probability bin & [0.3,0.7) & 22 & -- & -- & -- \\
by probability bin & [0.7,0.95) & 17 & -- & -- & -- \\
by probability bin & [0.95,1] & 1 & -- & -- & -- \\
sensitivity & excluded, as reported throughout & 49 & 0.908 & 0.0231 & 0.0 \\
sensitivity & counted as model errors & 49 & 0.875 & 0.0228 & 6.9 \\
sensitivity & counted as model-correct & 49 & 0.912 & 0.0232 & 0.0 \\
sensitivity & routed to review regardless of probability & 49 & -- & -- & 3.7 \\
\end{longtable}
\end{small}

\begin{figure}[pos=!htbp]\centering
\includegraphics[width=\textwidth]{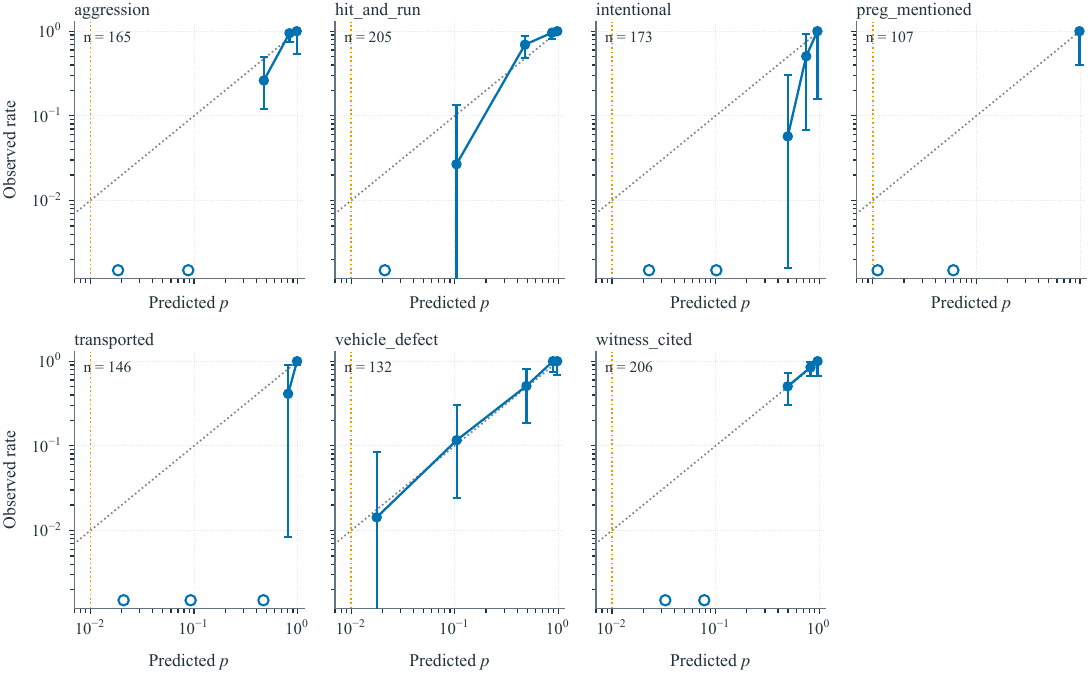}
\caption{Reliability against the human reference set for the seven presence variables with no
coded counterpart, which therefore do not appear in Figure~\ref{fig:calibration}. Each panel
gives the observed positive rate among the human labels within each sampling stratum,
weighted by the calibration weights of Section~\ref{sec:reference}, with approximate
intervals on the Kish effective sample size, on logarithmic axes. The dotted diagonal marks
calibration, the vertical guide the grid floor, open circles strata with no positive, and the
number of judgments is printed.}
\label{fig:goldrel}
\end{figure}

\begin{figure}[pos=!htbp]\centering
\includegraphics[width=\textwidth]{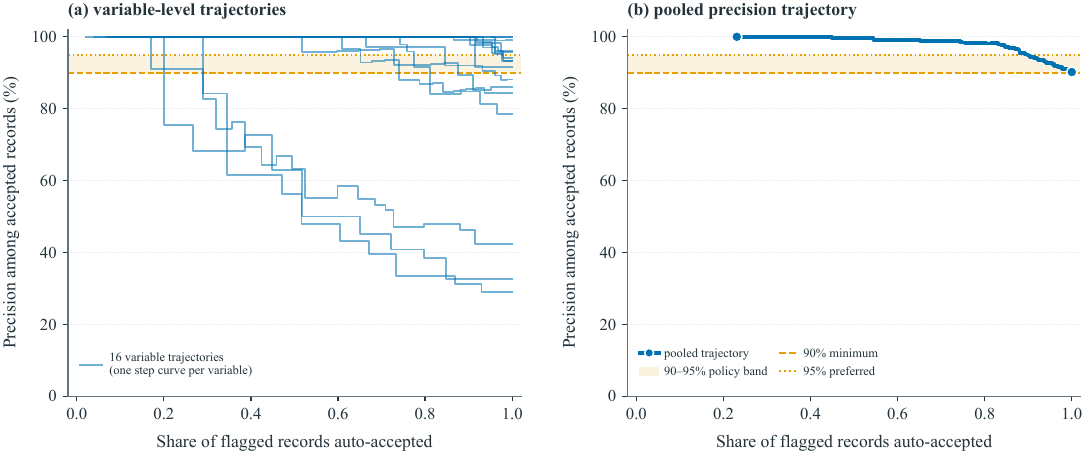}
\caption{Precision among accepted records against the share of flagged records auto-accepted,
measured on the human reference set with the same construction as
Figure~\ref{fig:riskcoverage}b. (a) Each variable as a thin line, most holding
\PrecTargetNinety\% precision until more than four fifths of their flags are accepted.
(b) All variables pooled, crossing \PrecTargetNinetyFive\% near nine tenths of the range and
\PrecTargetNinety\% only at its right-hand end. Both panels use the whole reference set, so
they give the in-sample budgets \GoldBudgetNinetyFiveInSample\% and
\GoldBudgetNinetyInSample\%; Section~\ref{sec:res-accuracy} reports the cross-fitted budgets
\GoldBudgetNinetyFive\% and \GoldBudgetNinety\% instead.}
\label{fig:goldprec}
\end{figure}

\FloatBarrier
\section{Proofs and the implementation map}
\label{app:proofs}
\setcounter{table}{0}\setcounter{figure}{0}\setcounter{equation}{0}
\appendixtheorems

This appendix collects the mathematics on which the audit of Section~\ref{sec:evaluation}
rests. It fixes the notation shared by every statement and proves that the Murphy
decomposition of the Brier score is exact on the output grid. It then proves
Proposition~\ref{prop:floor} and Lemma~\ref{lem:noise} together with a tightness statement
for each, derives the closed form of the break-even fraction of Eq.~(\ref{eq:twostage}), and
closes with the map from every numbered equation to the released code. Every statement is written for the weighted empirical
distribution of a labeled sample, with weights $w_i$ that sum to $n$, and each holds verbatim
for the population once weighted means are replaced by expectations.

Throughout, $G \subset [0,1]$ is the finite output grid, $W_v = \sum_i w_i\,\mathbb{1}[p_i = v]$
is the weight at grid value $v$, $r_v = \frac{1}{W_v}\sum_i w_i\,\mathbb{1}[p_i = v]\,y_i$ is the
observed positive rate at $v$ whenever $W_v > 0$, $\bar{p} = \sum_v (W_v/n)\,v$ is the mean
probability, and $\bar{y} = \sum_v (W_v/n)\,r_v = \frac{1}{n}\sum_i w_i y_i$ is the base rate,
written $\rho$ in the main text. Grouping by grid value is exact rather than approximate,
because every record at a value $v$ carries the same probability. Two identities follow
directly from the definitions,
\begin{equation}
\sum_i w_i\,\mathbb{1}[p_i = v]\,(y_i - r_v) = 0 \quad\text{for every } v,\qquad
\sum_{v \in G} \frac{W_v}{n}\,(v - r_v) = \bar{p} - \bar{y}.
\label{eq:app-ident}
\end{equation}
A model is calibrated on the sample when $r_v = v$ at every grid value with $W_v > 0$, and
$w_0 = W_0/n$ denotes the weight the model places on the value zero.

\begin{lemma}[Exact Murphy decomposition on the grid]
\label{lem:murphy}
With $\mathrm{BS}$ as in Eq.~(\ref{eq:brier}) and the grouping by grid value,
\begin{equation}
\mathrm{BS} \;=\; \sum_{v \in G} \frac{W_v}{n}\,(v - r_v)^2
\;-\; \sum_{v \in G} \frac{W_v}{n}\,(r_v - \bar{y})^2 \;+\; \bar{y}\,(1 - \bar{y}),
\label{eq:app-murphy}
\end{equation}
with equality and no binning error, and the reliability term bounds the exact calibration
error of Eq.~(\ref{eq:ecegrid}) through $\mathrm{ECE}_G^2 \le \sum_v (W_v/n)(v - r_v)^2$.
\end{lemma}

\begin{pf}
Fix $v$ with $W_v > 0$. For every record at $v$,
$(v - y_i)^2 = (v - r_v)^2 + 2(v - r_v)(r_v - y_i) + (r_v - y_i)^2$. Summing with the weights
$w_i$ over the records at $v$, the middle term vanishes by the first identity of
Eq.~(\ref{eq:app-ident}). Because $y_i \in \{0, 1\}$ gives $y_i^2 = y_i$, the last term sums to
$W_v r_v - W_v r_v^2 = W_v\, r_v(1 - r_v)$. Summing over $v$ and dividing by $n$,
$\mathrm{BS} = \sum_v (W_v/n)(v - r_v)^2 + \sum_v (W_v/n)\, r_v(1 - r_v)$. Expanding
$\sum_v (W_v/n)(r_v - \bar{y})^2 = \sum_v (W_v/n)\, r_v^2 - \bar{y}^2$ and substituting gives
$\sum_v (W_v/n)\, r_v(1 - r_v) = \bar{y}(1 - \bar{y}) - \sum_v (W_v/n)(r_v - \bar{y})^2$, which
is Eq.~(\ref{eq:app-murphy}). The inequality is Jensen's inequality for the square applied to
the probability weights $W_v/n$ and the values $|v - r_v|$.
\end{pf}

\begin{proofprop}
Suppose first that the model is calibrated on the sample, so that $r_v = v$ wherever
$W_v > 0$. Then $\rho = \bar{y} = \sum_v (W_v/n)\, r_v = \sum_v (W_v/n)\, v = \bar{p}$, and
because every positive grid value is at least $\delta$,
$\bar{p} \ge \delta \sum_{v > 0} W_v/n = \delta(1 - w_0)$. Rearranging $\rho \ge \delta(1 - w_0)$
gives $w_0 \ge 1 - \rho/\delta$. The bound used only $v \ge \delta$, so it holds for every
assignment of positive grid values, and the contrapositive is the first claim. If
$\rho < \delta(1 - w_0)$, then $r_v = v$ fails at some grid value carrying weight, whatever
values the model returns. For the second claim, the second identity of Eq.~(\ref{eq:app-ident})
and the triangle inequality give
$\mathrm{ECE}_G = \sum_v (W_v/n)\,|v - r_v| \ge \bigl|\sum_v (W_v/n)(v - r_v)\bigr|
= |\bar{p} - \rho| \ge \bar{p} - \rho$. The inequality $\bar{p} \ge \delta(1 - w_0)$ holds
whether or not the model is calibrated, so $\bar{p} - \rho \ge \delta(1 - w_0) - \rho > 0$,
which is Eq.~(\ref{eq:floor}). Equality in the triangle inequality holds exactly when
$v - r_v$ has one sign at every grid value carrying weight, so $\mathrm{ECE}_G = \bar{p} - \rho$
for a model whose reliability curve lies on or above the diagonal everywhere, and the floor
$\delta(1 - w_0) - \rho$ is attained when in addition all positive weight sits at $\delta$.
Finally, let $g \colon G \to [0,1]$ be any recalibration map. If $g(p)$ is calibrated then
$\sum_v (W_v/n)\, g(v) = \rho < \delta(1 - w_0)$, so $g$ must send some positive grid value
carrying weight below $\delta$. The maps of Eq.~(\ref{eq:platt}) can do so and the model
cannot, which is why the floor is a property of the grid and not of the ranking.
\end{proofprop}

\begin{prooflem}
Write the coded field as $y' = y\,(1 - m)$, where $m$ is an omission indicator with
$\Pr(m = 1 \mid y = 1, p = v) = f_v$, so that $y' = 1$ implies $y = 1$ and
$r'_v = \Pr(y' = 1 \mid p = v) = (1 - f_v)\, r_v$. The weights $W_v$ depend on $p$ alone and are
therefore the same under both references. Since $0 \le f_v \le 1$, $r'_v \le r_v$, which is
the first claim. If $(1 - f_v)\, r_v \ge v$ then $v \le r'_v \le r_v$, so
$|v - r'_v| = r'_v - v \le r_v - v = |v - r_v|$, which is the second claim, and summing with the
weights $W_v/n$ over every grid value with $W_v > 0$ gives Eq.~(\ref{eq:noise}). Under that
condition the gap is exact rather than an inequality,
\begin{equation}
\mathrm{ECE}_G(y) - \mathrm{ECE}_G(y') \;=\; \sum_{v \in G} \frac{W_v}{n}\,(r_v - r'_v)
\;=\; \sum_{v \in G} \frac{W_v}{n}\, f_v\, r_v \;=\; \Pr(y = 1,\; y' = 0),
\label{eq:app-gap}
\end{equation}
the share of records that the narrative reference marks positive and the coded field does
not, which is the narrative-only cell measured in Section~\ref{sec:res-frontier}. Where the
condition fails at some $v$, $r'_v < v$ and $|v - r'_v| = v - r'_v$. This exceeds $|v - r_v|$
whenever $r_v < v$, and whenever $r_v \ge v$ with $2v > r_v + r'_v$, so omission then inflates
the measured error and the bias reverses on that part of the grid. If the coded field also
records factors the narrative does not state, $r'_v$ can exceed $r_v$ and the direction of the
bias is no longer determined by the omission rate alone.
\end{prooflem}

\begin{lemma}[Break-even fraction]
\label{lem:breakeven}
Under the cost model of Eq.~(\ref{eq:cost}) with $k_f > k_s$, the break-even fraction of
Eq.~(\ref{eq:twostage}) has the closed form
\begin{equation}
q^{*} \;=\; \frac{\beta\,(k_f - k_s)}{\alpha + \beta k_f + b\bar{\ell}},\qquad
\frac{\partial q^{*}}{\partial \bar{\ell}} \;=\; -\,\frac{b\,\beta\,(k_f - k_s)}{(\alpha + \beta k_f + b\bar{\ell})^2},\qquad
\frac{\partial \ln q^{*}}{\partial \ln \bar{\ell}} \;=\; -\,\frac{b\bar{\ell}}{\alpha + \beta k_f + b\bar{\ell}},
\label{eq:app-breakeven}
\end{equation}
so $q^{*}$ decreases in narrative length, and its elasticity with respect to length equals
minus the share of the full-schema token count that the text contributes.
\end{lemma}

\begin{pf}
Screening costs less than coding everything once exactly when
$T(k_s, \bar{\ell}) + q\,T(k_f, \bar{\ell}) < T(k_f, \bar{\ell})$, that is when
$q < 1 - T(k_s, \bar{\ell})/T(k_f, \bar{\ell}) = [T(k_f, \bar{\ell}) - T(k_s, \bar{\ell})]/T(k_f, \bar{\ell})$.
The intercept $\alpha$ and the text term $b\bar{\ell}$ cancel in the difference, which leaves
$\beta(k_f - k_s)$ in the numerator. Differentiating the quotient in $\bar{\ell}$ gives the
second expression, and multiplying it by $\bar{\ell}/q^{*}$ gives the third. When
$\beta k_f$ dominates $b\bar{\ell}$ the elasticity is small, which is the claim of
Section~\ref{sec:schema} that the schema rather than the text governs the two-stage design.
\end{pf}

The weighting of Eq.~(\ref{eq:ht}) is a calibration estimator rather than an inverse-probability
one, and the distinction matters for what can be claimed about it. The classical result is that
$\mathbb{E}\bigl[\sum_{i \in s} z_i/\pi_i\bigr] = \sum_{i \in U} z_i$ for any sample $s$ drawn
from the frame $U$ with inclusion probabilities $\pi_i > 0$, because
$\mathbb{E}\,\mathbb{1}[i \in s] = \pi_i$ \citep{horvitz1952generalization}. That result is
background here rather than an assertion about this reference set, because the greedy draw and
the identifier screen of Section~\ref{sec:reference} leave no $\pi_i$ that can be written down.
The weights used instead satisfy the calibration equations exactly, so the weighted count of
every cell reproduces its population count, and the resulting estimator is design-consistent
under the condition that labeling is uninformative about the outcome within a cell
\citep{deville1992calibration}. The ratio of two weighted sums is consistent but not unbiased,
with a bias of order $1/n$ \citep{sarndal1992model}, which is why the percentile interval of
Eq.~(\ref{eq:htvar}) is re-centered by the bootstrap estimate of that bias. Table~\ref{tab:eqmap} lists, for every numbered equation of
Section~\ref{sec:methods}, the released script, the function and the result file that compute
it, so that a reader can locate the implementation of each quantity without searching the
repository.

\begin{sidewaystable}[!htbp]
\centering
\rmfamily\small
\setlength{\tabcolsep}{3.0pt}
\jevrotcaption{tab:eqmap}{Where each formalized quantity is computed.}
\begin{threeparttable}
\begin{tabular}{>{\raggedright\arraybackslash}p{64.1pt}>{\raggedright\arraybackslash}p{132.5pt}>{\raggedright\arraybackslash}p{135.4pt}>{\raggedright\arraybackslash}p{130.6pt}>{\raggedright\arraybackslash}p{136.7pt}}
\toprule
\jevheadp{equation} & \jevheadp{quantity} & \jevheadp{script (\texttt{src/})} & \jevheadp{function} & \jevheadp{result file (\texttt{data/})} \\
\midrule
Eq.~(\ref{eq:gate}) & gating operator and no-match override & \texttt{s06\_flatten.py} & \texttt{flatten} & \texttt{stage2\_flat.parquet} \\
Eq.~(\ref{eq:leak}) & leakage rate $\lambda_v$ & \texttt{s06\_flatten.py} & \texttt{flatten} & \texttt{stage2\_leakage.json} \\
Eq.~(\ref{eq:ht}) & H\'ajek ratio with calibration weights $w_i$ & \texttt{s05b\_design\_weights.py; s17\_gold\_analysis.py} & \texttt{design\_weights; wrate, block} & \texttt{gold\_design\_weights.csv; gold\_analysis.json} \\
Eq.~(\ref{eq:htvar}) & stratified cluster bootstrap interval & \texttt{s07\_metrics.py} & \texttt{cluster\_bootstrap; weighted\_bootstrap} & \texttt{analysis.json; gold\_analysis.json} \\
Eq.~(\ref{eq:brier}) & Brier score and Murphy decomposition & \texttt{s07\_metrics.py} & \texttt{brier, brier\_decomposition} & \texttt{analysis.json} \\
Eq.~(\ref{eq:ece}) & expected and maximum calibration error, equal-mass bins & \texttt{s07\_metrics.py} & \texttt{ece, mce, equal\_mass\_bins} & \texttt{analysis.json} \\
Eq.~(\ref{eq:ecegrid}) & exact calibration error on the output grid & \texttt{s07\_metrics.py} & \texttt{ece\_discrete, discrete\_reliability} & \texttt{analysis.json; gold\_analysis.json} \\
Eq.~(\ref{eq:cox}) & calibration slope and intercept & \texttt{s07\_metrics.py} & \texttt{calibration\_slope} & \texttt{analysis.json; gold\_analysis.json} \\
Eq.~(\ref{eq:spieg}) & Spiegelhalter statistic & \texttt{s07\_metrics.py} & \texttt{spiegelhalter\_z} & \texttt{analysis.json; gold\_analysis.json} \\
Eq.~(\ref{eq:floor}) & resolution-floor bound & \texttt{s07\_metrics.py} & \texttt{resolution\_floor\_report} & \texttt{analysis.json} \\
Eq.~(\ref{eq:prec}) & precision and coverage over flagged records & \texttt{s07\_metrics.py} & \texttt{precision\_coverage, coverage\_at\_precision} & \texttt{analysis.json} \\
Eq.~(\ref{eq:budget}) & review budget $H(\pi^*)$, cross-fitted & \texttt{s07\_metrics.py; s17\_gold\_analysis.py} & \texttt{positive\_review\_budget; held\_out\_budget} & \texttt{analysis.json; gold\_analysis.json} \\
Eq.~(\ref{eq:riskcov}) & risk--coverage curve and area & \texttt{s07\_metrics.py} & \texttt{risk\_coverage, aurc, coverage\_at\_risk} & \texttt{analysis.json} \\
Eq.~(\ref{eq:noise}) & one-sided reference-noise bound & \texttt{s17\_gold\_analysis.py} & \texttt{narrative\_only} & \texttt{gold\_analysis.json} \\
Eq.~(\ref{eq:cost}) & token cost model & \texttt{s08\_cost\_model.py} & \texttt{main} & \texttt{cost\_model.json} \\
Eq.~(\ref{eq:twostage}) & two-stage cost and break-even fraction & \texttt{s08\_cost\_model.py; pricing.py} & \texttt{main; jev\_cost\_per\_1k} & \texttt{cost\_model.json} \\
Eq.~(\ref{eq:platt}) & Platt and isotonic recalibration, split over narratives & \texttt{s17\_gold\_analysis.py} & \texttt{recalibration} & \texttt{gold\_analysis.json} \\
Eq.~(\ref{eq:kappa}) & Cohen's kappa, weighted against the reference set & \texttt{s07\_metrics.py; s16\_gold\_ingest.py} & \texttt{cohen\_kappa} & \texttt{gold\_reliability.json; gold\_analysis.json} \\
Eq.~(\ref{eq:cp}) & Clopper--Pearson interval & \texttt{s07\_metrics.py} & \texttt{clopper\_pearson} & \texttt{analysis.json} \\
Eq.~(\ref{eq:mcnemar}) & McNemar test & \texttt{s07\_metrics.py} & \texttt{mcnemar} & \texttt{analysis.json} \\
\bottomrule
\end{tabular}
% \begin{tablenotes}[flushleft]\small\rmfamily
% \item \textit{Note:}~Scripts are in \texttt{paper1/src/} and result files in \texttt{paper1/data/} of the released repository. Every equation in Section~\ref{sec:methods} maps to a function that a reader can execute on the released outputs.
% \end{tablenotes}
\end{threeparttable}
\end{sidewaystable}

\section*{Data and code availability}
The question schema, the asynchronous runner, the metric implementations with their unit
tests, the figure and table generators, the dependency file and a release manifest naming what
is public and what is not, and all aggregated outputs behind every figure and table are
released at \url{https://github.com/pozapas/jev-calibrated-narrative-coding}. All runs use a
pinned model identifier, recorded per call, and every script resolves its paths from the
repository root.

The underlying narratives and the record-level files that carry crash identifiers are not
redistributable under the data agreement, and the reference frame is among them, so it is not
released. A de-identified pair-level audit table has been prepared instead, carrying an
anonymous narrative identifier, the variable, the model probability, the consensus label, the
dispute status, the calibration cell, the draw route, the narrative weight, the assignment
probability, the pair weight and both frontier arms' predictions, with no crash identifier and
no narrative text. It is the file from which the human-reference results of
Section~\ref{sec:res-accuracy}, the recalibration of Section~\ref{sec:res-calibration} and the
frontier comparison of Section~\ref{sec:res-frontier} can be recomputed in full. It will be
made available on request once the data owner has confirmed that the agreement permits it. The
population-scale results of Section~\ref{sec:res-population} and the downstream counts of
Section~\ref{sec:res-downstream} depend on the Stage-2 outputs joined to coded fields at crash
level and cannot be recomputed from any releasable file, so the aggregated result files behind
them are released in place of the inputs, and a controlled-access route through the data owner
is the only way to reproduce them from source.

% \section*{Declaration of competing interest}
% The authors declare that they have no known competing financial interests or personal
% relationships that could have appeared to influence the work reported in this paper. The
% authors have no relationship with the developer of the model audited here, and the model
% was accessed at published list prices.

% ---------------------------------------------------------------------------
% REMOVED AT THE OWNER'S REQUEST, 2026-09-20. To restore, delete this banner
% and the leading %% on each line of the block below.
%% \section*{Funding}
%% This research did not receive any specific grant from funding agencies in the public,
%% commercial, or not-for-profit sectors.

% ---------------------------------------------------------------------------
% REMOVED AT THE OWNER'S REQUEST, 2026-09-20. To restore, delete this banner
% and the leading %% on each line of the block below.
%% \section*{Declaration of generative AI and AI-assisted technologies in the writing process}
%% During the preparation of this work the authors used Claude (Anthropic) in order to edit the
%% manuscript for structure and style and to check the consistency of reported numbers against
%% the generated analysis files. After using this tool, the authors reviewed and edited the
%% content as needed and take full responsibility for the content of the publication.

\printcredits

\bibliographystyle{cas-model2-names}
\bibliography{refs}

\end{document}